%% file: OMD4SEA_arxiv.tex
\documentclass[twoside,11pt]{article}
\usepackage[preprint]{jmlr2e}
\input{mymacro}
\usepackage{natbib}%
\usepackage{xspace}
\usepackage{paralist}
\usepackage{hyperref,url,dsfont,nicefrac}
\usepackage{bbm}
\usepackage{amsmath}
\usepackage{empheq}
\usepackage{xcolor}
\usepackage{graphicx}
\definecolor{gray}{rgb}{0.86,0.86,0.86}
\usepackage{verbatim}
\hypersetup{
    colorlinks,
    breaklinks,
    urlcolor = black,
    linkcolor = blue,
    citecolor = blue,
}

\newtheorem{ass}{Assumption}
\allowdisplaybreaks

\def \Dreg {\mathbf{Reg}_T^{\mathbf{d}}}
\def \Bcal {\mathcal{B}}
\def \ellb {\boldsymbol{\ell}}
\def \mb {\boldsymbol{m}}
\def \epsilon {\varepsilon}

\def \y {\mathbf{y}}
\def \E {\mathbb{E}}
\def \x {\mathbf{x}}
\def \g {\mathbf{g}}

\def \D {\mathcal{D}}
\def \O {\mathcal{O}}
\def \z {\mathbf{z}}
\def \u {\mathbf{u}}
\def \H {\mathcal{H}}

\def \R {\mathbb{R}}
\def \Z {\mathcal{Z}}

\def \W {\mathcal{W}}

\def \p {\boldsymbol{p}}

\def \B {\mathbf{B}}

\def \m {M}
\def \xh {\widehat{\x}}

\def \B {\mathcal{B}}
\def \X {\mathcal{X}}

\def \MR {\psi}

\newcommand\xqed[1]{%
  \leavevmode\unskip\penalty9999 \hbox{}\nobreak\hfill
  \quad\hbox{#1}}
\newcommand\remarkend{\xqed{$\triangleleft$}}

\newcommand \term[1]{\mathtt{term}~(\mathtt{#1})}
\newcommand \metaregret{\mathtt{meta}\mbox{-}\mathtt{regret}}
\newcommand \baseregret{\mathtt{base}\mbox{-}\mathtt{regret}}

\usepackage{lastpage}
\jmlrheading{1}{2023}{1-\pageref{LastPage}}{4/00}{10/00}{Chen et al}{Sijia Chen, Yu-Jie Zhang, Wei-Wei Tu, Peng Zhao, and Lijun Zhang}

\ShortHeadings{Optimistic OMD for Stochastic and Adversarial OCO}{Chen, Zhang, Tu, Zhao, Zhang}
\firstpageno{1}

\begin{document}
\title{Optimistic Online Mirror Descent for Bridging \\ Stochastic and Adversarial Online Convex Optimization}

\author{\name Sijia Chen \email chensj@lamda.nju.edu.cn \\
     \addr National Key Laboratory for Novel Software Technology, Nanjing University, China \AND
     \name Yu-Jie Zhang \email yujie.zhang@ms.k.u-tokyo.ac.jp \\
     \addr The University of Tokyo, Chiba, Japan \AND
     \name Wei-Wei Tu \email tuweiwei@4paradigm.com \\
     \addr 4Paradigm Inc., Beijing, China \AND 
     \name Peng Zhao \email zhaop@lamda.nju.edu.cn \AND
     \name Lijun Zhang \email zhanglj@lamda.nju.edu.cn\\
     \addr National Key Laboratory for Novel Software Technology, Nanjing University, China\\
     School of Artificial Intelligence, Nanjing University, China}

\editor{my editor}
\maketitle

\begin{abstract}
The Stochastically Extended Adversarial (SEA) model, introduced by~\citet{OCO:Between}, serves as an interpolation between stochastic and adversarial online convex optimization. Under the smoothness condition on expected loss functions, it is shown that the expected \emph{static regret} of optimistic Follow-The-Regularized-Leader (FTRL) depends on the cumulative stochastic variance $\sigma_{1:T}^2$ and the cumulative adversarial variation $\Sigma_{1:T}^2$ for convex functions.~\citet{OCO:Between} also provide a regret bound based on the maximal stochastic variance $\sigma_{\max}^2$ and the maximal adversarial variation $\Sigma_{\max}^2$ for strongly convex functions. Inspired by their work, we investigate the theoretical guarantees of optimistic Online Mirror Descent (OMD) for the SEA model with smooth expected loss functions. For convex and smooth functions, we obtain the \emph{same} $\O(\sqrt{\sigma_{1:T}^2}+\sqrt{\Sigma_{1:T}^2})$ regret bound, but with a relaxation of the convexity requirement from individual functions to expected functions. For strongly convex and smooth functions, we establish an $\O\left(\frac{1}{\lambda}\left(\sigma_{\max}^2+\Sigma_{\max}^2\right)\log \left(\left(\sigma_{1:T}^2 + \Sigma_{1:T}^2\right)/\left(\sigma_{\max}^2+\Sigma_{\max}^2\right)\right)\right)$ bound, \emph{better} than their $\O((\sigma_{\max}^2$ $ + \Sigma_{\max}^2) \log T)$ result. For \mbox{exp-concave} and smooth functions, our approach yields a \emph{new} $\O(d\log(\sigma_{1:T}^2+\Sigma_{1:T}^2))$ bound. Moreover, we introduce the first expected \emph{dynamic regret} guarantee for the SEA model with convex and smooth expected functions, which is more favorable than static regret bounds in non-stationary environments. Furthermore, we expand our investigation to scenarios with non-smooth expected loss functions and propose novel algorithms built upon optimistic OMD with an implicit update, successfully attaining both static and dynamic regret guarantees.
\end{abstract}

\input{sections/introduction}
\input{sections/related_work}
\input{sections/approach}
\input{sections/extensions}

\input{sections/implication}
\input{sections/conclusion}

\bibliography{SEA_bib}
\bibliographystyle{plainnat}
\newpage

\appendix
\input{appendixes/appendix.tex}

\end{document}

%% file: mymacro.tex
\usepackage{lipsum,booktabs}
\usepackage{amsmath,mathrsfs,amssymb,amsfonts,bm,enumitem}
\usepackage{rotating}
\usepackage{pdflscape}
\usepackage{tcolorbox}
\usepackage{hyperref,url,dsfont,nicefrac}
\hypersetup{
    colorlinks,
    breaklinks,
    linkcolor = blue,
    citecolor = blue,
    urlcolor  = blue,
}
\usepackage{amsmath}
\usepackage{amssymb}
\usepackage{mathtools}
\usepackage{dsfont}
\usepackage{lipsum}
\usepackage{bm,bbm}
\usepackage{bbding}
\usepackage{accents}

\usepackage{multirow}
\usepackage{graphicx}
\usepackage{subfigure}
\usepackage[font=small,labelfont=bf]{caption}
\usepackage{subfloat}
\usepackage{float}
\usepackage{algorithm,algorithmic}
\usepackage{wrapfig}
\usepackage{booktabs}
\usepackage{verbatim}
\usepackage{setspace}
\usepackage{tcolorbox}

\let\proof\relax

\usepackage{amsthm}
\let\proof\relax

\newenvironment{proof}{\par\noindent{\bf Proof\ }}{\hfill\BlackBox\\[2mm]}

\def \E {\mathbb{E}}
\def \x {\mathbf{x}}
\def \y {\mathbf{y}}
\def \g {\mathbf{g}}

\def \D {\mathcal{D}}
\def \z {\mathbf{z}}
\def \u {\mathbf{u}}
\def \H {\mathcal{H}}

\def \O {\mathcal{O}}
\def \R {\mathbb{R}}

\def \DReg {\mathop{\bm{\mathrm{Reg}}}\nolimits_{T}^{\bm{\mathsf{d}}}}

\def \W {\mathcal{W}}

\def \p {\mathbf{p}}

\def \B {\mathbf{B}}

\def \m {\mathbf{m}}
\def \xh {\widehat{\x}}

\def \B {\mathcal{B}}

\def \X {\mathcal{X}}

\def \ellb {\bm{\ell}}

\DeclareMathOperator*{\argmin}{arg\,min}
\renewcommand{\tilde}{\widetilde}
\renewcommand{\hat}{\widehat}

\usepackage{mathtools}
\usepackage{appendix}
\let\norm\undefined 
\DeclarePairedDelimiter\norm{\lVert}{\rVert}

\newcommand\inner[2]{\langle #1, #2 \rangle}
\newtheorem{myThm}{Theorem}

\newtheorem{myCor}{Corollary}

\newtheorem{myLemma}{Lemma}

\theoremstyle{definition}

\newtheorem{myRemark}{Remark}

\renewcommand{\tilde}{\widetilde}
\renewcommand{\hat}{\widehat}

\def \E {\mathbb{E}}
\def \D {\mathcal{D}}

\def \H {\mathcal{H}}

\def \R {\mathbb{R}}

\def \W {\mathcal{W}}
\def \X {\mathcal{X}}

\def \x {\mathbf{x}}
\def \y {\mathbf{y}}

\def \epsilon {\varepsilon}

\definecolor{DSgray}{cmyk}{0,1,0,0}

\usepackage{prettyref}
\newcommand{\pref}[1]{\prettyref{#1}}

\newcommand{\savehyperref}[2]{\texorpdfstring{\hyperref[#1]{#2}}{#2}}
\newrefformat{eq}{\savehyperref{#1}{Eq.~(\ref*{#1})}}
\newrefformat{eqn}{\savehyperref{#1}{Equation~\ref*{#1}}}
\newrefformat{lem}{\savehyperref{#1}{Lemma~\ref*{#1}}}
\newrefformat{lemma}{\savehyperref{#1}{Lemma~\ref*{#1}}}
\newrefformat{def}{\savehyperref{#1}{Definition~\ref*{#1}}}
\newrefformat{line}{\savehyperref{#1}{Line~\ref*{#1}}}
\newrefformat{thm}{\savehyperref{#1}{Theorem~\ref*{#1}}}
\newrefformat{corr}{\savehyperref{#1}{Corollary~\ref*{#1}}}
\newrefformat{cor}{\savehyperref{#1}{Corollary~\ref*{#1}}}
\newrefformat{sec}{\savehyperref{#1}{Section~\ref*{#1}}}
\newrefformat{app}{\savehyperref{#1}{Appendix~\ref*{#1}}}
\newrefformat{appendix}{\savehyperref{#1}{Appendix~\ref*{#1}}}
\newrefformat{assum}{\savehyperref{#1}{Assumption~\ref*{#1}}}
\newrefformat{ex}{\savehyperref{#1}{Example~\ref*{#1}}}
\newrefformat{fig}{\savehyperref{#1}{Figure~\ref*{#1}}}
\newrefformat{alg}{\savehyperref{#1}{Algorithm~\ref*{#1}}}
\newrefformat{rem}{\savehyperref{#1}{Remark~\ref*{#1}}}
\newrefformat{conj}{\savehyperref{#1}{Conjecture~\ref*{#1}}}
\newrefformat{prop}{\savehyperref{#1}{Proposition~\ref*{#1}}}
\newrefformat{proto}{\savehyperref{#1}{Protocol~\ref*{#1}}}
\newrefformat{prob}{\savehyperref{#1}{Problem~\ref*{#1}}}
\newrefformat{claim}{\savehyperref{#1}{Claim~\ref*{#1}}}
\newrefformat{que}{\savehyperref{#1}{Question~\ref*{#1}}}
\newrefformat{op}{\savehyperref{#1}{Open Problem~\ref*{#1}}}
\newrefformat{fn}{\savehyperref{#1}{Footnote~\ref*{#1}}}
\allowdisplaybreaks[3]

%% file: sections/introduction.tex

\section{Introduction}
\label{Introduction}
Online convex optimization (OCO) is a fundamental framework for online learning and has been applied in a variety of real-world applications such as spam filtering and portfolio management~\citep{Intro:Online:Convex}. OCO problems can be mainly divided into two categories: adversarial online convex optimization (adversarial OCO)~\citep{zinkevich-2003-online, ML:Hazan:2007} and stochastic online convex optimization (SCO)~\citep{nemirovski-2008-robust,COLT:Hazan:2011,Lan:SCO}. Adversarial OCO assumes that the loss functions are chosen arbitrarily or adversarially and the goal is to minimize the regret. SCO assumes that the loss functions are independently and identically distributed (i.i.d.), and the goal is to minimize the excess risk. Although the two models have been extensively studied~\citep{COLT:Shalev:2009,Intro:Online:Convex,Modern:Online:Learning}, in real scenarios the nature is not always completely adversarial or stochastic, but often lies somewhere in between.

\subsection{The Stochastically Extended Adversarial Model}\label{subsec:intro:sea}
The Stochastically Extended Adversarial (SEA) model is introduced by~\citet{OCO:Between} as an intermediate problem setup between adversarial OCO and SCO. In round $t \in [T]$, the learner selects a decision $\x_t$ from a convex feasible domain $\X \subseteq \R^d$, and nature chooses a  distribution $\mathfrak{D}_t$ from a set of distributions. Then, the learner suffers a loss $f_t(\x_t)$, where the individual function (also called random function) $f_t$ is sampled from the distribution $\mathfrak{D}_t$. The distributions are allowed to vary over time, and by choosing them appropriately, the SEA model reduces to adversarial OCO, SCO, or other intermediate settings. Additionally, for each $t \in [T]$, they define the (conditional) expected function as $F_t(\x) = \E_{f_t \sim \mathfrak{D}_t}[f_t(\x)]$.

Due to the randomness in the online process, our goal in the SEA model is to bound the \emph{expected regret} against any fixed comparator $\u \in \X$, defined as
\begin{equation} \label{eqn:expect:regret}
\E[\mathbf{Reg}_T(\u)] \triangleq \E\left[ \sum_{t=1}^T f_t(\x_t) - \sum_{t=1}^T f_t(\u) \right].
\end{equation}
Furthermore, to capture the characteristics of the SEA model, \citet{OCO:Between} introduce the following quantities. For each $t \in [T]$, define the (conditional) variance of gradients as 
\begin{equation} \label{eqn:variance:grad}
\sigma_t^2 = \sup_{\x \in \X} \E_{f_t \sim \mathfrak{D}_t}  \big[ \| \nabla f_t(\x) - \nabla F_t(\x) \|_2^2 \big].
\end{equation}
Notice that both $F_t(\x)$ and $\sigma_t^2$ can be random variables due to the randomness of distribution $\mathfrak{D}_t$. Then, the cumulative stochastic variance can be defined as
\begin{equation} \label{eqn:cumu:variance}
\sigma_{1:T}^2 = \E \left[ \sum_{t=1}^T \sigma_t^2 \right],
\end{equation}
which reflects the stochastic aspect of the online process. Moreover, the cumulative adversarial variation is defined as
\begin{equation} \label{eqn:cumu:varia}
\Sigma_{1:T}^2= \E \left[ \sum_{t=1}^T \sup_{\x \in \X} \|\nabla F_t(\x)- \nabla F_{t-1}(\x)\|_2^2\right],
\end{equation}
where $\nabla F_0(\x) = 0$, reflecting the adversarial difficulty.\footnote{If the nature is oblivious, then both $F_t(\x)$ and $\sigma_t^2$ will be deterministic and we can remove the expectation in (\ref{eqn:cumu:variance}) and~\eqref{eqn:cumu:varia}.}

\subsection{Existing Results}\label{subsec:intro:existresults}
With the smoothness of expected loss functions, \citet{OCO:Between} establish a series of results for the SEA model, including convex functions and strongly convex functions. 

In the case of convex and smooth functions, they prove an $\O(\sqrt{\sigma_{1:T}^2}+\sqrt{\Sigma_{1:T}^2})$ regret bound of optimistic follow-the-regularized-leader (FTRL). Note that they require the individual functions $\{f_t\}_{t=1}^T$ to be convex, which is relatively strict. When facing the adversarial setting, we have $\sigma_t^2 = 0$ for all $t$ and $\Sigma_{1:T}^2$ is equivalent to the gradient variation $V_T \triangleq \sum_{t=2}^T \sup_{\x \in \X} \| \nabla f_t(\x) - \nabla f_{t-1}(\x)\|_2^2$, so the bound implies a regret bound in the form of $\sum_{t=1}^T f_t(\x_t) - \min_{\x \in \X}\sum_{t=1}^T f_t(\x)\leq \O(\sqrt{V_T})$, matching the gradient-variation bound of~\citet{Gradual:COLT:12} and also recovering the $\O(\sqrt{T})$ bound in the worst case~\citep{zinkevich-2003-online}. In the SCO setting, we have $\Sigma_{1:T}^2 = 0$ since $F_1 = \cdots = F_T \triangleq F$, and $\sigma_t = \sigma$ for all $t$, where $\sigma$ denotes the variance of stochastic gradients. Then they obtain $\O(\sigma \sqrt{T})$ regret, leading to an excess risk bound in the form of $F(\x_T) - \min_{\x \in \X} F(\x)\leq \O(\sigma/\sqrt{T})$ through the standard online-to-batch conversion~\citep{TIT04:Bianchi}.

To investigate the strongly convex case, they assume that the \emph{maximum value} of stochastic variance is $\sigma_{\max}^2$ and the \emph{maximum value} of adversarial variation is $\Sigma_{\max}^2$; please refer to Assumption~\ref{ass:sigma} for details. Then~\citet{OCO:Between} prove an $\O((\sigma_{\max}^2 + \Sigma_{\max}^2) \log T)$ expected regret bound of optimistic FTRL for $\lambda$-strongly convex and smooth functions. Considering the adversarial setting, we have $\sigma_{\max}^2 = 0$ and $\Sigma_{\max}^2 \leq 4G^2$ where $G$ is the upper bound of individual function gradients, so their bound implies an $\O(\log T)$ regret bound. We note that unlike in the convex and smooth case, their expected regret bound fails to recover the $\O(\log V_T)$ gradient-variation bound~\citep{ICML:2022:Zhang}. In the SCO setting, we have $\Sigma_{\max}^2 = 0$ and $\sigma_{\max}^2 = \sigma^2$. Therefore, their result brings an $\O([\sigma^2 \log T]/T)$ excess risk bound through the online-to-batch conversion.

\subsection{Our Contributions}\label{subsec:intro:contribution}
Optimistic FTRL is an optimistic online learning algorithm~\citep{Predictable:COLT:2013}, which aims to exploit prior knowledge during the online process. Optimistic Online Mirror Descent (OMD) is another popular optimistic online learning algorithm, from which the gradient-variation bound of~\citet{Gradual:COLT:12} is (originally) derived. With the promising outcomes of optimistic FTRL~\citep{OCO:Between}, it is natural to inquire about optimistic OMD's theoretical guarantees for the SEA model, and we address this below.
\begin{compactitem}
    \item For convex and smooth functions, optimistic OMD enjoys the same $\O(\sqrt{\sigma_{1:T}^2}+\sqrt{\Sigma_{1:T}^2})$ expected regret bound as~\citet{OCO:Between}, but reduces their need for convexity of individual functions to a need for convexity of expected functions. 
    \item For strongly convex and smooth functions, optimistic OMD attains an $\O(\frac{1}{\lambda}(\sigma_{\max}^2$ $+\Sigma_{\max}^2)\log (\left(\sigma_{1:T}^2 + \Sigma_{1:T}^2\right)/(\sigma_{\max}^2+\Sigma_{\max}^2)))$ bound, better than the $\O(\frac{1}{\lambda}(\sigma_{\max}^2 + \Sigma_{\max}^2) $ $\log T)$ bound of \citet{OCO:Between} for optimistic FTRL in any case. 
    \item For exp-concave and smooth functions, our work establishes a new $\O(d\log(\sigma_{1:T}^2+\Sigma_{1:T}^2))$ bound for optimistic OMD, where $d$ denotes the dimensionality of decisions. 
    \item Our better results for optimistic OMD stem from more careful analyses and do not imply inherent superiority over optimistic FTRL for regret minimization. When encountering convex functions, we present a different analysis from \citet{OCO:Between}'s analysis of optimistic FTRL, thereby similarly weakening the convexity-related assumption as in optimistic OMD while achieving the same regret bound. We also provide new analyses for strongly convex functions and exp-concave functions respectively, both obtaining the same expected regret bounds as optimistic OMD.\vspace{2mm}
\end{compactitem}

\paragraph{Extension to Dynamic Regret.}~The metric~\eqref{eqn:expect:regret} is commonly referred to as expected static regret since the comparator is unchanged over time. We further extend the scope of the SEA model to optimize \emph{expected dynamic regret}~\citep{zinkevich-2003-online}, defined as
\begin{align}
\label{eq:dynamic-regret-measure}
	\E[\Dreg(\u_1,\cdots,\u_T)] \triangleq \E\left[ \sum_{t=1}^T f_t(\x_t) - \sum_{t=1}^T f_t(\u_t)\right],
\end{align}
where $\u_1,\ldots,\u_T \in \X$ is a sequence of (potentially) time-varying comparators. Note that the comparators can depend on the expected functions $\{F_1,\ldots,F_T\}$ and are required to be independent of the individual functions $\{f_1,\ldots,f_T\}$. To optimize the dynamic regret, we introduce the path length $P_T = \E[\sum_{t=2}^T \norm{\u_t -\u_{t-1}}_2]$ to measure the non-stationarity level, where $\E[\cdot]$ is taken over the potential randomness of the expected functions. Notably, the static regret~\eqref{eqn:expect:regret} can be treated as a special case with $\u_1=\ldots=\u_T = \u$. For the SEA model with convex and smooth expected functions, we obtain an $\O(P_T+\sqrt{1+P_T}(\sqrt{\sigma_{1:T}^2} + \sqrt{\Sigma_{1:T}^2}))$ expected dynamic regret. The bound is new and immediately recovers the $\O(\sqrt{\sigma_{1:T}^2}+\sqrt{\Sigma_{1:T}^2})$ expected static regret given $P_T=0$. It can also imply the $\O(\sqrt{(1+P_T+V_T)(1+P_T)})$ gradient-variation dynamic regret bound of \citet{Problem:Dynamic:Regret,JMLR:sword++} in the adversarial setting and reduce to the $\O(\sqrt{T(1+P_T)})$ dynamic regret in the worst case~\citep{Adaptive:Dynamic:Regret:NIPS} . We regard the support of dynamic regret as an advantage of optimistic OMD over optimistic FTRL. To the best of our knowledge, even $\O(\sqrt{T(1+P_T)})$ dynamic regret has not been established for FTRL-style methods in online convex optimization. 

\paragraph{Extension to Non-smooth Functions.}~In addition, by combining optimistic OMD with \emph{implicit update}, we extend our investigation to \emph{non-smooth} loss functions. For the SEA model with convex and non-smooth functions, we first establish an $\O(\sqrt{\tilde{\sigma}_{1:T}^2} +\sqrt{\Sigma_{1:T}^2} )$ static regret, based on which we further propose a two-layer algorithm equipped with an $\O(\sqrt{1+P_T}(\sqrt{\tilde{\sigma}_{1:T}^2}+\sqrt{\Sigma_{1:T}^2}))$ dynamic regret, where $\tilde{\sigma}_{1:T}^2$ defined in~\eqref{nonsmooth:tildesigma} represents a slightly more relaxed measure than $\sigma_{1:T}^2$.

Based on all the above theoretical guarantees, we apply optimistic OMD to a variety of intermediate cases between adversarial OCO and SCO. This leads to \emph{better} results for strongly convex functions and \emph{new} results for exp-concave functions, thereby enriching our understanding of the intermediate scenarios. Furthermore, our emphasis on dynamic regret minimization enables us to derive novel corollaries for the online label shift problem~\citep{NeurIPS'22:label_shift}, an interesting new problem setup with practical appeals. 

Compared to our earlier conference version~\citep{ICML'23:OMD4SEA}, this extended version provides significantly more results, along with refined presentations and more detailed analysis. Firstly, by revisiting and refining our analysis, we provide a better regret bound for strongly convex functions than our previous bound of \citep{ICML'23:OMD4SEA}. Secondly, we incorporate a more detailed analysis of dynamic regret minimization within the SEA model, adding insights to explain the optimism design's rationale and highlighting the disadvantages of alternative approaches. Thirdly, we investigate the SEA model with \emph{non-smooth} functions, where we employ optimistic OMD with an implicit update and obtain favorable regret guarantees. Additionally, we explore dynamic regret minimization with non-smooth functions. Lastly, we apply our findings to address the online label shift problem, yielding results that further demonstrate the SEA model's real-world applicability.

\paragraph{Organization. } The remainder of the paper is structured as follows. Section~\ref{sec:related} briefly reviews the related work. Our main results can be found in Section~\ref{sec:results}, in which we establish theoretical guarantees for convex, strongly convex, and exp-concave loss functions under the smoothness condition on loss functions respectively. In Section~\ref{sec:extensions}, we extend the investigations to dynamic regret minimization and non-smooth loss functions. In Section~\ref{sec:examples}, we illustrate our results by giving some special implications, such as online learning with limited resources and online label shift.  Section~\ref{sec:conclusion} concludes the paper and discusses future work. Some omitted details and proofs are provided in the appendix.

%% file: sections/related_work.tex
\section{Related Work}\label{sec:related}
This section reviews related works in adversarial OCO, SCO, and intermediate settings. 

\subsection{Adversarial Online Convex Optimization}
\label{subsec:adversarialOCO}
Adversarial OCO can be seen as a repeated game between the online learner and the nature (or called the environment). In round $t \in [T]$, the online learner chooses a decision $\x_t$ from the convex feasible set $\X \subseteq \R^d$, and suffers a convex loss $f_t(\x_t)$ which the nature may adversarially select. The goal in adversarial OCO is to minimize the \emph{regret}:
\begin{align*}
    \mathbf{Reg}_T \triangleq \sum_{t=1}^T f_t(\x_t) - \min_{\x \in \X}\sum_{t=1}^T f_t(\x),
\end{align*}
which measures the cumulative loss difference between the learner and the best decision in hindsight \citep{Modern:Online:Learning}. For convex functions, Online Gradient Descent (OGD) achieves an $\O(\sqrt{T})$ regret with a step size of $\eta_t = \O(1/\sqrt{t})$ \citep{zinkevich-2003-online}. For $\lambda$-strongly convex functions, an $\O(\frac{1}{\lambda}\log T)$ bound is attained by OGD with $\eta_t=\O(1/[\lambda t])$~\citep{Shai:thesis}. For $\alpha$-exp-concave functions, Online Newton Step (ONS)~\citep{ML:Hazan:2007} obtains an $\O(\frac{d}{\alpha}\log T)$ bound. Those results are considered minimax optimal~\citep{Lower:bound:Portfolio,Minimax:Online} and cannot be improved in general.

Furthermore, various algorithms have been proposed to achieve \emph{problem-dependent} regret guarantees, which safeguard the minimax rates in the worst case and become better when problems satisfy benign properties such as smoothness~\citep{NIPS2010_Smooth,Gradual:COLT:12,Beyond:Logarithmic,Problem:Dynamic:Regret,JMLR:sword++}, sparsity~\citep{JMLR:Adaptive,COLT'10:adaptiveOCO,NIPS'18:sparse-OCO}, or other structural properties~\citep{Adam,TCS'20:modular-composite}. Among them, it is shown by~\citet{Gradual:COLT:12} that the regret for OCO with smooth functions can be upper bounded by the gradient-variation quantity, defined as
\begin{equation}
    \label{eq:variation}
    V_T = \sum_{t=2}^T \sup_{\x \in \X} \| \nabla f_t(\x) - \nabla f_{t-1}(\x)\|_2^2.
\end{equation}
Specifically, using the OMD framework with suitable configurations can attain an $\O(\sqrt{V_T})$ regret for convex and smooth functions and attain an $\O(\frac{d}{\alpha}\log V_T)$ regret for $\alpha$-exp-concave and smooth functions. \citet{ICML:2022:Zhang} extended the result to $\lambda$-strongly convex and smooth functions, achieving an $\O(\frac{1}{\lambda}\log V_T)$ bound. These bounds are notably tighter than previous problem-independent results when the loss functions change slowly such that the gradient variation $V_T$ is small.

Subsequently, \citet{Predictable:COLT:2013} introduced the paradigm of optimistic online learning, designed to leverage prior knowledge about upcoming loss functions. In this approach, the learner receives a prediction of the next loss in each round, which is used to secure tighter bounds when the predictions prove accurate and still preserve the worst-case regret bound otherwise. Then two frameworks are developed: optimistic FTRL and optimistic OMD, where the latter generalized the algorithm of \citet{Gradual:COLT:12}.

\subsection{Stochastic Online Convex Optimization}\label{related-work:sco}
SCO assumes i.i.d.~loss functions and aims to minimize the convex objective in an expectation form: $\min_{\x \in \X} F(\x)$, where $F(\x) = \E_{f\sim \mathfrak{D}}[f(\x)]$. The performance measure is the \emph{excess risk} of the solution point over the optimum, that is, $F(\x_T) - \min_{\x \in \X} F(\x)$.

For Lipschitz and convex functions, Stochastic Gradient Descent (SGD) achieves an $\O(1/\sqrt{T})$ excess risk bound. Improved rates are achievable when functions have additional properties. For smooth functions, SGD reaches an $\O(1/T + \sqrt{F_{*}/T})$ rate with $F_{*} = \min_{\x \in \X} F(\x)$, which will be tighter than $\O(1/\sqrt{T})$ when $F_{*}$ is small \citep{NIPS2010_Smooth}. For $\lambda$-strongly convex functions, \citet{COLT:Hazan:2011} establish an $\O(1/[\lambda T])$ excess risk bound through a variant of SGD. For $\alpha$-exp-concave functions, ONS provides an $\O(d\log T/[\alpha T])$ rate~\citep{ML:Hazan:2007,COLT:2015:Mahdavi}. When functions satisfy strong convexity and smoothness simultaneously, Accelerated Stochastic Approximation (AC-SA) achieves an $\O(1/T)$ rate with a smaller constant~\citep{Lan:SCSC}. Even faster results can be attained with strengthened conditions and advanced algorithms~\citep{NIPS13:ASGD,NIPS:13:Mixed, COLT'18:weighted-LS,Stochastic:Approximation:COLT}.

\subsection{Intermediate Setting}
In recent years, intermediate settings between adversarial OCO and SCO have drawn attention in Prediction with Expert Advice (PEA) problems~\citep{2020Prediction} and bandit problems~\citep{zimmert2021tsallis}. \citet{2020Prediction} study the stochastic regime with adversarial corruptions in PEA problems, achieving an $\O(\log N/\Delta + C_T)$ bound, where $N$ is the number of experts, $\Delta$ the suboptimality gap and $C_T \geq 0$ the corruption level. In bandit problems, \citet{zimmert2021tsallis} focus on the adversarial regime with a self-bounding constraint, establishing an $\O(N\log T/\Delta + \sqrt{C_T N\log T/\Delta})$ bound. \citet{2021On} further demonstrates an expected regret bound of $\O(\log N/\Delta + \sqrt{C_T\log N/\Delta})$ in this context. However, as mentioned by~\citet{2021On}, we know very little about the intermediate setting in OCO, with recent contributions like \citet{OCO:Between} being exceptions.

%% file: sections/approach.tex

\section{Optimistic Mirror Descent for the SEA Model}
\label{sec:results}
In this section, we first list the assumptions that will be used later. Then, we introduce  \textsc{optimistic OMD}, our main algorithmic framework. After that, we discuss its theoretical guarantees for the SEA model, along with new results of optimistic FTRL. The final subsection is dedicated to analyzing these results. 

\subsection{Assumptions}
The assumptions listed below may be employed in our analysis. It is important to note that we will clearly specify the assumptions utilized in the theorem statements.
\begin{ass}[gradient norms boundedness]\label{ass:3} The gradient norms of all the individual functions are bounded by $G$, i.e.~for all $t \in [T]$, we have $\max_{\x \in \X}\|\nabla f_t(\x)\|_2 \leq G$.
\end{ass}
\begin{ass}[domain boundedness]\label{ass:4} The domain $\X$ contains the origin $\mathbf{0}$, and the diameter of $\X$ is bounded by $D$, i.e.,~for all $\x,\y \in \X$, we have $\|\x -\y\|_2 \leq D$.
\end{ass}
\begin{ass}[maximal stochastic variance and adversarial variation]\label{ass:sigma}
All the variances of realizable gradients are at most $\sigma_{\max}^2$, and all the adversarial variations are upper bounded by $\Sigma_{\max}^2$, i.e., $\forall t \in [T]$, it holds that $\sigma_{t}^2 \leq \sigma_{\max}^2$ and $\sup_{\x \in \X} \|\nabla F_t(\x)- \nabla F_{t-1}(\x)\|_2^2 \leq \Sigma_{\max}^2$.
\end{ass}
\begin{ass}[smoothness of expected functions]\label{ass:2} For all $t\in [T]$, the expected function $F_t(\cdot)$ is $L$-smooth over $\X$, i.e.,~$\|\nabla F_t(\x) -\nabla F_t(\y)  \|_2 \leq L \|\x -\y\|_2,\ \forall \x,\y \in \X$.
\end{ass}
\begin{ass}[convexity of expected functions]\label{ass:1}
For all $t\in [T]$, the expected function $F_t(\cdot)$ is convex over $\X$.
\end{ass}
\begin{ass}[strong convexity of expected functions]\label{ass:5} For $t\in[T]$, the expected function $F_t(\cdot)$ is $\lambda$-strongly convex over $\X$.
\end{ass}
\begin{ass}[exponential concavity of individual functions]\label{ass:6} For $t\in [T]$, the individual function $f_t(\cdot)$ is $\alpha$-exp-concave over $\X$.
\end{ass}
\begin{ass}[convexity of individual functions]\label{ass:individual:convexity}
For all $t\in [T]$, the individual function $f_t(\cdot)$ is convex over $\X$.
\end{ass}

\subsection{Algorithm}
Optimistic OMD is a versatile and powerful framework for online learning~\citep{Predictable:COLT:2013}. During the learning process, it maintains two sequences $\{\x_t\}_{t=1}^T$ and $\{\xh_t\}_{t=1}^T$. In round $t \in [T]$, the learner first submits the decision $\x_t$ and observes the individual function $f_t(\cdot)$. Then, an optimistic vector $\m_{t+1} \in \R^d$ is received that encodes certain prior knowledge of the (unknown) function $f_{t+1}(\cdot)$, and the algorithm updates by
\begin{align}
\xh_{t+1} &= \argmin_{\x \in \X} \langle \nabla f_t(\x_t), \x \rangle + \D_{\MR_t}(\x, \xh_t) \label{eqn:update:u},\\
\x_{t+1} &= \argmin_{\x \in \X} \langle \m_{t+1}, \x \rangle + \D_{\MR_{t+1}}(\x, \xh_{t+1}) \label{eqn:update:x} ,
\end{align}
where $\D_{\MR}(\x,\y) = \MR(\x)-\MR(\y)- \langle \nabla \MR(\y), \x-\y \rangle$ denotes the Bregman divergence induced by a differentiable convex function $\MR: \X \mapsto \R$ (or usually called regularizer). In our work, we allow the regularizer to be time-varying. The specific choice of $\MR_t(\cdot)$ depends on the type of online functions and will be determined later. 

To leverage the possible smoothness of functions, we simply set the optimism as the last-round gradient, that is, $\m_{t+1}=\nabla f_t(\x_t)$~\citep{Gradual:COLT:12}. We initialize $\x_1=\xh_1$ as an arbitrary point in $\X$. The overall procedures are summarized in Algorithm~\ref{alg:1}.

\begin{algorithm}[t]
\caption{Optimistic Online Mirror Descent (Optimistic OMD)}
\begin{algorithmic}[1]
\REQUIRE{Regularizer $\MR_t: \X \mapsto \R$}
\STATE Set $\x_1=\xh_1$ to be any point in $\X$
\FOR{$t=1,\ldots,T$}
\STATE Submit $\x_t$ and the nature selects a distribution $\mathfrak{D}_t$
\STATE Receive $f_t(\cdot)$, which is sampled from $\mathfrak{D}_t$
\STATE Receive an optimistic vector $\m_{t+1}$ encoding certain prior knowledge of $f_{t+1}(\cdot)$
\STATE Update $\hat{\x}_{t+1}$ and $\x_{t+1}$ according to (\ref{eqn:update:u}) and (\ref{eqn:update:x})
\ENDFOR
\end{algorithmic}
\label{alg:1}
\end{algorithm}

\begin{myRemark}
If we drop the expectation operation, the measure~\eqref{eqn:expect:regret} becomes the standard regret. Consequently, a straightforward way is to integrate existing regret bounds of optimistic OMD~\citep{Gradual:COLT:12,Predictable:COLT:2013} and subsequently simplify the expectation. However, as elaborated in~\citet[Remark 4]{OCO:Between}, this approach only yields very loose bounds. Therefore, it becomes necessary to dig into the analysis and scrutinize the influence of expectations during the intermediate steps. 
\remarkend
\end{myRemark}

In the following, we consider three different instantiations of Algorithm~\ref{alg:1}, each corresponding to the SEA model with different types of functions: convex, strongly convex, and exp-concave functions, respectively. We also provide their respective theoretical guarantees.

\subsection{Convex and Smooth Functions}\label{subsec:con}
In this part, we focus on the case that \emph{expected functions} are convex and smooth. \citet{OCO:Between} require individual functions $f_t(\cdot)$ ($t\in[T]$) to be convex (see Assumption~A1 of their paper), whereas we only require expected functions $F_t(\cdot)$ ($t\in[T]$) to be convex, which is a much weaker condition. This relaxation, which has been studied in many stochastic optimization works \citep{shalev2016sdca,hu2017unified,ahn2020sgd}, is due to the observation that the expectation in~\eqref{eqn:expect:regret} eliminates the need for convexity in individual functions. Specifically, for any fixed $\u \in \X$ we have
\begin{equation} 
\E\big[ f_t(\x_t) - f_t(\u) \big] = \E\big[ F_t(\x_t) - F_t(\u) \big] 
\leq \E\big[ \langle \nabla F_t(\x_t), \x_t -\u \rangle \big] = \E\big[  \langle \nabla f_t(\x_t), \x_t -\u \rangle \big]. \label{eqn:exp:cov}
\end{equation}
The inequality arises from the convexity of $F_t(\cdot)$ and the last step is due to the interchangeability of differentiation and integration by Leibniz integral rule. Note that the independence between $\u$ and $f_t$ is important for this derivation. We emphasize that if $\u$ is chosen based on random functions, then the convexity of random functions will be necessary.\footnote{Fortunately, a favorable choice of $\u$ is usually independent of random functions. For instance, if the nature is oblivious, we can choose $\u = \u^* \in \argmin_{\u \in \X}\sum_{t=1}^T F_t(\u)$, which only depends on expected functions. Additionally, fitting to $\{F_1,\cdots,F_T\}$ is preferable in practice as fitting to $\{f_1,\cdots,f_T\}$ might cause overfitting.} 

Below, we focus on the optimization over bound the expected regret in terms of the linearized function, i.e., $\sum_{t=1}^T \langle \nabla f_t(\x_t), \x_t -\u \rangle$. For convex and smooth functions, we configure the algorithm with  Euclidean regularizer 
\begin{equation} \label{eqn:step:size:cov}
\MR_t(\x) = \frac{1}{2 \eta_t} \|\x\|_2^2\quad \text{and step size} \quad \eta_t = \frac{D}{\sqrt{\delta+4G^2 + \bar{V}_{t-1}}},
\end{equation}
where $\bar{V}_{t-1} = \sum_{s=1}^{t-1} \|\nabla f_s(\x_s) - \nabla f_{s-1} (\x_{s-1}) \|_2^2$ (assuming $\nabla f_0(\x_0)=0$) and $\delta>0$ is a parameter to be specified later. Then, the optimistic OMD updates in (\ref{eqn:update:u}) and (\ref{eqn:update:x}) become
\begin{align} 
    \label{eqn:update:cov:u}
    \xh_{t+1} = \Pi_{\X} \big[ \xh_t - \eta_t  \nabla f_t(\x_t)\big],~~\x_{t+1} = \Pi_{\X} \big[ \xh_{t+1} - \eta_{t+1}  \nabla f_t(\x_t)\big],  
\end{align}
where $\Pi_{\X}[\cdot]$ denotes the Euclidean projection onto the feasible domain $\X$. The algorithm executes gradient descent twice per round, using an adaptive step size akin to self-confident tuning~\citep{AUER200248}. This approach obviates the need for the doubling trick used in prior works~\citep{Gradual:COLT:12,Predictable:COLT:2013,Dynamic:AISTATS:15}.

Below, we present the theoretical guarantee of optimistic OMD for the SEA model with convex and smooth functions. The proof is in Section~\ref{appendix:proofs-convex}.
\begin{myThm} \label{thm:1} Under Assumptions~\ref{ass:3}, \ref{ass:4}, \ref{ass:2} and \ref{ass:1}, optimistic OMD with regularizer~\eqref{eqn:step:size:cov} and updates~\eqref{eqn:update:cov:u} enjoys the following guarantee:
\begin{align*}
\E[\mathbf{Reg}_T(\u)]
\leq  5 \sqrt{10} D^2 L +  \frac{5\sqrt{5}  DG}{2} + 5 \sqrt{2}D\sqrt{ \sigma_{1:T}^2 } + 5 D \sqrt{\Sigma_{1:T}^2 } 
=  \O\left(\sqrt{\sigma_{1:T}^2}+\sqrt{\Sigma_{1:T}^2}\right),
\end{align*}
where we set $\delta=10 D^2 L^2$ in~\eqref{eqn:step:size:cov}.
\end{myThm}
\begin{myRemark}
Theorem~\ref{thm:1} demonstrates the same regret bound as the work of~\citet{OCO:Between}, but under weaker assumptions --- we require only the convexity of expected functions, as opposed to individual functions in their work. The regret bound is optimal according to the lower bound of~\citet[Theorem 6]{OCO:Between}. 
\remarkend
\end{myRemark}

In this subsection's final part, we provide a new result of optimistic FTRL for the SEA model. Notably, we illustrate that even \emph{without} the convexity of individual functions, optimistic FTRL can achieve the \emph{same} guarantee as~\citet{OCO:Between}. This is achieved by using a linearized surrogate loss $\{\langle \nabla f_t(\x_t), \cdot\rangle\}_{t=1}^T$ instead of the original loss $\{ f_t(\cdot)\}_{t=1}^T$.
\begin{myThm} \label{thm:convex-FTRL}
Under Assumptions~\ref{ass:3}, \ref{ass:4}, \ref{ass:2}, and~\ref{ass:1}~(without assuming convexity of individual functions), with an appropriate setup for the optimistic FTRL (see details in Appendix~\ref{appendix:proofs-convex-FTRL}), the expected regret  is at most $\O\big(\sqrt{\sigma_{1:T}^2}+\sqrt{\Sigma_{1:T}^2}\big)$.
\end{myThm}

\subsection{Strongly Convex and Smooth Functions}\label{subsec:result:str}
In this part, we examine the case when \emph{expected functions} are strongly convex and smooth. We still employ optimistic OMD (Algorithm~\ref{alg:1}) and define the regularizer as
\begin{equation} \label{eqn:step:size:str}
    \MR_t(\x) = \frac{1}{2 \eta_t} \|\x\|_2^2\quad \text{with step size} \quad \eta_t = \frac{2}{\lambda t} .
\end{equation}
It is worth mentioning that this step size configuration is \emph{new} and much simpler than the self-confident step size used in earlier research on gradient-variation bounds for strongly convex and smooth functions~\citep{ICML:2022:Zhang}. Then the update rules maintain the same form as~\eqref{eqn:update:cov:u} in essence. We provide the following expected regret bound for the SEA model with strongly convex and smooth functions, the proof of which is in Section~\ref{appendix:proofs-sc}.

\begin{myThm} \label{thm:2} Under Assumptions~\ref{ass:3}, \ref{ass:4}, \ref{ass:sigma}, \ref{ass:2} and \ref{ass:5}, optimistic OMD with regularizer~\eqref{eqn:step:size:str} and updates~\eqref{eqn:update:cov:u} enjoys the following guarantee
\begin{align*}
 \E[\mathbf{Reg}_T(\u)]
\leq{}&\frac{32\sigma_{\max}^2+16\Sigma_{\max}^2}{\lambda}\ln \left(\frac{1}{2\sigma_{\max}^2+\Sigma_{\max}^2}\left(2\sigma_{1:T}^2 + \Sigma_{1:T}^2\right)+1\right) + \frac{64\sigma_{\max}^2+32\Sigma_{\max}^2 }{\lambda}\\
    {}&+ \frac{16L^2 D^2}{\lambda}\ln \bigg(1+8\sqrt{2}\frac{L}{\lambda}\bigg)+ \frac{16L^2 D^2 + 4G^2}{\lambda} + \frac{\lambda D^2}{4}\\
    ={}&\O\left(\frac{1}{\lambda}\left(\sigma_{\max}^2+\Sigma_{\max}^2\right)\log \left(\left(\sigma_{1:T}^2 + \Sigma_{1:T}^2\right)/\left(\sigma_{\max}^2+\Sigma_{\max}^2\right)\right)\right).
\end{align*}
\end{myThm}

\begin{table}[t]
    \centering
    \caption{Comparison of different theoretical guarantees of SEA for strongly convex functions.}
    \label{tab:str}
    \renewcommand*{\arraystretch}{1.4}
    \begin{tabular}{cc}
    \toprule
    Reference & Regret bound of SEA with $\lambda$-strongly convex functions\\
    \midrule
    \citet{OCO:Between} & $\O\big(\frac{1}{\lambda}\big(\sigma_{\max}^2 + \Sigma_{\max}^2\big) \log T\big)$ \\[1ex]
    \citet{ICML'23:OMD4SEA} & $\O\big(\min\{\frac{G^2}{\lambda}\log \big(\sigma_{1:T}^2+\Sigma_{1:T}^2\big), \frac{1}{\lambda}\big(\sigma_{\max}^2 + \Sigma_{\max}^2\big) \log T\}\big)$\\[1ex]
    This paper & $\O\left(\frac{1}{\lambda}\left(\sigma_{\max}^2+\Sigma_{\max}^2\right)\log \left(\left(\sigma_{1:T}^2 + \Sigma_{1:T}^2\right)/\left(\sigma_{\max}^2+\Sigma_{\max}^2\right)\right)\right)$\\[1ex]
    \bottomrule
    \end{tabular}
\end{table}
Table~\ref{tab:str} compares our result with those previously reported by~\citet{OCO:Between} and our earlier conference version~\citep{ICML'23:OMD4SEA}. Our result is strictly better than theirs, and we demonstrate the advantages in the following.
\begin{myRemark}
Compare to~\citet{OCO:Between}'s $\O(\frac{1}{\lambda}(\sigma_{\max}^2 + \Sigma_{\max}^2) \log T)$ bound, our result shows advantages in benign problems with small cumulative quantities $\sigma_{1:T}^2$ and $\Sigma_{1:T}^2$. Notably, even when $\sigma_{1:T}^2$ and $\Sigma_{1:T}^2$ are small, $\sigma_{\max}^2$ and $\Sigma_{\max}^2$ can be large, making their bound less effective. For instance, in an adversarial setting where $\sigma_{1:T}^2=\sigma_{\max}^2=0$ and online functions only change \emph{once} such that $\Sigma_{1:T}^2=\Sigma_{\max}^2=\O(1)$, Theorem~\ref{thm:2} yields an $\O(1)$ bound, outperforming \citet{OCO:Between}'s $\O(\log T)$ guarantee. Furthermore, our bound can imply an $\O(\frac{G^2}{\lambda}\log V_T)$ gradient-variation bound in adversarial OCO settings, whereas \citet{OCO:Between}'s bound cannot.
\remarkend
\end{myRemark}

\begin{myRemark}
Our new result surpasses the $\O(\min\{\frac{G^2}{\lambda}\log (\sigma_{1:T}^2+\Sigma_{1:T}^2), \frac{1}{\lambda}(\sigma_{\max}^2 + \Sigma_{\max}^2) \log T\})$ bound from our earlier conference version~\citep{ICML'23:OMD4SEA}. It exhibits greater adaptivity since $\O(\sigma_{\max}^2+\Sigma_{\max}^2)$ is always at most $\O(G^2)$ and $\O\left(\left(\sigma_{1:T}^2 + \Sigma_{1:T}^2\right)/\left(\sigma_{\max}^2+\Sigma_{\max}^2\right)\right)$ is always at most $\O(T)$. This improvement is due to a refined analysis --- we apply Lemma~\ref{lemma:strongly-convex-lemma} to obtain a regret bound of the $\left(\sigma_{\max}^2+\Sigma_{\max}^2\right)\log \left(\left(\sigma_{1:T}^2 + \Sigma_{1:T}^2\right)/\left(\sigma_{\max}^2+\Sigma_{\max}^2\right)\right)$ form, which is inspired by Lemma 6 of~\citet{ICML'23:OMD4SEA}. See Section~\ref{appendix:proofs-sc} for details.
\remarkend
\end{myRemark}

\begin{myRemark}
Our new upper bound in Theorem~\ref{thm:2} does not contradict with the $\Omega(\frac{1}{\lambda}(\sigma_{\max}^2 + \Sigma_{\max}^2) \log T)$ lower bound of~\citet[Theorem 8]{OCO:Between}, because their lower bound focuses on the worst-case behavior while our result is better only in certain cases.
\remarkend
\end{myRemark}

Similar to Theorem~\ref{thm:2}, we demonstrate that for strongly convex and smooth functions, optimistic FTRL can also attain the \emph{same} guarantee as optimistic OMD for the SEA model.
\begin{myThm} \label{thm:FTRL} Under Assumptions~\ref{ass:3}, \ref{ass:4}, \ref{ass:sigma}, \ref{ass:2} and \ref{ass:5}, with an appropriate setup for the optimistic FTRL (see details in~\pref{appendix:proofs-sc-FTRL}), the expected regret is at most $\O\big(\frac{1}{\lambda}\left(\sigma_{\max}^2+\Sigma_{\max}^2\right)$ $\log \left(\left(\sigma_{1:T}^2 + \Sigma_{1:T}^2\right)/\left(\sigma_{\max}^2+\Sigma_{\max}^2\right)\right)\big)$.
\end{myThm}

\subsection{Exp-concave and Smooth Functions}
\label{sec:exp-concave-results}
We further explore the SEA model for exp-concave and smooth functions. Notably,~\citet{OCO:Between} only investigate convex and strongly convex functions, without studying exp-concave functions. Our results and analysis in this part is a \emph{new} contribution.

Throughout this part, we will assume the individual functions are exp-concave rather than the expected functions, see Assumption~\ref{ass:6}. This is due to the need to use the exponential concavity of individual functions in our regret analysis.  It is common in stochastic exp-concave optimization to assume exp-concavity of individual functions~\citep{COLT:2015:Mahdavi,NIPS2015_Exp}. Importantly, we need to emphasize that the exponential concavity of individual functions \emph{does not} imply the same for expected functions, which implies that the two assumptions are incomparable.

Following \citet{Gradual:COLT:12}, we set the regularizer $\MR_t(\x) = \frac{1}{2} \|\x\|_{H_t}^2$, where $H_t = I+\frac{\beta}{2}G^2 I +\frac{\beta}{2}\sum_{s=1}^{t-1}\nabla f_{s}(\x_{s})\nabla f_{s}(\x_{s})^{\top}$, $I$ is the $d$-dimensional identity matrix, and $\beta = \frac{1}{2}\min \left\{\frac{1}{4GD},\alpha\right\}$. Then, the updating rules of optimistic OMD in \eqref{eqn:update:u} and~\eqref{eqn:update:x} become 
\begin{align} 
\xh_{t+1} =& \argmin_{\x \in \X}\langle \nabla f_t(\x_t),\x \rangle + \frac{1}{2}\|\x-\xh_t\|_{H_t}^2, \label{eq:ONS-1}\\
\x_{t+1} =& \argmin_{\x \in \X}\langle \nabla f_t(\x_t),\x \rangle + \frac{1}{2}\|\x-\xh_{t+1}\|_{H_{t+1}}^2. \label{eq:ONS-2}
\end{align}

For exp-concave and smooth functions, we can realize the following bound of optimistic OMD for the SEA model with proof in Section~\ref{appendix:proofs-exp-concave}.
\begin{myThm} \label{thm:3} Under Assumptions~\ref{ass:3}, \ref{ass:4}, \ref{ass:2} and \ref{ass:6}, optimistic OMD with updates~\eqref{eq:ONS-1} and \eqref{eq:ONS-2} enjoys the following guarantee:
\begin{align*}
\E[\mathbf{Reg}_T(\u)] 
\leq {}& \frac{16d}{\beta}\ln \left( \frac{\beta}{d} \sigma_{1:T}^2 + \frac{\beta}{2d} \Sigma_{1:T}^2 + \frac{\beta}{8d}G^2 + 1\right)+ \frac{16d}{\beta}\ln\left( 32L^2 + 1 \right)+ D^2\left(1+\frac{\beta}{2}G^2\right)\\
= {}& \O\Big(\frac{d}{\alpha}\log(\sigma_{1:T}^2+\Sigma_{1:T}^2)\Big),
\end{align*}
where $\beta = \frac{1}{2}\min \left\{\frac{1}{4GD},\alpha\right\}$, and $d$ is the dimensionality of decisions.
\end{myThm}
\begin{myRemark}
This is the \emph{first} regret bound for the SEA model with exp-concave and smooth functions. Owing to analytical differences, we are unable to attain an $\O(\frac{d}{\alpha} (\sigma_{\max}^2 + \Sigma_{\max}^2) \log T)$ regret bound, and further we can not get an $\O\big(\frac{d}{\alpha}\big(\sigma_{\max}^2+\Sigma_{\max}^2\big)\log \big(\big(\sigma_{1:T}^2 + \Sigma_{1:T}^2\big)/\big(\sigma_{\max}^2+\Sigma_{\max}^2\big)\big)\big)$ bound as in the strongly convex case (\pref{thm:2}). We will investigate this possibility in the future.
\remarkend
\end{myRemark}

Similarly, we obtain the same guarantee by optimistic FTRL in the exp-concave case.
\begin{myThm}
\label{thm:FTRL-exp-concave}
Under Assumptions~\ref{ass:3}, \ref{ass:4}, \ref{ass:2} and \ref{ass:6} with an appropriate setup for the optimistic FTRL (see details in Appendix~\ref{appendix:proofs-exp-FTRL}), the expected regret is at most $\O(\frac{d}{\alpha}\log(\sigma_{1:T}^2 +\Sigma_{1:T}^2))$.
\end{myThm}

\subsection{Analysis}
In this section, we analyze the three theoretical guarantees based on optimistic OMD. Analyses of optimistic FTRL and proofs of all lemmas used are postponed to~\pref{appendix:proofs-section4}.
\subsubsection{{Proof of Theorem~\ref{thm:1}}}
\label{appendix:proofs-convex}
\begin{proof}
Before proving Theorem~\ref{thm:1}, we present a variant of the Bregman proximal inequality lemma \citep[Lemma 3.1]{nemirovski-2005-prox}, commonly used in optimistic OMD analysis. The proof is detailed in~\pref{appendix:useful-lemmas}. 

\begin{myLemma}[Variant of Bregman proximal inequality] \label{lem:1} 
Assume $\MR_t(\cdot)$ is an $\alpha$-strongly convex function with respect to $\|\cdot\|$, and denote by $\|\cdot\|_*$ the dual norm. Based on the updating rules of optimistic OMD in (\ref{eqn:update:u}) and (\ref{eqn:update:x}), for  all $\x \in \X$ and $t\in [ T ]$,  we have
\begin{align*}
\langle \nabla f_{t}(\x_{t}), \x_{t} - \x \rangle \leq  & \frac{1}{\alpha} \|\nabla f_{t}(\x_{t}) -\nabla f_{t-1}(\x_{t-1}) \|_*^2   \\
&+\Big(\D_{\MR_t}(\x, \xh_{t})-\D_{\MR_t}(\x,\xh_{t+1}) \Big) - \Big( \D_{\MR_t}(\hat{\x}_{t+1},\x_t)+\D_{\MR_t}(\x_{t},\hat{\x}_t)\Big),
\end{align*}
where we set $\nabla f_0(\x_0)=0$.
\end{myLemma}
Given that Theorem~\ref{thm:1} performs optimistic OMD on individual functions $\{f_1,\ldots,f_T\}$, we utilize Lemma~\ref{lem:1} as $\MR_t(\x) = \frac{1}{2\eta_t} \norm{\x}_2^2$ is $\frac{1}{\eta_t}$-strongly convex with respect to $\|\cdot\|_2$ and sum the inequality over $t=1,\ldots,T$: 
\begin{align}
{}&\sum_{t=1}^T \langle \nabla f_{t}(\x_{t}), \x_{t} - \u \rangle\nonumber\\ \leq{}&  \underbrace{\sum_{t=1}^T \frac{1}{2 \eta_t}\left( \|\u -\xh_{t}\|_2^2 - \|\u -\xh_{t+1}\|_2^2 \right)}_{\term{a}} + \underbrace{\sum_{t=1}^T\eta_t \|\nabla f_{t}(\x_{t}) -\nabla f_{t-1}(\x_{t-1}) \|_2^2}_{\term{b}}\nonumber \\
{}&- \underbrace{\sum_{t=1}^T \frac{1}{2 \eta_t} \big( \|\x_{t}- \xh_{t} \|_2^2 + \| \xh_{t+1} - \x_{t} \|_2^2 \big)}_{\term{c}} . \label{eqn:sum:cov1}
\end{align}
In the following, we will bound the three terms on the right hand respectively. 

First, given $\eta_t = D/\sqrt{\delta + 4G^2 + \bar{V}_{t-1}}$ and $\bar{V}_{t-1} = \sum_{s=1}^{t-1}\|\nabla f_s(\x_s) - \nabla f_{s-1}(\x_{s-1})\|_2^2$, we derive that $\eta_t  \leq D/\sqrt{\delta+ \bar{V}_t}$ using Assumption~\ref{ass:3} (boundedness of gradient norms). For term (a), by the fact $\eta_t \leq \eta_{t-1}$ and Assumption~\ref{ass:4} (domain boundedness), we have
\begin{align*}
\term{a}={}& \frac{1}{2 \eta_1} \|\u -\xh_{1}\|_2^2 +  \frac{1}{2} \sum_{t=2}^T \left( \frac{1}{ \eta_t}  - \frac{1}{ \eta_{t-1}} \right) \|\u -\xh_{t}\|_2^2 - \frac{1}{2 \eta_T} \|\u -\xh_{T+1}\|_2^2 \\
 \leq{}&  \frac{1}{2 \eta_1} D^2 +  \frac{1}{2} \sum_{t=2}^T \left( \frac{1}{ \eta_t}  - \frac{1}{ \eta_{t-1}} \right) D^2 = \frac{D^2}{2 \eta_T} =   \frac{D}{2} \sqrt{\delta+4G^2 + \Bar{V}_{T-1}}.
\end{align*}
For term (b), we utilize Lemma~\ref{lem:sum} to bound it as
\begin{align*}
\term{b} \leq  \sum_{t=1}^T \frac{D}{\sqrt{\delta+ \Bar{V}_{t}}} \|\nabla f_{t}(\x_{t}) -\nabla f_{t-1}(\x_{t-1}) \|_2^2 \leq  2 D \sqrt{\delta+\Bar{V}_T }. 
\end{align*}
For term (c), we rely on the fact that $\eta_t \leq \frac{D}{\sqrt{\delta}}$:
\begin{align*}
\term{c}={}&  \sum_{t=1}^T \frac{1}{2 \eta_t} \big( \|\x_{t}- \xh_{t} \|_2^2 + \| \xh_{t+1} - \x_{t} \|_2^2 \big) \geq  \frac{\sqrt{\delta}}{2D} \sum_{t=1}^T  \big( \|\x_{t}- \xh_{t} \|_2^2 + \| \xh_{t+1} - \x_{t} \|_2^2 \big) \\
\geq{}& \frac{\sqrt{\delta}}{2D}  \sum_{t=2}^T  \big( \|\x_{t}- \xh_{t} \|_2^2 +   \| \xh_{t} - \x_{t-1} \|_2^2 \big) \geq  \frac{\sqrt{\delta}}{4D} \sum_{t=2}^T   \|\x_{t}-  \x_{t-1}\|_2^2. 
\end{align*}
Then we substitute the three bounds above into \eqref{eqn:sum:cov1} and use Assumption~\ref{ass:3} to get
\begin{align*}
 \sum_{t=1}^T \langle \nabla f_{t}(\x_{t}), \x_{t} - \u  \rangle 
 \leq \frac{5 D}{2} \sqrt{\delta+4G^2 + \Bar{V}_{T-1}} - \frac{\sqrt{\delta}}{4D} \sum_{t=2}^T   \|\x_{t}-  \x_{t-1}\|_2^2,
\end{align*}
In order to bound the $\Bar{V}_{T-1}$ term, we incorporate a crucial lemma extracted from the analysis of~\citet{OCO:Between}. Refer to~\pref{appendix:useful-lemmas} for the proof.
\begin{myLemma}[Boundedness of cumulative norm of gradient difference (\citet{OCO:Between}, Analysis of Theorem 5)]
\label{lem:sum:diff} 
Under Assumptions~\ref{ass:3} and~\ref{ass:2}, we have
\begin{equation} \label{eqn:sum:diff}
\begin{split}
\sum_{t=1}^{T} \|\nabla f_t(\x_t) - \nabla f_{t-1} (\x_{t-1}) \|_2^2 \leq  &G^2  + 4L^2 \sum_{t=2}^{T} \|\x_t-\x_{t-1}\|_2^2 \\
+ 8 \sum_{t=1}^{T}  \|\nabla f_t(\x_t) &-\nabla F_t(\x_t)\|_2^2  + 4 \sum_{t=2}^{T} \|\nabla F_t(\x_{t-1}) -\nabla F_{t-1}(\x_{t-1})\|_2^2.
\end{split}
\end{equation}
\end{myLemma}
As a result, by applying Lemma~\ref{lem:sum:diff}, we have
\begin{align*}
    {}&\sum_{t=1}^T \langle \nabla f_{t}(\x_{t}), \x_{t} - \u  \rangle\\
    \leq{}&\frac{5 D}{2} \sqrt{\delta+5G^2} + 5\sqrt{2}D\sqrt{\sum_{t=1}^T\left\|\nabla f_t(\x_t)-\nabla F_t(\x_t)\right\|_2^2}+5DL\sqrt{\sum_{t=2}^T\left\|\x_t-\x_{t-1}\right\|_2^2}\\
    {}&+5D\sqrt{\sum_{t=2}^T\left\|\nabla F_t(\x_{t-1})-\nabla F_{t-1}(\x_{t-1})\right\|_2^2}- \frac{\sqrt{\delta}}{4D} \sum_{t=2}^T   \|\x_{t}-  \x_{t-1}\|_2^2\\
    \leq{}&\frac{5 D}{2} \sqrt{\delta+5G^2} + \frac{25D^3L^2}{\sqrt{\delta}}+5\sqrt{2}D\sqrt{\sum_{t=1}^T\left\|\nabla f_t(\x_t)-\nabla F_t(\x_t)\right\|_2^2}\\
    {}&+5D\sqrt{\sum_{t=2}^T\left\|\nabla F_t(\x_{t-1})-\nabla F_{t-1}(\x_{t-1})\right\|_2^2}
\end{align*}
where the second step uses AM-GM inequality as $5D L\sqrt{ \sum_{t=2}^{T} \|\x_t-\x_{t-1}\|_2^2} \leq  \frac{25 D^3 L^2}{\sqrt{\delta}} + \frac{\sqrt{\delta}}{4D}  \sum_{t=2}^T  \|\x_{t}-  \x_{t-1}\|_2^2$. Taking expectations and applying Jensen's inequality lead to 
\begin{align*}
\E\left[ \sum_{t=1}^T \langle \nabla f_{t}(\x_{t}), \x_{t} - \u \rangle \right] \leq {}& \frac{5 D}{2} \sqrt{\delta} +  \frac{25 D^3 L^2}{\sqrt{\delta}} +  \frac{5\sqrt{5}  DG}{2} + 5 \sqrt{2}D\sqrt{ \sigma_{1:T}^2 } + 5 D \sqrt{\Sigma_{1:T}^2 } \\
={}& 5 \sqrt{10} D^2 L +  \frac{5\sqrt{5}  DG}{2} + 5 \sqrt{2}D\sqrt{ \sigma_{1:T}^2 } + 5 D \sqrt{\Sigma_{1:T}^2 }\\
= {}& \O\left(\sqrt{\sigma_{1:T}^2} + \sqrt{\Sigma_{1:T}^2}\right),
\end{align*}
where we set $\delta=10 D^2 L^2$ and recall definitions of $\sigma_{1:T}^2$ in~\eqref{eqn:cumu:variance} and $\Sigma_{1:T}^2$ in\eqref{eqn:cumu:varia}. We end the proof by noting the expectation upper-bounds the expected regret as in \eqref{eqn:exp:cov}.
\end{proof}

\subsubsection{{Proof of Theorem~\ref{thm:2}}}
\label{appendix:proofs-sc}
\begin{proof}
Since the expected functions are $\lambda$-strongly convex now, we have $F_t(\x_t) - F_t(\u) \leq  \langle \nabla F_t(\x_t), \x_t -\u \rangle -\frac{\lambda}{2}\|\u - \x_t\|_2^2$. Then by the definition $F_t(\x) = \E_{f_t\sim \mathfrak{D}_t}[f_t(\x)]$, we obtain
\begin{align}
{}&\E\left[ \sum_{t=1}^T f_t(\x_t) - \sum_{t=1}^T f_t(\u) \right] = \E\left[ \sum_{t=1}^T F_t(\x_t) - \sum_{t=1}^T F_t(\u) \right]\label{eqn:exp:str}\\
\leq{}& \E\left[ \sum_{t=1}^T \big( \langle \nabla F_t(\x_t), \x_t -\u \rangle -\frac{\lambda}{2}\|\u - \x_t\|_2^2 \big)\right] = \E\left[ \sum_{t=1}^T \big( \langle \nabla f_t(\x_t), \x_t -\u \rangle -\frac{\lambda}{2}\|\u - \x_t\|_2^2\big) \right]. \nonumber
\end{align}
Similar to the analysis of Theorem~\ref{thm:1}, we have the following regret upper bound,
\begin{align}
 {}&\sum_{t=1}^T  \langle \nabla f_t(\x_t), \x_t -\u \rangle -\frac{\lambda}{2}\sum_{t=1}^T \|\u - \x_t\|_2^2 \nonumber\\
 \leq{}&  \underbrace{\sum_{t=1}^T \left(\frac{1}{2 \eta_t} \|\u -\xh_{t}\|_2^2 - \frac{1}{2 \eta_t} \|\u -\xh_{t+1}\|_2^2 \right)-\frac{\lambda}{2}\sum_{t=1}^T \|\u - \x_t\|_2^2}_{\term{a}}\label{eqn:sum:str1}\\
{}&+ \underbrace{\sum_{t=1}^T\eta_t \|\nabla f_{t}(\x_{t}) -\nabla f_{t-1}(\x_{t-1}) \|_2^2}_{\term{b}} - \underbrace{\sum_{t=1}^T \frac{1}{2 \eta_t} \big( \|\x_{t}- \xh_{t} \|_2^2 + \| \xh_{t+1} - \x_{t} \|_2^2 \big)}_{\term{c}}.\nonumber
\end{align}
We then provide the upper bounds of term (a), term (b), and term (c) respectively. 

To bound term (a), we need the following classic lemma.
\begin{myLemma}[{Stability lemma~\citep[Proposition 7]{Gradual:COLT:12}}]
\label{lem:stability}
Consider the following two updates: (i) $\x_{*}=\argmin_{\x\in\X}\langle\mathbf{a},\x\rangle+\D_{\MR}(\x,\mathbf{c})$, and (ii) $\x_{*}'=\argmin_{\x\in\X}\langle\mathbf{a}',\x\rangle+\D_{\MR}(\x,\mathbf{c})$. When the regularizer $\psi:\X\rightarrow\mathbbm{R}$ is a 1-strongly convex function with respect to the norm $\|\cdot\|$, we have $\|\x_{*}-\x_{*}'\|\leq\|(\nabla\MR (\mathbf{c})-\mathbf{a})-(\nabla\MR (\mathbf{c})-\mathbf{a}')\|_{*}=\|\mathbf{a}-\mathbf{a}'\|_{*}$.
\end{myLemma}
Using Lemma~\ref{lem:stability} with our algorithm yields $\|\xh_{t+1} -\x_{t} \|_2 \leq \eta_t \| \nabla f_{t}(\x_{t}) - \nabla f_{t-1}(\x_{t-1}) \|_2$. Considering Assumption~\ref{ass:4} (domain boundedness) and the step size $\eta_t = \frac{2}{\lambda t}$, we obtain
\begin{align*}
    \term{a} \leq{}&  \frac{1}{2 \eta_1} D^2 +  \frac{1}{2} \sum_{t=2}^T \left( \frac{1}{ \eta_t}  - \frac{1}{ \eta_{t-1}} \right) \|\u -\xh_{t}\|_2^2  -\frac{\lambda}{2}\sum_{t=1}^T \|\u - \x_t\|_2^2\\
\leq{}& \frac{\lambda D^2}{4} + \frac{\lambda}{4}\sum_{t=1}^{T-1} \left(\|\u - \xh_{t+1}\|_2^2- 2\|\u - \x_t\|_2^2 \right)\leq \frac{\lambda D^2}{4} + \frac{\lambda}{2}\sum_{t=1}^{T-1} \|\xh_{t+1} - \x_t\|_2^2\\
 \leq{}& \frac{\lambda D^2}{4} + \frac{\lambda \eta_1}{2}\sum_{t=1}^{T-1} \eta_t \| \nabla f_t(\x_t) - \nabla f_{t-1}(\x_{t-1})\|_2^2\leq \frac{\lambda D^2}{4} + \term{b} ,
\end{align*}
where the last step is based on $\eta_t$ being non-increasing. This shows that the upper bound of term (a) depends on term (b). For term (b), after inserting the definition of $\eta_t$, we get
\begin{align*}
  \term{b}= 2\sum_{t=1}^T \frac{1}{\lambda t}\| \nabla f_t(\x_t) - \nabla f_{t-1}(\x_{t-1})\|_2^2.
\end{align*}
Making use of the fact that $\eta_t$ is non-increasing again, we bound term (c) by
\begin{align*}
    \term{c}\geq \sum_{t=2}^T \bigg(\frac{1}{2\eta_t} \|\x_{t}- \xh_{t} \|_2^2 + \frac{1}{2\eta_{t-1}}\| \xh_{t} - \x_{t-1} \|_2^2\bigg)
    \geq  \sum_{t=2}^T \frac{1}{4\eta_{t-1}}\| \x_t - \x_{t-1}\|_2^2.
\end{align*}
Combining the upper bounds of term (a), term (b) and term (c) into~\eqref{eqn:sum:str1} with $\eta_t = \frac{2}{\lambda t}$ gives
\begin{align*}
    {}& \sum_{t=1}^T \langle  \nabla f_{t}(\x_{t}), \x_{t} - \x \rangle -\frac{\lambda}{2}\sum_{t=1}^T \|\x - \x_t\|_2^2 \\
    \leq{}& \frac{\lambda D^2}{4} + 4\sum_{t=1}^T \frac{1}{\lambda t} \| \nabla f_t(\x_t) - \nabla f_{t-1}(\x_{t-1})\|_2^2 - \sum_{t=2}^T \frac{\lambda (t-1)}{8}\| \x_t - \x_{t-1}\|_2^2 . 
\end{align*}

Then we need to use the following lemma with its proof in~\pref{appendix:useful-lemmas}.
\begin{myLemma} [Boundedness of the norm of gradient difference (\citet{OCO:Between}, Analysis of Theorem 5)]\label{lem:ind:diff} Under Assumptions~\ref{ass:2} and \ref{ass:3}, we have
\begin{align*}
 \|\nabla f_t(\x_t) - \nabla f_{t-1} (\x_{t-1}) \|_2^2 \leq & 4 \|\nabla f_t(\x_t) -\nabla F_t(\x_t)\|_2^2 + 4 \|\nabla F_t(\x_{t-1}) -\nabla F_{t-1}(\x_{t-1})\|_2^2 \\
 & +  4L^2 \|\x_t-\x_{t-1}\|_2^2+ 4 \|\nabla F_{t-1}(\x_{t-1}) -  \nabla f_{t-1} (\x_{t-1})\|_2^2,
\end{align*}
where $\|\nabla f_1(\x_1) - \nabla f_{0} (\x_{0}) \|_2^2=\|\nabla f_1(\x_1)  \|_2^2 \leq G^2$.
\end{myLemma}
So applying Lemma~\ref{lem:ind:diff} yields the following result,
\begin{align}
    {}& \sum_{t=1}^T \langle \nabla f_{t}(\x_{t}), \x_{t} - \u \rangle -\frac{\lambda}{2}\sum_{t=1}^T \|\u - \x_t\|_2^2 \nonumber\\
    \leq{}& \frac{4G^2}{\lambda} + 4\sum_{t=2}^T \frac{1}{\lambda t} \left( 4 \|\nabla f_t(\x_t) -\nabla F_t(\x_t)\|_2^2 + 4 \|\nabla F_t(\x_{t-1}) -\nabla F_{t-1}(\x_{t-1})\|_2^2 \right.\nonumber\\
    {}&\left.+ 4 \|\nabla F_{t-1}(\x_{t-1}) -  \nabla f_{t-1} (\x_{t-1})\|_2^2 \right) + \sum_{t=2}^{T} \left(\frac{16L^2}{\lambda t} - \frac{\lambda (t-1)}{8}\right) \|\x_t-\x_{t-1}\|_2^2 + \frac{\lambda D^2}{4}\nonumber\\
    \leq{}& \frac{4G^2}{\lambda} + \sum_{t=2}^T \frac{16}{\lambda t} \|\nabla F_{t}(\x_{t}) -  \nabla f_{t} (\x_{t})\|_2^2 +\sum_{t=2}^{T} \frac{16}{\lambda t}\|\nabla F_t(\x_{t-1}) -\nabla F_{t-1}(\x_{t-1})\|_2^2 \nonumber\\
    {}&+ \sum_{t=2}^T \frac{16}{\lambda (t-1)} \|\nabla F_{t-1}(\x_{t-1}) -  \nabla f_{t-1} (\x_{t-1})\|_2^2 + \sum_{t=1}^{T-1} \left(\frac{16L^2}{\lambda t} - \frac{\lambda t}{8}\right) \|\x_{t+1}-\x_{t}\|_2^2 + \frac{\lambda D^2}{4}\nonumber\\
    \leq{}& \frac{4G^2}{\lambda} + \sum_{t=1}^{T} \frac{32}{\lambda t}\|\nabla f_t(\x_t) -\nabla F_t(\x_t)\|_2^2 +\sum_{t=2}^{T} \frac{16}{\lambda t}\|\nabla F_t(\x_{t-1}) -\nabla F_{t-1}(\x_{t-1})\|_2^2\nonumber\\
    {}& + \sum_{t=1}^{T-1} \left(\frac{16L^2}{\lambda t} - \frac{\lambda t}{8}\right) \|\x_{t+1}-\x_{t}\|_2^2+ \frac{\lambda D^2}{4}.\label{lem4:2}
\end{align}
Following \citet{OCO:Between}, we define $\kappa = \frac{L}{\lambda}$. Then for $t \geq 8\sqrt{2}\kappa$, we have $\frac{16L^2}{\lambda t} - \frac{\lambda t}{8} \leq 0$. Using Assumption~\ref{ass:4} (domain boundedness), the fourth term above is bounded as
\begin{align*}
     {}&\sum_{t=1}^{T-1} \left(\frac{16L^2}{\lambda t} - \frac{\lambda t}{8}\right) \|\x_{t+1}-\x_{t}\|_2^2
     \leq  \sum_{t=1}^{\lceil 8\sqrt{2}\kappa \rceil} \left(\frac{16L^2}{\lambda t} - \frac{\lambda t}{8}\right)D^2 \leq \frac{16L^2 D^2}{\lambda}\sum_{t=1}^{\lceil 8\sqrt{2}\kappa \rceil}\frac{1}{t}\\
     \leq{}&  \frac{16L^2 D^2}{\lambda}\left( 1+\int_{t=1}^{\lceil 8\sqrt{2}\kappa \rceil} \frac{1}{t} dt\right) 
     =  \frac{16L^2 D^2}{\lambda}\ln\left(1+8\sqrt{2}\frac{L}{\lambda}\right) + \frac{16L^2 D^2}{\lambda} .
\end{align*}
Combining the above two formulas and taking the expectation, we can get that
\begin{align*}
    {}&\E\left[\sum_{t=1}^T \langle \nabla f_{t}(\x_{t}), \x_{t} - \u,    \rangle -\frac{\lambda}{2}\sum_{t=1}^T \|\u - \x_t\|_2^2\right]\\
    \leq{}&\E\left[\sum_{t=1}^{T} \frac{32}{\lambda t}\sigma_t^2 +\sum_{t=2}^{T} \frac{16}{\lambda t}\sup_{\x\in \X} \|\nabla F_t(\x) -\nabla F_{t-1}(\x)\|_2^2 \right] + \frac{16L^2 D^2}{\lambda}\ln \bigg(1+8\sqrt{2}\frac{L}{\lambda}\bigg) \\
    {}&+ \frac{16L^2 D^2 + 4G^2}{\lambda} + \frac{\lambda D^2}{4},
\end{align*}
where $\sigma_t^2 = \max_{\x\in\X}\E_{f_t\sim \mathfrak{D}_t}\left[ \|\nabla f_t(\x)-\nabla F_t(\x)\|_2^2 \right]$ as defined in~\eqref{eqn:variance:grad}. To deal with the first term, we introduce a new lemma below, with its proof in~\pref{appendix:useful-lemmas}. 
\begin{myLemma}
\label{lemma:strongly-convex-lemma}
    Under Assumption~\ref{ass:sigma}, we have
\begin{align*}
    {}&\sum_{t=1}^T\frac{1}{\lambda t}\left(2\sigma_t^2 + \sup_{\x\in \X} \|\nabla F_t(\x) -\nabla F_{t-1}(\x)\|_2^2\right) \\
    \leq{}& \frac{2\sigma_{\max}^2+\Sigma_{\max}^2}{\lambda}\ln \left(\sum_{t=1}^T\frac{1}{2\sigma_{\max}^2+\Sigma_{\max}^2}\left(2\sigma_t^2 + \sup_{\x\in \X} \|\nabla F_t(\x) -\nabla F_{t-1}(\x)\|_2^2\right)+1\right) + \frac{4\sigma_{\max}^2+2\Sigma_{\max}^2}{\lambda}.
\end{align*}
\end{myLemma}
Then, we can arrive at
\begin{align*}
    {}& \E \left[\sum_{t=1}^T \langle \nabla f_{t}(\x_{t}), \x_{t} - \u    \rangle -\frac{\lambda}{2}\sum_{t=1}^T \|\u - \x_t\|_2^2 \right]\\
    \leq{}&\frac{32\sigma_{\max}^2+16\Sigma_{\max}^2}{\lambda}\ln \left(\frac{1}{2\sigma_{\max}^2+\Sigma_{\max}^2}\left(2\sigma_{1:T}^2 + \Sigma_{1:T}^2\right)+1\right) + \frac{64\sigma_{\max}^2+32\Sigma_{\max}^2 }{\lambda}\\
    {}&+ \frac{16L^2 D^2}{\lambda}\ln \bigg(1+8\sqrt{2}\frac{L}{\lambda}\bigg)+ \frac{16L^2 D^2 + 4G^2}{\lambda} + \frac{\lambda D^2}{4}\\
    ={}& \O\left(\frac{1}{\lambda}\left(\sigma_{\max}^2+\Sigma_{\max}^2\right)\log \left(\left(\sigma_{1:T}^2 + \Sigma_{1:T}^2\right)/\left(\sigma_{\max}^2+\Sigma_{\max}^2\right)\right)\right).
\end{align*}
This ends the proof.
\end{proof}

\subsubsection{{Proof of Theorem~\ref{thm:3}}}
\label{appendix:proofs-exp-concave}
\begin{proof}
Due to the exp-concavity assumption, we have $f_t(\x_t) - f_t(\u) \leq \langle \nabla f_t(\x_t), \x_t -\u \rangle -\frac{\beta}{2} \|\u - \x_t\|_{h_t}^2$, where $\beta = \frac{1}{2}\min \left\{\frac{1}{4GD},\alpha\right\}$, and $h_t = \nabla f_{t}(\x_{t})\nabla f_{t}(\x_{t})^{\top}$. Therefore, we can take advantage of the above formula to get tighter regret bounds as follows 
\begin{align}
\E\left[ \sum_{t=1}^T f_t(\x_t) - \sum_{t=1}^T f_t(\u) \right] \leq \E\left[ \sum_{t=1}^T  \langle \nabla f_t(\x_t), \x_t -\u \rangle -\frac{\beta}{2}\sum_{t=1}^T \|\u - \x_t\|_{h_t}^2 \right]. \label{eqn:exp:exp}
\end{align}
Clearly, $\MR_t(\x) = \frac{1}{2}\|\x\|_{H_t}^2$ is a 1-strongly convex function with respect to $\|\cdot\|_{H_t}$, and $\|\cdot\|_{H_t}^{-1}$ is the dual norm. Thus, from Lemma~\ref{lem:1} (Variant of Bregman proximal inequality), we have
\begin{align}
{}& \sum_{t=1}^T \langle \x_{t} - \u,   \nabla f_{t}(\x_{t}) \rangle -\frac{\beta}{2}\sum_{t=1}^T \|\u - \x_t\|_{h_t}^2 \nonumber\\
\leq{}&  \underbrace{\sum_{t=1}^T \left(\frac{1}{2} \|\u -\xh_{t}\|_{H_t}^2 - \frac{1}{2} \|\u -\xh_{t+1}\|_{H_t}^2 \right)-\frac{\beta}{2}\sum_{t=1}^T \|\u - \x_t\|_{h_t}^2}_{\term{a}}\nonumber\\
{}&+ \underbrace{\sum_{t=1}^T \|\nabla f_{t}(\x_{t}) -\nabla f_{t-1}(\x_{t-1}) \|_{H_t^{-1}}^2}_{\term{b}}- \underbrace{\sum_{t=1}^T \frac{1}{2} \left( \|\x_{t}- \xh_{t} \|_{H_t}^2 + \| \xh_{t+1} - \x_{t} \|_{H_t}^2 \right)}_{\term{c}} .\label{eqn:sum:exp}
\end{align}
Then, we discuss the upper bounds of term (a), term (b) and term (c), respectively. According to \citet[Proof of Lemma~14]{Gradual:COLT:12}, we write term (a) as 
\begin{align*}
    \frac{1}{2} \Big( \| \u - \xh_1 \|_{H_1}^2 - \| \u - \xh_{T+1} \|_{H_{T+1}}^2 + \sum_{t=1}^T\big( \| \u -\xh_{t+1} \|_{H_{t+1}}^2 - \| \u - \xh_{t+1} \|_{H_t}^2 \big) \Big) - \frac{\beta}{2}\sum_{t=1}^T\| \u-\x_t\|_{h_t}^2.
\end{align*}
Based on Assumption~\ref{ass:4} (domain boundedness) and Assumption~\ref{ass:3} (boundedness of gradient norms), with the definition that $H_t = I + \frac{\beta}{2}G^2 I + \frac{\beta}{2}\sum_{\tau = 1}^{t-1}\nabla f_{\tau}(\x_{\tau})\nabla f_{\tau}(\x_{\tau})^{\top}$ and $h_t = \nabla f_{t}(\x_{t})\nabla f_{t}(\x_{t})^{\top}$, we have $\| \u - \xh_1 \|_{H_1}^2 \leq D^2\big(1+\frac{\beta}{2}G^2\big)$ and $H_{t+1} - H_{t} = \frac{\beta}{2}h_t$ for every $t$. So we can simplify term (a) to
\begin{align*}
    \term{a} {}&\leq  \frac{D^2}{2}\left(1+\frac{\beta}{2}G^2\right) + \frac{\beta}{4}\sum_{t=1}^T\| \u -\xh_{t+1}\|_{h_t}^2  - \frac{\beta}{2}\sum_{t=1}^T\| \u-\x_t\|_{h_t}^2\\
    {}&\leq  \frac{D^2}{2}\left(1+\frac{\beta}{2}G^2\right)+ \frac{\beta}{2}\sum_{t=1}^T\|\x_t -\xh_{t+1}\|_{h_t}^2\leq  \frac{D^2}{2}\left(1+\frac{\beta}{2}G^2\right)+ \sum_{t=1}^T\|\x_t - \xh_{t+1}\|_{H_t}^2\\
    {}&\leq  \frac{D^2}{2}\left(1+\frac{\beta}{2}G^2\right)+ \sum_{t=1}^T\|\nabla f_t(\x_t) - \nabla f_{t-1}(\x_{t-1}) \|_{H_t^{-1}}^2=  \frac{D^2}{2}\left(1+\frac{\beta}{2}G^2\right)+ \term{b} ,
\end{align*}
where we use $H_{t}\succeq \frac{\beta}{2}G^2 I \succeq \frac{\beta}{2}h_t$ for the third inequality and Lemma~\ref{lem:stability} (Stability lemma) in the fourth inequality. Notably, the upper bound of term (b) determines that of term (a). Hence we move to bound term (b). By definition of $H_t$, there is $G^2 I \succeq \nabla f_{t}(\x_{t})\nabla f_{t}(\x_{t})^{\top}$ for every $t$. In addition, we know $\nabla f_{0}(\x_{0}) = 0$, so
\begin{align}
    H_t \succeq I+\frac{\beta}{4}\sum_{\tau=1}^{t}\left( \nabla f_{\tau}(\x_{\tau})\nabla f_{\tau}(\x_{\tau})^{\top} + \nabla f_{\tau-1}(\x_{\tau-1})\nabla f_{\tau-1}(\x_{\tau-1})^{\top}\right).\label{eqn:exp:Ht}
\end{align}
Similar to~\citet{Gradual:COLT:12}, we claim that
\begin{align}
    \nabla f_{\tau}(\x_{\tau})\nabla f_{\tau}(\x_{\tau})^{\top} & + \nabla f_{\tau-1}(\x_{\tau-1})\nabla f_{\tau-1}(\x_{\tau-1})^{\top}\nonumber\\ & \succeq \frac{1}{2}\left(\nabla f_{\tau}(\x_{\tau}) -\nabla f_{\tau-1}(\x_{\tau-1})\right)\left(\nabla f_{\tau}(\x_{\tau}) -\nabla f_{\tau-1}(\x_{\tau-1})\right)^{\top}. \label{eqn:exp:Ht:nabla}
\end{align}
The above inequality comes from subtracting the RHS of it from the left and getting that $\frac{1}{2}\left(\nabla f_{\tau}(\x_{\tau}) +\nabla f_{\tau-1}(\x_{\tau-1})\right)\left(\nabla f_{\tau}(\x_{\tau}) +\nabla f_{\tau-1}(\x_{\tau-1})\right)^{\top} \succeq 0$. Based on this, we obtain 
\begin{align}
    H_t \overset{(\ref{eqn:exp:Ht:nabla})}{\succeq} I+\frac{\beta}{8}\sum_{\tau=1}^{t}\left(\nabla f_{\tau}(\x_{\tau}) -\nabla f_{\tau-1}(\x_{\tau-1})\right)\left(\nabla f_{\tau}(\x_{\tau}) -\nabla f_{\tau-1}(\x_{\tau-1})\right)^{\top}.  \nonumber
\end{align}
Let $P_t = I+\frac{\beta}{8}\sum_{\tau=1}^{t}\left(\nabla f_{\tau}(\x_{\tau}) -\nabla f_{\tau-1}(\x_{\tau-1})\right)\left(\nabla f_{\tau}(\x_{\tau}) -\nabla f_{\tau-1}(\x_{\tau-1})\right)^{\top}$, we have
\begin{align*}
   \term{b} \leq \sum_{t=1}^T\|\nabla f_t(\x_t) - \nabla f_{t-1}(\x_{t-1}) \|_{P_t^{-1}}^2 ={}& \frac{8}{\beta}\sum_{t=1}^T\left\| \sqrt{\frac{\beta}{8}}\left(\nabla f_t(\x_t) - \nabla f_{t-1}(\x_{t-1}) \right)\right\|_{P_t^{-1}}^2\\
   \leq{}& \frac{8d}{\beta}\ln \left( \frac{\beta}{8d}\Bar{V}_T + 1  \right),
\end{align*}
where we apply Lemma~\ref{lem:hazan} with $\u_t =\sqrt{\frac{\beta}{8}}\left( \nabla f_t(\x_t) - \nabla f_{t-1}(\x_{t-1})\right)$ and $\epsilon = 1$.

Then, derived from the fact that $H_t \succeq H_{t-1} \succeq I$, we can bound term (c) as  
\begin{align}
    \term{c} ={}& \frac{1}{2}\sum_{t=1}^T\|\x_t - \xh_t \|_{H_t}^2 + \frac{1}{2}\sum_{t=2}^{T+1}\|\x_{t-1}-\xh_t \|_{H_{t-1}}^2\nonumber\\
     \geq{}& \frac{1}{2}\sum_{t=2}^T\|\x_t - \xh_t \|_{H_{t-1}}^2 +  \frac{1}{2}\sum_{t=2}^{T}\|\x_{t-1}-\xh_t \|_{H_{t-1}}^2
     \geq  \frac{1}{4}\sum_{t=2}^T\| \x_t -\x_{t-1} \|_2^2. \nonumber
\end{align}

Combining the above bounds of term (a), term (b) and term (c), we can get
\begin{align}
    {}&\sum_{t=1}^T \langle \x_{t} - \u,   \nabla f_{t}(\x_{t}) \rangle -\frac{\beta}{2}\sum_{t=1}^T \|\u - \x_t\|_{h_t}^2 \nonumber\\
    \leq{}& \frac{16d}{\beta}\ln \left( \frac{\beta}{8d}\Bar{V}_T + 1 \right) + \frac{D^2}{2}\left(1+\frac{\beta}{2}G^2\right)- \frac{1}{4}\sum_{t=2}^T\| \x_t -\x_{t-1} \|_2^2 .\nonumber
\end{align}
Further exploiting Lemma~\ref{lem:sum:diff} (Boundedness of cumulative norm of gradient difference) with the inequality $\ln(1+u+v) \leq \ln(1+u) + \ln(1+v) (u,v > 0)$, we have
\begin{align}
    {}&\sum_{t=1}^T \langle \x_{t} - \u,   \nabla f_{t}(\x_{t}) \rangle -\frac{\beta}{2}\sum_{t=1}^T \|\u - \x_t\|_{h_t}^2 \nonumber\\
    \leq{}&\frac{16d}{\beta}\ln\left(\frac{\beta}{d} \sum_{t=1}^{T}  \|\nabla f_t(\x_t) -\nabla F_t(\x_t)\|_2^2 +\frac{\beta}{2d} \sum_{t=2}^{T} \|\nabla F_t(\x_{t-1}) -\nabla F_{t-1}(\x_{t-1})\|_2^2+ \frac{\beta}{8d}G^2 + 1 \right) \nonumber\\
    {}& + \frac{16d}{\beta}\ln \left(\frac{\beta L^2}{2d} \sum_{t=2}^{T} \|\x_t-\x_{t-1}\|_2^2+1\right) + \frac{D^2}{2}\left(1+\frac{\beta}{2}G^2\right) - \frac{1}{4}\sum_{t=2}^T\| \x_t -\x_{t-1} \|_2^2 \nonumber\\
    \leq{}& \frac{16d}{\beta}\ln\left(\frac{\beta}{d} \sum_{t=1}^{T}  \|\nabla f_t(\x_t) -\nabla F_t(\x_t)\|_2^2 +\frac{\beta}{2d} \sum_{t=2}^{T} \|\nabla F_t(\x_{t-1}) -\nabla F_{t-1}(\x_{t-1})\|_2^2+ \frac{\beta}{8d}G^2 + 1 \right)\nonumber\\
    {}& + \frac{16d}{\beta}\ln\left( 32L^2 + 1 \right) + \frac{D^2}{2}\left(1+\frac{\beta}{2}G^2\right),  \nonumber
\end{align}
where the last step is due to Lemma~\ref{lem:ln:pq}. 

Taking the expectation, and making use of Jensen's inequality, the above bound becomes
\begin{align*}
&\E\left[\sum_{t=1}^T  \langle \nabla f_t(\x_t), \x_t -\u \rangle -\frac{\beta}{2}\sum_{t=1}^T \|\u- \x_t\|_{h_t}^2 \right]  \\
\leq & \frac{16d}{\beta}\ln \left( \frac{\beta}{d} \sigma_{1:T}^2 + \frac{\beta}{2d} \Sigma_{1:T}^2 + \frac{\beta}{8d}G^2 + 1\right)+ \frac{16d}{\beta}\ln\left( 32L^2 + 1 \right)+ \frac{D^2}{2}\left(1+\frac{\beta}{2}G^2\right)\\
= & \O\Big(\frac{d}{\alpha}\log (\sigma_{1:T}^2+ \Sigma_{1:T}^2)\Big)
\end{align*}
We finish the proof by integrating the above inequality to~\eqref{eqn:exp:exp}.
\end{proof}

%% file: sections/extensions.tex

\section{Extensions: Dynamic Regret Minimization and Non-smooth Functions}\label{sec:extensions}
In this section, we investigate a new measure for the SEA model -- dynamic regret, a more suitable metric for non-stationary environments. Subsequently, we explore the SEA model for non-smooth loss functions, proposing algorithms for minimizing static regret and dynamic regret respectively. Detailed analysis and proofs are placed in Section~\ref{subsec:extensions:analysis}.

\input{sections/dynamic_regret.tex}

\input{sections/non_smooth}

\subsection{Analysis}\label{subsec:extensions:analysis}
In this section, we give the analysis of Theorem~\ref{thm:sword}, Theorem~\ref{thm:nonsmooth:convex} and Theorem~\ref{thm:nonsmooth:dynamic} respectively, with some supplementary analysis and useful lemmas provided in~\pref{appendix:sec-extensions}.
\input{appendixes/appendix_dynamic_regret.tex}

\input{appendixes/appendix_non_smooth.tex}

%% file: sections/dynamic_regret.tex

\subsection{Dynamic Regret Minimization}
\label{sec:dynamci-regret}
To optimize the expected dynamic regret in~\eqref{eq:dynamic-regret-measure}, following the recent studies of non-stationary online learning~\citep{Adaptive:Dynamic:Regret:NIPS,Problem:Dynamic:Regret}, we develop a two-layer approach based on the optimistic OMD framework, which consists of a meta-learner running over a group of base-learners. The full procedure is summarized in Algorithm~\ref{alg:Sword++forSEA}. Specifically, we maintain a pool for candidate step sizes $\H = \{\eta_i=c\cdot 2^i \mid i \in [N]\}$, where $N$ is the number of base-learners of order $\O(\log T)$ and $c$ is some small constant given later. We denote by $\Bcal_i$ the $i$-th base-learner for $i\in[N]$. At round $t \in [T]$, the online learner obtains the decision $\x_t$ by aggregating local base decisions via the meta-learner, namely, $\x_t = \sum_{i=1}^N p_{t,i} \x_{t,i}$, where $\x_{t,i}$ is the decision returned by the base-learner $\Bcal_i$ for $i \in [N]$ and $\p_t \in \Delta_N$ is the weight vector returned by the meta-algorithm. The nature then chooses a distribution $\mathfrak{D}_t$ and the individual function $f_t(\cdot)$ is sampled from $\mathfrak{D}_t$. Subsequently, the online learner suffers the loss $f_t(\x_t)$ and observes the gradient $\nabla f_t(\x_t)$.

For the base-learner $\Bcal_i$, in each round $t$, she obtains her local decision $\x_{t+1,i}$ by instantiating the optimistic OMD algorithm (see Algorithm~\ref{alg:1}) with $\psi(\x) = \frac{1}{2\eta_i}\|\x\|_2^2$ and $M_{t+1} = \nabla f_{t}(\x_{t})$ over the linearized surrogate loss $g_t(\x) = \langle \nabla f_t(\x_t),\x \rangle$, where $\eta_i \in \H$ is the step size associated with the $i$-th base-learner. Since $\nabla g_t(\x_{t,i}) = \nabla f_t(\x_t)$, the updating rules of $\B_i$ are demonstrated as
\begin{equation}
\label{eq:base-update}
    \xh_{t+1,i} = \Pi_{\X} \big[ \xh_{t,i} - \eta_i  \nabla f_t(\x_t)\big],~~\x_{t+1,i} = \Pi_{\X} \big[ \xh_{t+1,i} - \eta_i  \nabla f_t(\x_t)\big].  
\end{equation}

The meta-learner updates the weight vector $\p_{t+1} \in \Delta_N$ by  Optimistic Hedge \citep{syrgkanis2015fast} with a time-varying learning rate $\epsilon_t$, that is,
\begin{align}
\label{eq:optimistic-Hedge}
	p_{t+1,i} \propto \exp\bigg(-\epsilon_t \Big(\sum_{s=1}^{t} \ell_{s,i} + m_{t+1,i} \Big)\bigg),
\end{align}
where the feedback loss $\ellb_t \in \R^N$ is constructed by
\begin{equation}
  \label{eq:feedback-loss}
  \ell_{t,i} = \inner{\nabla f_t(\x_t)}{\x_{t,i}} + \lambda \norm{\x_{t,i} - \x_{t-1,i}}_2^2
\end{equation}
for $t \geq 2$ and $\ell_{1,i} = \inner{\nabla f_1(\x_1)}{\x_{1,i}}$; and the optimism $\mb_{t+1} \in \R^N$ is constructed as 
\begin{equation}
  \label{eq:optimism}
  m_{t+1,i} = \inner{M_{t+1}}{\x_{t+1,i}} + \lambda \norm{\x_{t+1,i} - \x_{t,i}}_2^2
\end{equation}
with $M_{t+1}= \nabla f_t(\x_t)$ for $t \geq 2$ and $\m_{1} = \mathbf{0}$; $\lambda \geq 0$ being the coefficient of the correction terms; and we set $\x_{0,i} = \mathbf{0}$ for $i \in [N]$. Note that the correction term $\lambda \norm{\x_{t,i} - \x_{t-1,i}}_2^2$ in the meta-algorithm (both feedback loss and optimism) plays an important role. Indeed, our algorithm design and regret analysis follow the collaborative online ensemble framework proposed by~\citet{JMLR:sword++} for optimizing the gradient-variation dynamic regret. Technically, for such a two-layer structure, to cancel the additional positive term $\sum_{t=2}^T \norm{\x_t - \x_{t-1}}_2^2$ appearing in the derivation of $\sigma_{1:T}^2$ and $\Sigma_{1:T}^2$, one needs to ensure an effective collaboration between the meta and base layers. This involves simultaneously exploiting negative terms of the regret upper bounds in both the base and meta layers as well as leveraging additional negative terms introduced by the above correction term.

\begin{algorithm}[t]
\caption{Dynamic Regret Minimization of the SEA Model}
\label{alg:Sword++forSEA}
    \begin{algorithmic}[1]
  \REQUIRE{step size pool $\H = \{\eta_1,\ldots,\eta_N\}$, learning rate of meta-algorithm $\epsilon_t > 0$, correction coefficient $\lambda >0$}
  \STATE{Initialization: $\x_1 = \widehat{\x}_1 \in \X$,  $\p_1 = \frac{1}{N} \cdot \mathbf{1}_N$}
    \FOR{$t=1$ {\bfseries to} $T$}
      \STATE {Submit the decision $\x_t = \sum_{i=1}^{N} p_{t,i}\x_{t,i}$}
      \STATE {Observe the online function $f_t:\X \mapsto \R$ sampled from the underlying distribution $\mathfrak{D}_t$ and suffer the loss $f_t(\x_t)$}
      \STATE {Base-learner $\B_i$ updates the local decision by optimistic OMD, see~\eqref{eq:base-update}, $\forall i \in [N]$}
      \STATE {Receive $\x_{t+1,i}$ from base-learner $\B_i$ for $i \in [N]$}
      \STATE {Construct the feedback loss $\ellb_t \in \R^N$ and optimism $\mb_{t+1} \in \R^N$ by~\eqref{eq:feedback-loss} and~\eqref{eq:optimism} }
      \STATE {Update the weight $\p_{t+1} \in \Delta_N$ by optimistic Hedge in~\eqref{eq:optimistic-Hedge}}
    \ENDFOR
\end{algorithmic}
\end{algorithm}

\begin{myRemark}
\label{remark:compare-Sachs}
After the submission of our conference paper, \citet{OCO:Between} released an updated version~\citep{arxiv'2023:SEA}, where they also utilized optimistic OMD to achieve the same dynamic regret as our approach. However, there is a significant difference between their method and ours. They employed an optimism design with $m_{t,i} = \langle \nabla f_{t-1}(\bar{\x}_t), \x_{t,i} \rangle$, based on another solution for gradient-variation dynamic regret of online convex optimization~\citep{Problem:Dynamic:Regret}, where $\bar{\x}_t = \sum_{i=1}^N p_{t-1,i}\x_{t,i}$. This design actually introduces a dependence issue in the SEA model because $\bar{\x}_t$ depends on $f_{t-1}(\cdot)$. We provide more elaborations in Appendix~\ref{appendix:proof-xbar-regret} and technical discussions in Remark~\ref{remark:technical-bar-x}.
\remarkend
\end{myRemark}

Below, we provide the dynamic regret upper bound of Algorithm~\ref{alg:Sword++forSEA} for the SEA model, and we will give the proof in Section~\ref{appendix:proof-dynamic-regret}.
\begin{myThm} \label{thm:sword}
Under Assumptions~\ref{ass:3}, \ref{ass:4}, \ref{ass:2} and \ref{ass:1}, setting the step size pool $\H = \{\eta_1,\ldots,\eta_N\}$ with $\eta_i = $ $ \min \{1/(8L), \sqrt{(D^2/(8G^2T))\cdot 2^{i-1}} \}$ and $N = \lceil 2^{-1}\log_2(G^2T/(8L^2D^2)) \rceil+1$, and setting the learning rate of meta-algorithm as $\epsilon_t = \min\{ 1/(8D^2L)$, $\sqrt{(\ln N)/(D^2 \bar{V}_{t})}\}$ for all $t \in [T]$, Algorithm~\ref{alg:Sword++forSEA} ensures 
\begin{align*}
	\E[\Dreg(\u_1,\cdots,\u_T)] \leq \O\Big(P_T + \sqrt{1+P_T}\big(\sqrt{\sigma_{1:T}^2}+\sqrt{\Sigma_{1:T}^2}\big)\Big)
\end{align*}
for any comparator sequence $\u_1,\ldots,\u_T \in \X$, where $\bar{V}_{t} = \sum_{s=2}^{t} \|\nabla f_s(\x_s) - \nabla f_{s-1} (\x_{s-1}) \|_2^2$ with $\nabla f_0(\x_0)$ defined as $\mathbf{0}$, and $P_T = \E[\sum_{t=2}^T \norm{\u_t -\u_{t-1}}_2]$ is the path length of comparators.
\end{myThm}

\begin{myRemark}\label{remark:dynamic}
As mentioned, the static regret studied in earlier sections is a special case of dynamic regret with a fixed comparator. As a consequence, Theorem~\ref{thm:sword} directly implies an $\O(\sqrt{\sigma_{1:T}^2} + \sqrt{\Sigma_{1:T}^2})$ static regret bound by noticing that $P_T = 0$ when comparing to a fixed benchmark, which recovers the result in Theorem~\ref{thm:1}. Moreover, Theorem~\ref{thm:sword} also recovers the $\O(\sqrt{(1+P_T+V_T)(1+P_T)})$ gradient-variation bound of \citet{Problem:Dynamic:Regret,JMLR:sword++} for the adversarial setting and the minimax optimal $\O(\sqrt{T(1+P_T)})$ bound of \citet{Adaptive:Dynamic:Regret:NIPS} since $\sigma_{1:T}^2 = 0$ and $\Sigma_{1:T}^2 = V_T \leq 4G^2T$ in this case.
\remarkend
\end{myRemark}

We focus on the convex and smooth case, while for the strongly convex and exp-concave cases, current understandings of their dynamic regret are still far from complete~\citep{AISTATS'22:sc-proper}. In particular, how to realize optimistic online learning in strongly convex/exp-concave dynamic regret minimization remains open. Lastly, we note that to the best of our knowledge, FTRL has not yet achieved the worst-case $\O(\sqrt{T(1+P_T)})$ dynamic regret~\citet{Adaptive:Dynamic:Regret:NIPS}, let alone the gradient-variation bound. In fact, FTRL is more like a lazy update~\citep{Intro:Online:Convex}, which seems unable to track a sequence of changing comparators. We found that \citet{jacobsen2022parameter} have given preliminary results (in Theorem 2 and Theorem 3 of their work): \emph{all the parameter-free FTRL-based algorithms we are aware of cannot achieve a dynamic regret bound better than $\O(P_T\sqrt{T})$}. Although this cannot cover all the cases of FTRL-based algorithms on dynamic regret, it has at least shown that FTRL-based algorithms do have certain limitations in dynamic regret minimization.

%% file: sections/non_smooth.tex

\subsection{SEA with Non-smooth Functions}\label{sec:non-smooth:static}
The analysis in the previous section depends on the smoothness assumptions of expected functions (see Assumption~\ref{ass:2}). In this part, we further generalize the scope of the SEA model to the \emph{non-smooth} functions. This is facilitated by the optimistic OMD framework again, but we replace gradient-descent updates with \emph{implicit updates} in the optimistic step. 

We consider the static regret minimization for the SEA model with convex and non-smooth functions. Assuming that all individual functions $f_t(\cdot)$'s are convex on $\X$, we update the decision $\x_t$ by deploying optimistic OMD with $\psi_t(\x)=\frac{1}{2\eta_t}\|\x\|_2^2$, i.e.,
\begin{align}
    &\hat{\x}_{t+1} = \Pi_{\X}\left[\hat{\x}_t - \eta_t\nabla f_t(\x_t)\right],\label{non-smooth-convex-update1}\\
    &\x_{t+1} = \argmin_{\x \in \X}f_t(\x) + \frac{1}{2\eta_{t+1}}\|\x - \hat{\x}_{t+1}\|_2^2,\label{non-smooth-convex-update2}
\end{align}
where the update \eqref{non-smooth-convex-update2} is an implicit update and the step size is set as
\begin{align}
    \eta_t = \frac{D}{\sqrt{1 + 4G^2+\sum_{s=1}^{t-1}\|\nabla f_s(\x_s)-\nabla f_{s-1}(\x_{s})\|_2^2}} \label{non-smooth;convex;stepsize}
\end{align} 
for $t \in [T]$ (we define $\eta_{T+1} = \eta_T$). Note that the second step~\eqref{non-smooth-convex-update2} is crucial to remove the dependence on the smoothness of loss functions. Unlike the gradient-based update $\x_{t+1} = \Pi_{\X} \big[ \xh_{t+1} - \eta_{t+1}  \nabla f_t(\x_t)\big]$ used in previous sections, it directly updates over the original function $f_t(\x)$ without linearization, so this is often referred to as ``implicit update"~\citep{campolongo2020temporal,chen2023generalized,NeurIPS'22:label_shift}.

Our algorithm can achieve a similar regret form as the smooth case scaling with the quantities $\Sigma_{1:T}^2$ to reflect the adversarial difficulty and $\tilde{\sigma}_{1:T}^2$ to indicate the stochastic aspect, where the variance quantity $\tilde{\sigma}_{1:T}^2$ is defined as
\begin{align}\label{nonsmooth:tildesigma}
    \tilde{\sigma}_{1:T}^2 = \E\left[\sum_{t=1}^T \tilde{\sigma}_t^2\right], \mbox{ with } \tilde{\sigma}_t^2 = \E_{f_t\sim \mathfrak{D}_t}\left[\sup_{\x\in\X}\|\nabla f_t(\x)-\nabla F_t(\x)\|_2^2\right].
\end{align}
\begin{myRemark}
Note that $\tilde{\sigma}_{1:T}^2$ also captures the stochastic difficulty of the SEA model due to the sample randomness. However, admittedly it is larger than $\sigma_{1:T}^2$ because of the convex nature of the supremum operator. Despite this, we are unable to obtain any $\sigma_{1:T}^2$-type bound for the non-smooth case, and the technical discussions are deferred to Remark~\ref{remark:technical-bar-x}. It is crucial to highlight that later implications will demonstrate significant relevance of this quantity, particularly in real-world problems like online label shift (Section~\ref{subsec:OLaS}).
\remarkend
\end{myRemark}

Below we present the regret guarantee for SEA with \emph{non-smooth} and convex functions. Refer to Section~\ref{appendix:proofs-non-smooth} for the proof.
\begin{myThm} \label{thm:nonsmooth:convex} Under Assumptions~\ref{ass:3}, \ref{ass:4} and \ref{ass:individual:convexity}, optimistic OMD with updates~\eqref{non-smooth-convex-update1}--\eqref{nonsmooth:tildesigma} enjoys the following guarantee:
\[
\begin{split}
\E[\mathbf{Reg}_T(\u)]
    \leq 5D\sqrt{1 + G^2} + 10\sqrt{2}D\sqrt{\tilde{\sigma}_{1:T}^2} + 10D\sqrt{\Sigma_{1:T}^2}
    =\O\left(\sqrt{\tilde{\sigma}_{1:T}^2} +\sqrt{\Sigma_{1:T}^2} \right).
\end{split}
\]
\end{myThm}
This bound is similar in form to the bound for the smooth case (Theorem~\ref{thm:1}), albeit with a slight loss in terms of the variance definition. However, in specific cases, this bound can be as good as the smooth case. For example, for fully adversarial OCO, we have $\tilde{\sigma}_{1:T}^2 = \sigma_{1:T}^2= 0 $ since $f_t(\cdot)= F_t(\cdot)$ for each $t\in[T]$. Moreover, when applying the result to the online label shift problem (see Section~\ref{subsec:OLaS}), using no matter  $\sigma_{1:T}^2$ or $\tilde{\sigma}_{1:T}^2$ will deliver the same regret guarantee that scales with meaningful quantities for online label shift, the detailed analysis of which will be provided in Section~\ref{subsec:OLaS} and Remark~\ref{remark:olas}.

\subsection{SEA with Non-smooth Functions: Dynamic Regret Minimization}\label{subsec:nonsmooth:dynamic}
We further investigate the dynamic regret of SEA with non-smooth and convex functions. To minimize the dynamic regret, we still employ a two-layer online ensemble structure based on the optimistic OMD framework as in Section~\ref{sec:dynamci-regret}, but with \emph{implicit updates} in the base learners and additional ingredients for the design of meta learner.

Specifically, we construct a step size pool $\H = \left\{\eta_i = c \cdot 2^{i} \mid i\in [N] \right\}$ to cover the (approximate) optimal step size, where $N = \O(\log T)$ is the number of candidate step sizes and $c$ is a constant given later. Then we maintain a meta-learner running over a group of base-learners $\{\Bcal_i\}_{i\in [N]}$, each associated with a candidate step size $\eta_i$ from the pool $\H$. The main procedure is summarized in Algorithm~\ref{alg:non-smooth}. 

Consistent with the learner for static regret, each base-learner $\Bcal_i$ here performs the optimistic OMD algorithm parallelly with $\psi(\x)=\frac{1}{2\eta_i}\|\x\|_2^2$ and an implicit update in the optimistic step. That means the updating rules of base-learner $\Bcal_i$ are
\begin{align}
    \label{eq:non-smooth-base}
    \hat{\x}_{t+1,i} = \Pi_{\X}\left[\hat{\x}_{t,i}-\eta_i\nabla f_t(\x_{t,i})\right],~~\x_{t+1,i} = \argmin_{\x \in \X}f_t(\x)+\frac{1}{2\eta_i}\left\| \x-\hat{\x}_{t+1,i} \right\|_2^2,
\end{align}
where $\eta_i\in \H$ is the corresponding candidate step size and $\x_{t+1,i}$ is the local decision.

Then the meta-learner collects local decisions and updates the weight $\p_{t+1}\in\Delta_N$ by
\begin{align}
    p_{t+1,i} \propto \exp \left(-\epsilon_t \left(\sum_{s=1}^t f_s(\x_{s,i}) + f_t(\x_{t+1,i}) \right)\right),\label{non-smooth:dynamic:updatept}
\end{align}
where $p_{t+1,i}$ denotes the weight of the $i$-th base-learner and $\epsilon_t$ is the learning rate to be set later. After that, the online learner submits the decision $\x_{t+1} = \sum_{i=1}^N p_{t+1,i}\x_{t+1,i}$ to the nature, and consequently suffers the loss $f_{t+1}(\x_{t+1})$, where $f_{t+1}$ is sampled from the distribution $\mathfrak{D}_{t+1}$ selected by the nature. Compared to Algorithm~\ref{alg:Sword++forSEA} in the smooth case, we no longer use surrogate losses and correction terms because we apply the technique of converting function variation into gradient variation without any negative term cancellations.

\begin{algorithm}[t]
\caption{Dynamic Regret Minimization of SEA Model with Non-smooth Functions}
\label{alg:non-smooth}
    \begin{algorithmic}[1]
  \REQUIRE{step size pool $\H = \{\eta_1,\ldots,\eta_N\}$, learning rate of meta-algorithm $\epsilon_t > 0$}
  \STATE{Initialization: $\x_1 = \widehat{\x}_1 \in \X$,  $\p_1 = \frac{1}{N} \cdot \mathbf{1}_N$}
    \FOR{$t=1$ {\bfseries to} $T$}
      \STATE {Submit the decision $\x_t = \sum_{i=1}^{N} p_{t,i}\x_{t,i}$}
      \STATE {Observe the online function $f_t:\X \mapsto \R$ sampled from the underlying distribution $\mathfrak{D}_t$ and suffer the loss $f_t(\x_t)$}
      \STATE {Base-learner $\B_i$ updates by optimistic OMD with implicit updates~\eqref{eq:non-smooth-base} for $i \in [N]$}
      \STATE {Receive $\x_{t+1,i}$ from base-learner $\B_i$ for $i \in [N]$}
      \STATE {Update the weight $\p_{t+1} \in \Delta_N$ by optimistic Hedge in~\eqref{non-smooth:dynamic:updatept}}
    \ENDFOR
\end{algorithmic}
\end{algorithm}

We have the following theoretical guarantee for Algorithm~\ref{alg:non-smooth} with proof in Section~\ref{appendix:proof-nonsmooth-dynamic}. 
\begin{myThm} \label{thm:nonsmooth:dynamic}
Under Assumptions~\ref{ass:3}, \ref{ass:4} and \ref{ass:individual:convexity}, setting the step size pool $\H = \{\eta_1,\ldots,\eta_N\}$ with $\eta_i = $ $ (D/\sqrt{1+4TG^2})\cdot 2^{i-1}$ and $N = \lceil \frac{1}{2}\log((1+2T)(1+4TG^2))\rceil+1$, and setting the learning rate of meta-algorithm as $\epsilon_t = 1/\sqrt{1+\sum_{s=1}^t( \max_{i \in [N]}\{|\tilde{f}_s(\x_{s,i})-\tilde{f}_{s-1}(\x_{s,i})|\})^2}$, Algorithm~\ref{alg:non-smooth} ensures 
\begin{align*}
	\E[\Dreg(\u_1,\cdots,\u_T)] \leq \O\left(\sqrt{1+P_T}\Big(\sqrt{\tilde{\sigma}_{1:T}^2}+\sqrt{\Sigma_{1:T}^2}\Big)\right),
\end{align*}
which holds for any comparator sequence $\u_1,\ldots,\u_T \in \X$. 
\end{myThm}

\begin{myRemark}
\label{remark:technical-bar-x}
 Theorem~\ref{thm:nonsmooth:dynamic} is not dependent on the smoothness of expected functions but is applicable to the smooth scenario as well. The $\O(\sqrt{1+P_T}(\sqrt{\tilde{\sigma}_{1:T}^2}+\sqrt{\Sigma_{1:T}^2}))$ bound detailed here and the $\O(P_T + \sqrt{1+P_T}(\sqrt{\sigma_{1:T}^2}+\sqrt{\Sigma_{1:T}^2}))$ bound obtained under smoothness in Theorem~\ref{thm:sword} exhibit similar scaling in their corresponding variance quantities --- $\tilde{\sigma}_{1:T}^2$ and $\sigma_{1:T}^2$,  respectively. However, the definition of $\tilde{\sigma}_{1:T}^2$ is slightly less favorable than that of $\sigma_{1:T}^2$. In addition, using implicit updates in the non-smooth case instead of using the first-order method in the smooth case may be more costly. Moreover, we argue that methods employing information about the function value $f_t(\accentset{\circ}{\x})$ (or the gradient $\nabla f_t(\accentset{\circ}{\x})$) where $\accentset{\circ}{\x}$ is generated \emph{afterward} the decision $\x_t$ can hardly achieve regret bounds scaling with $\sigma_{1:T}^2$.
This holds true for the non-smooth part, as optimistic update steps of base-learners demand the full function information, and the meta-learner requires the value of  $f_t(\x_{t+1,i})$. It also applies to the case of~\citet{arxiv'2023:SEA}, who use the optimism design of~\citet{Problem:Dynamic:Regret} to optimize the dynamic regret of SEA with smooth functions. This would require the gradient $\nabla f_{t-1}(\bar{\x}_t)$ with $\bar{\x}_t = \sum_{i=1}^N p_{t-1,i} \x_{t,i}$ as mentioned in Remark~\ref{remark:compare-Sachs}, and can only obtain a weaker bound scaling with $\tilde{\sigma}_{1:T}^2$. We provide the details in~\pref{appendix:proof-xbar-regret}.
\remarkend    
\end{myRemark}

%% file: appendixes/appendix_dynamic_regret.tex

\subsubsection{Proof of Theorem~\ref{thm:sword}} 
\label{appendix:proof-dynamic-regret} 

This part presents the proof of Theorem~\ref{thm:sword}. Since our algorithmic design is based on the collaborative online ensemble framework of~\citet{JMLR:sword++}, we first introduce the general theorem~\citep[Theorem 9]{JMLR:sword++} and provide the proof for our theorem based on it.
\begin{myThm}[Adaptation of Theorem 9 of~\citet{JMLR:sword++}.]
\label{thm:general-sword-framework}
Under Assumption~\ref{ass:3} (boundedness of gradient norms) and Assumption~\ref{ass:4} (domain boundedness), setting the step size pool $\H$ as
\begin{align}
\H = \left\{ \eta_i =  \min \left\{\Bar{\eta}, \sqrt{\frac{D^2}{8G^2 T}\cdot 2^{i-1}} \right\} ~\Big\vert~ i\in [N]\right\},\label{eqn:dyn:H}
\end{align}
where $N = \lceil 2^{-1}\log_2((8G^2T \Bar{\eta}^2)/D^2) \rceil+1$, and setting meta-algorithm's learning rate as
\begin{align*}
  \epsilon_t = \min\left\{ \bar{\epsilon}, \sqrt{\frac{\ln N}{D^2 \sum_{s=1}^t\norm{\nabla f_s(\x_s) - f_{s-1}(\x_{s-1})}_2^2}} \right\},
\end{align*}
Algorithm~\ref{alg:Sword++forSEA} enjoys the following dynamic regret guarantee:
\begin{equation}
	\label{eq:general-optimistic-bound}
	\begin{split}
	{}& \E\bigg[\sum_{t=1}^T \inner{\nabla f_t(\x_t)}{\x_t - \u_t}\bigg]\\
	 \leq{} & 5\sqrt{D^2\ln N\E[\bar{V}_T ]}+ 2 \sqrt{(D^2+2DP_T)\E[\bar{V}_T]}+\E\bigg[ \frac{\ln N}{\bar{\epsilon}} + 8\bar{\epsilon} D^2G^2+ \frac{D^2+2DP_T}{\bar{\eta}} \notag\\
	&   + \left(\lambda-\frac{1}{4\bar{\eta}}\right) \sum_{t=2}^T\norm{\x_{t,i}-\x_{t-1,i}}_2^2 - \frac{1}{4\bar{\epsilon}} \sum_{t=2}^T\norm{\p_t-\p_{t-1}}_1^2 - \lambda \sum_{t=2}^T\sum_{i=1}^N p_{t,i}\norm{\x_{t,i}-\x_{t-1,i}}_2^2\bigg].
	\end{split}
\end{equation}
In the above, $\bar{V}_T = \sum_{t=1}^{T} \norm{\nabla f_t(\x_t)-\nabla f_{t-1}(\x_{t-1})}_2^2$ is the adaptivity term measuring the quality of optimistic gradient vectors $\{M_t = \nabla f_{t-1}(\x_{t-1})\}_{t=1}^T$, and $P_T = \E[\sum_{t=2}^{T} \norm{\u_{t-1} - \u_{t}}_2]$ is the path length of comparators.
\end{myThm}

\begin{myRemark} Note that $\u_1,\cdots,\u_T$ may exhibit randomness in the SEA model, so the path length $P_T$ we define is in the expected form. Consequently, we have introduced a subtle modification to Theorem 5 of~\citet{JMLR:sword++}, in which the expectation is taken before tuning the step size in its analysis.
\remarkend
\end{myRemark}

In the following, we prove Theorem~\ref{thm:sword} based on Theorem~\ref{thm:general-sword-framework}.
~\\

\begin{proof}[of {Theorem~\ref{thm:sword}}]
In Theorem~\ref{thm:general-sword-framework}, where $\Bar{V}_T = \sum_{t=1}^{T} \norm{\nabla f_t(\x_t)-\nabla f_{t-1}(\x_{t-1})}_2^2$, applying Lemma~\ref{lem:sum:diff} (boundedness of cumulative norm of gradient difference) allows us to bound the first and second term as
\begin{align}
    & 5\sqrt{D^2\ln N\E[\bar{V}_T] }+ 2 \sqrt{(D^2+2DP_T)\E[\bar{V}_T]} \nonumber\\
    \leq{}&G\left(5\sqrt{D^2\ln N}+ 2 \sqrt{(D^2+2DP_T)}\right)+ \left(5\sqrt{D^2\ln N}+ 2 \sqrt{(D^2+2DP_T)}\right)\sqrt{4L^2\E\left[\sum_{t=2}^T\|\x_t - \x_{t-1}\|_2^2\right]} \nonumber\\
    & + \left(5\sqrt{D^2\ln N}+ 2 \sqrt{(D^2+2DP_T)}\right)\left( 2\sqrt{2}\sqrt{\sigma_{1:T}^2}+ 2\sqrt{\Sigma_{1:T}^2}\right).\label{eqn:dyn:Vbar}
\end{align}
To eliminate the relevant terms of $\|\x_t - \x_{t-1}\|_2^2$, we first notice that
\begin{align*}
    \|\x_t - \x_{t-1}\|_2^2 ={}&\left \|\sum_{i=1}^N p_{t,i}\x_{t,i} - \sum_{i=1}^N p_{t-1,i}\x_{t-1,i}\right\|_2^2 \\
    \leq{}& 2\left \|\sum_{i=1}^N p_{t,i}\x_{t,i} - \sum_{i=1}^N p_{t,i}\x_{t-1,i}\right\|_2^2 + 2\left \|\sum_{i=1}^N p_{t,i}\x_{t-1,i} - \sum_{i=1}^N p_{t-1,i}\x_{t-1,i}\right\|_2^2\\
    \leq{}& 2\left(\sum_{i=1}^N p_{t,i}\|\x_{t,i} - \x_{t-1,i}\|_2 \right)^2 + 2\left(\sum_{i=1}^N |p_{t,i} - p_{t-1,i}|\|\x_{t-1,i}\|_2 \right)^2\\
    \leq{}& 2\sum_{i=1}^N p_{t,i}\|\x_{t,i} - \x_{t-1,i}\|_2^2 + 2D^2\|\p_t - \p_{t-1}\|_1^2.
\end{align*}
Thus we get $\sum_{t=2}^T\|\x_t - \x_{t-1}\|_2^2 \leq  2\sum_{t=2}^T\sum_{i=1}^N p_{t,i}\|\x_{t,i} - \x_{t-1,i}\|_2^2 + 2D^2\sum_{t=2}^T \|\p_t - \p_{t-1}\|_1^2$. Then we can use it and the AM-GM inequality to bound the second term in \eqref{eqn:dyn:Vbar}:
\begin{align*}
    & \left(5\sqrt{D^2\ln N}+ 2 \sqrt{(D^2+2DP_T)}\right)\sqrt{4L^2\E\left[\sum_{t=2}^T\|\x_t - \x_{t-1}\|_2^2\right]}\\
    \leq{}& 5\sqrt{D^2\ln N \left(8L^2\E\left[\sum_{t=2}^T\sum_{i=1}^N p_{t,i}\|\x_{t,i} - \x_{t-1,i}\|_2^2\right] + 8L^2D^2\E\left[\sum_{t=2}^T \|\p_t - \p_{t-1}\|_1^2\right]\right)}\\
    &+ 2 \sqrt{(D^2+2DP_T)\left(8L^2\E\left[\sum_{t=2}^T\sum_{i=1}^N p_{t,i}\|\x_{t,i} - \x_{t-1,i}\|_2^2\right] + 8L^2D^2\E\left[\sum_{t=2}^T \|\p_t - \p_{t-1}\|_1^2\right]\right)}\\
    \leq{}& \frac{25\ln N}{4\Bar{\epsilon}} + \frac{D^2+2DP_T}{\Bar{\eta}} + \left(8\Bar{\epsilon}D^2L^2 +  8\Bar{\eta}L^2 \right)\E\left[\sum_{t=2}^T\sum_{i=1}^N p_{t,i}\|\x_{t,i} - \x_{t-1,i}\|_2^2\right] \\
    &+ \left(8\Bar{\epsilon}L^2D^4+ 8\Bar{\eta}L^2D^2\right)\E\left[\sum_{t=2}^T \|\p_t - \p_{t-1}\|_1^2\right].
\end{align*}
Combining \eqref{eqn:dyn:Vbar} and the above formula with the regret in Theorem~\ref{thm:general-sword-framework}, we have
\begin{align*}
    {}& \E\bigg[\sum_{t=1}^T \inner{\nabla f_t(\x_t)}{\x_t - \u_t}\bigg]\nonumber\\
    \leq{}& G\left(5\sqrt{D^2\ln N}+ 2 \sqrt{(D^2+2DP_T)}\right)\\ {}&+\left(5\sqrt{D^2\ln N}+ 2 \sqrt{(D^2+2DP_T)}\right)\left( 2\sqrt{2}\sqrt{\sigma_{1:T}^2}+ 2\sqrt{\Sigma_{1:T}^2}\right) + \frac{2D^2+4DP_T}{\bar{\eta}}\nonumber\\
    {}&  + \left(\lambda-\frac{1}{4\bar{\eta}}\right) \E\left[\sum_{t=2}^T\norm{\x_{t,i}-\x_{t-1,i}}_2^2\right]+\left(8\Bar{\epsilon}L^2D^4+ 8\Bar{\eta}L^2D^2 - \frac{1}{4\bar{\epsilon}}\right) \E\left[\sum_{t=2}^T\norm{\p_t-\p_{t-1}}_1^2\right]\nonumber\\
    {}&  + \left(8\Bar{\epsilon}D^2L^2 +  8\Bar{\eta}L^2 - \lambda 
 \right)\E\left[\sum_{t=2}^T\sum_{i=1}^N p_{t,i}\norm{\x_{t,i}-\x_{t-1,i}}_2^2\right] + \frac{29\ln N}{4\Bar{\epsilon}} + 8\bar{\epsilon} D^2G^2.
\end{align*}
Setting $\lambda = 2L$, $\Bar{\eta} = \frac{1}{8L}$ and $\Bar{\epsilon} = \frac{1}{8D^2L}$, we can drop the last three non-positive terms to get 
\begin{align}
    &\E\left[ \sum_{t=1}^T \inner{\nabla f_t(\x_t)}{\x_t - \u_t} \right]\nonumber\\
    \leq{}& G\left(5\sqrt{D^2\ln N}+ 2 \sqrt{(D^2+2DP_T)}\right) + \left(5\sqrt{D^2\ln N}+ 2 \sqrt{(D^2+2DP_T)}\right)\left( 2\sqrt{2}\sqrt{\sigma_{1:T}^2}+ 2\sqrt{\Sigma_{1:T}^2}\right) \nonumber\\
    & + (58\ln N + 16)D^2L + 32DLP_T + \frac{1}{L}G^2 = \O\left(P_T + \sqrt{(1+P_T)}\left(\sqrt{\sigma_{1:T}^2} + \sqrt{\Sigma_{1:T}^2}\right)\right),\label{eqn:proof:dynamic:bound}
\end{align}
which completes the proof.
\end{proof}

%% file: appendixes/appendix_non_smooth.tex
\subsubsection{{Proof of Theorem~\ref{thm:nonsmooth:convex}}}
\label{appendix:proofs-non-smooth}
Before giving proofs of the non-smooth case, for the sake of simplicity of the presentation, we first introduce the following notation:
\begin{align}
    \tilde{V}_T = \sum_{t=1}^T\sup_{\x\in\X}\left\|\nabla f_t(\x)- \nabla f_{t-1}(\x)\right\|_2^2,
\end{align}
which adds a supremum operation before summing compared with $V_T$.
\begin{proof}
Referring to \eqref{eqn:exp:cov} from the previous article, for convex random functions, we have:
\begin{align} 
\E\big[ f_t(\x_t) - f_t(\u) \big]
\leq \E\big[  \langle \nabla f_t(\x_t), \x_t -\u \rangle \big]. \label{non-smooth;convex;linearloss}
\end{align}
We decompose the instantaneous loss above as
\begin{align}
    {}&\langle \nabla f_t(\x_t),\x_t - \u \rangle\nonumber\\
    \leq{}& \underbrace{\langle \nabla f_t(\x_t) - \nabla f_{t-1}(\x_{t}),\x_t - \hat{\x}_{t+1} \rangle}_{\term{a}} + \underbrace{\langle \nabla f_{t-1}(\x_{t}), \x_t - \hat{\x}_{t+1}\rangle}_{\term{b}} + \underbrace{\langle \nabla f_t(\x_t), \hat{\x}_{t+1} - \u \rangle}_{\term{c}}.\label{appendix:nonsmooth:convex:proof:decompose}
\end{align}
So we give the upper bounds of these three terms respectively in the following. For term (a), by Fenchel's inequality for the squared $L_2$ norm, we have
\begin{align}
    \term{a} \leq{}& 2\eta_t\|\nabla f_t(\x_t) - \nabla f_{t-1}(\x_{t})\|_2^2 + \frac{1}{2\eta_t}\|\x_t - \hat{\x}_{t+1}\|_2^2. \label{appendix:nonsmooth:convex:proof:a}
\end{align}

We introduce the following lemma to bound term (b), which is related to the implicit update procedure with the proof presented in~\pref{appendix:nonsmooth-usefullemma}. Note that to make the following lemma hold, we need the convexity of individual functions. 
\begin{myLemma}\label{lem:proposition4.1}
    Let $\hat{\x}_{t+1}$ and $\x_{t+1}$ be defined as in~\eqref{non-smooth-convex-update1} and~\eqref{non-smooth-convex-update2}. Then, for any $\x \in \X$, 
    \begin{align*}
\langle \nabla f_t(\x_{t+1}),\x_{t+1}-\x \rangle \leq \frac{1}{2\eta_{t+1}}\left( \|\x-\hat{\x}_{t+1}\|_2^2 - \|\x-\x_{t+1}\|_2^2 - \| \x_{t+1}-\hat{\x}_{t+1}\|_2^2 \right).
    \end{align*}
\end{myLemma}
According to Lemma~\ref{lem:proposition4.1}, we set $\x = \hat{\x}_{t+1}$ and obtain 
\begin{align}
    \term{b} \leq \frac{1}{2\eta_{t}}\left( \|\hat{\x}_{t+1} - \hat{\x}_t\|_2^2 - \|\hat{\x}_{t+1}-\x_t\|_2^2 - \|\hat{\x}_t - \x_t\|_2^2 \right).\label{appendix:nonsmooth:convex:proof:b}
\end{align}
For term (c), we leverage Lemma 7 of \citet{Problem:Dynamic:Regret} to get 
\begin{align}
    \term{c} \leq \frac{1}{2\eta_t}\left( \|\u - \hat{\x}_t\|_2^2 - \|\u - \hat{\x}_{t+1}\|_2^2 - \|\hat{\x}_t - \hat{\x}_{t+1}\|_2^2 \right).\label{appendix:nonsmooth:convex:proof:c}
\end{align}
Combining the three upper bounds above and summing over $t=1,\cdots,T$, we have
\begin{align}
    {}&\sum_{t=1}^T \langle \nabla f_t(\x_t),\x_t - \u \rangle\nonumber\\
    \leq{}& \sum_{t=1}^T 2\eta_t\|\nabla f_t(\x_t) - \nabla f_{t-1}(\x_{t})\|_2^2 + \sum_{t=1}^T \frac{1}{2\eta_t}\left( \|\u - \hat{\x}_t\|_2^2 - \|\u - \hat{\x}_{t+1}\|_2^2\right) + \frac{D^2}{2\eta_T}\nonumber\\
    \leq{}& \sum_{t=1}^T 2\eta_t\|\nabla f_t(\x_t) - \nabla f_{t-1}(\x_{t})\|_2^2 + \frac{D^2}{2\eta_1}+\frac{D^2}{2}\sum_{t=2}^T\left(\frac{1}{\eta_t}-\frac{1}{\eta_{t-1}}\right)+ \frac{D^2}{2\eta_T}\nonumber\\
    \leq{}&\sum_{t=1}^T 2\eta_t\|\nabla f_t(\x_t) - \nabla f_{t-1}(\x_{t})\|_2^2 + \frac{D^2}{\eta_T},\label{non-smooth:convex:sum1}
\end{align}
where we drop the negative term $- \frac{1}{2\eta_{t}}\|\hat{\x}_t - \x_t\|_2^2$ to get the first inequality. Then we apply the inequality $\eta_t \leq D/\sqrt{1 +\sum_{s=1}^{t}\|\nabla f_s(\x_s)-\nabla f_{s-1}(\x_{s})\|_2^2}$ and Lemma~\ref{lem:sum} to obtain
\begin{align*}
    \sum_{t=1}^T \langle \nabla f_t(\x_t),\x_t - \u \rangle
    \leq{}& 2\sum_{t=1}^T \frac{D}{\sqrt{1 +\sum_{s=1}^{t}\|\nabla f_s(\x_s)-\nabla f_{s-1}(\x_{s})\|_2^2}} \|\nabla f_t(\x_t) - \nabla f_{t-1}(\x_{t})\|_2^2 \\
    {}&+ D\sqrt{1 +4G^2 + \sum_{t=1}^{T}\|\nabla f_t(\x_t)-\nabla f_{t-1}(\x_{t})\|_2^2}\\
    \leq{}& 5D\sqrt{1+4G^2 + \sum_{t=1}^{T}\sup_{\x\in\X}\|\nabla f_t(\x)-\nabla f_{t-1}(\x)\|_2^2}
\end{align*}
Moreover, we develop a lemma to bound the $\sum_{t=1}^{T}\sup_{\x\in\X}\|\nabla f_t(\x)-\nabla f_{t-1}(\x)\|_2^2$ term with its proof in~\pref{appendix:nonsmooth-usefullemma}.
\begin{myLemma}\label{lem:nonsmooth-sumlemma}
Under Assumption~\ref{ass:3}, we have
\begin{align*}
    {}& \sum_{t=1}^T\sup_{\x\in\X}\|\nabla f_t(\x)-\nabla f_{t-1}(\x)\|_2^2 \\
    \leq{}& G^2+ 6\sum_{t=1}^T\sup_{\x\in\X}\|\nabla f_t(\x)-\nabla F_t(\x)\|_2^2 + 4\sum_{t=2}^T\sup_{\x\in\X}\|\nabla F_t(\x)- \nabla F_{t-1}(\x)\|_2^2.
\end{align*}
\end{myLemma}
According to this lemma, we get that
\begin{align*}
    \sum_{t=1}^T \langle \nabla f_t(\x_t),\x_t - \u \rangle
    \leq{}& 5D\sqrt{1 + 5G^2} + 10\sqrt{2}D\sqrt{\sum_{t=1}^T\sup_{\x\in\X}\|\nabla f_t(\x)-\nabla F_t(\x)\|_2^2} \\
    {}&+ 10D\sqrt{\sum_{t=2}^T\sup_{\x\in\X}\|\nabla F_t(\x)- \nabla F_{t-1}(\x)\|_2^2}.
\end{align*}
Taking expectations with Jensen's inequality and combining with \eqref{non-smooth;convex;linearloss}, we arrive at
\begin{align*}
    \E\left[\sum_{t=1}^T f_t(\x_t)-\sum_{t=1}^T f_t(\u)\right]
    \leq{}& 5D\sqrt{1 + 5G^2} + 10\sqrt{2}D\sqrt{\tilde{\sigma}_{1:T}^2} + 10D\sqrt{\Sigma_{1:T}^2}\\
    ={}&\O\left(\sqrt{\tilde{\sigma}_{1:T}^2} +\sqrt{\Sigma_{1:T}^2} \right),
\end{align*}
which ends the proof.
\end{proof}

\subsubsection{{Proof of Theorem~\ref{thm:nonsmooth:dynamic}}}
\label{appendix:proof-nonsmooth-dynamic}
\begin{proof}
For dynamic regret minimization based on Algorithm~\ref{alg:non-smooth}, we can decompose the expected dynamic regret into the \emph{meta-regret} and \emph{base-regret}:
\begin{align}\label{appendix:eqn:meta-base-regret}
    \E\left[\sum_{t=1}^T f_t(\x_t)-\sum_{t=1}^T f_t(\u_t)\right]=\underbrace{\E\left[\sum_{t=1}^T f_t(\x_t)-\sum_{t=1}^T f_t(\x_{t,i})\right]}_{\metaregret}+\underbrace{\E\left[\sum_{t=1}^T f_t(\x_{t,i})-\sum_{t=1}^T f_t(\u_t)\right]}_{\baseregret}.
\end{align}
The first part quantifies the cumulative loss difference between overall and base decisions, while the second part measures the dynamic regret of base-learner $\Bcal_i$. This decomposition applies to any base-learner's index $i \in [N]$. We then present upper bounds for both terms.
\paragraph{Bounding the meta-regret.}~For the meta-regret, due to Jensen's inequality, we have
\begin{align*}
    \E\left[\sum_{t=1}^T f_t(\x_t)-\sum_{t=1}^T f_t(\x_{t,i})\right] \leq \E\left[\sum_{t=1}^T \sum_{j=1}^N p_{t,j}f_t(\x_{t,j})-\sum_{t=1}^T f_t(\x_{t,i})\right].
\end{align*}
By introducing the reference losses $\tilde{f}_t(\x_{t, i})=f_t(\x_{t,i})-f_t(\x_{\text{ref}})$ and $\tilde{f}_{t-1}(\x_{t,i})= f_{t-1}(\x_{t,i}) - f_{t-1}(\x_{\text{ref}})$, where $\x_{\text{ref}}$ is an arbitrary reference point in $\X$, we can easily verify that
\begin{align*}
    p_{t,i} = \frac{\exp \left( \epsilon_t \left(\sum_{s=1}^{t-1} f_s(\x_{s,i}) + f_{t-1}(\x_{t,i}) \right)\right)}{\sum_{j=1}^N \exp \left( \epsilon_t \left(\sum_{s=1}^{t-1} f_s(\x_{s,j}) + f_{t-1}(\x_{t,j}) \right)\right)} = \frac{\exp \left( \epsilon_t \left(\sum_{s=1}^{t-1} \tilde{f}_s(\x_{s,i}) + \tilde{f}_{t-1}(\x_{t,i}) \right)\right)}{\sum_{j=1}^N \exp \left( \epsilon_t \left(\sum_{s=1}^{t-1} \tilde{f}_s(\x_{s,j}) + \tilde{f}_{t-1}(\x_{t,j}) \right)\right)}.
\end{align*}
That means the updating rule of $\p_{t+1}$ for meta-learner in~\eqref{non-smooth:dynamic:updatept} can also be written as
\begin{align}
    p_{t+1,i} \propto \exp \left(-\epsilon_t \left(\sum_{s=1}^t \tilde{f}_s(\x_{s,i}) + \tilde{f}_t(\x_{t+1,i}) \right)\right).\label{appendix:nonsmooth:updating:pt}
\end{align}
According to~\citet{JMLR:sword++}, the updating rule \eqref{appendix:nonsmooth:updating:pt} which uses adaptive learning rate $\epsilon_t$ is identical to the optimistic FTRL algorithm which updates by
\begin{align*}
    \p_{t+1} = \argmin_{\p \in \Delta_N}\left\langle \p,\sum_{s=1}^t \ellb_s+\mb_{t+1}) \right\rangle + \psi_{t+1}(\p)
\end{align*}
with the regularizer $\psi_{t+1}(\p) = \frac{1}{\epsilon_t}(\sum_{i=1}^N p_i\ln p_i + \ln N)$, where the $i$-th component of $\ellb_s$ is $\ell_{s,i} = \tilde{f}_s(\x_{s,i}) (i \in [N])$ and the $i$-th component of $\mb_{t+1}$ is $m_{t+1,i} = \tilde{f}_t(\x_{t+1,i}) (i \in [N])$ (this is easily proved by computing the closed-form solution). As a result, we can apply Lemma~\ref{lem:FTRL} (standard analysis of optimistic FTRL) and the AM-GM inequality to obtain 
\begin{align*}
    {}&\sum_{t=1}^T\left\langle \p_t, \ellb_t \right\rangle - \sum_{t=1}^T\ell_{t,i}\\
    \leq{}& \max_{\p\in\Delta}\psi_{T+1}(\p)+\sum_{t=1}^T\left(\left\langle\ellb_t - \mb_t, \p_t - \p_{t+1} \right \rangle - \frac{1}{2\epsilon_{t-1}}\left\| \p_t - \p_{t+1} \right\|_1^2\right)\\
    \leq{}& \frac{\ln N}{\epsilon_T} + \sum_{t=1}^T \epsilon_{t-1}\left\| \ellb_t - \mb_t \right\|_{\infty}^2 + \frac{1}{4\epsilon_{t-1}}\left\| \p_t - \p_{t+1} \right\|_1^2-\frac{1}{2\epsilon_{t-1}}\left\| \p_t - \p_{t+1} \right\|_1^2\\
    \leq{}&\frac{\ln N}{\epsilon_T} + \sum_{t=1}^T \epsilon_{t-1}\left\| \ellb_t - \mb_t \right\|_{\infty}^2 = \frac{\ln N}{\epsilon_T} + \sum_{t=1}^T \epsilon_{t-1}\left( \max_{i \in [N]}\left\{\left|\tilde{f}_t(\x_{t,i})-\tilde{f}_{t-1}(\x_{t,i})\right|\right\} \right)^2.
\end{align*} 
Since $\epsilon_t = 1/\sqrt{1+\sum_{s=1}^t\left( \max_{i \in [N]}\left\{\left|\tilde{f}_s(\x_{s,i})-\tilde{f}_{s-1}(\x_{s,i})\right|\right\} \right)^2}$, we have
\begin{align*}
    {}&\sum_{t=1}^T\left\langle \p_t, \ellb_t \right\rangle - \sum_{t=1}^T\ell_{t,i}\\
    \leq{}&\frac{\ln N}{\epsilon_T} + \sum_{t=1}^T \frac{\left( \max_{i \in [N]}\left\{\left|\tilde{f}_t(\x_{t,i})-\tilde{f}_{t-1}(\x_{t,i})\right|\right\} \right)^2}{\sqrt{1+\sum_{s=1}^{t-1}\left( \max_{i \in [N]}\left\{\left|\tilde{f}_s(\x_{s,i})-\tilde{f}_{s-1}(\x_{s,i})\right|\right\} \right)^2}}\\
    \leq{}& (\ln N + 4)\sqrt{1+\sum_{t=1}^{T}\left( \max_{i \in [N]}\left\{\left|\tilde{f}_t(\x_{t,i})-\tilde{f}_{t-1}(\x_{t,i})\right|\right\} \right)^2} + \max_{t \in [T]}\left( \max_{i \in [N]}\left\{\left|\tilde{f}_t(\x_{t,i})-\tilde{f}_{t-1}(\x_{t,i})\right|\right\} \right)^2,
\end{align*}
where we exploit Lemma~\ref{lem:selftuning} in the second inequality. Next, to convert the function variation to the gradient variation, we define $H_t(\x_{t,i}) = f_t(\x_{t,i}) - f_{t-1}(\x_{t,i})$ and get
\begin{align*}
    \left|\tilde{f}_t(\x_{t,i})-\tilde{f}_{t-1}(\x_{t,i})\right| = {}&\left|H_t(\x_{t,i})-H_t(\x_{\text{ref}})\right| = \left|\langle \nabla H_t( \boldsymbol{\xi}_{t,i}),\x_{t,i}-\x_{\text{ref}} \rangle\right|\\
    \leq{}&D\left\| \nabla f_t(\boldsymbol{\xi}_{t,i})-\nabla f_{t-1}(\boldsymbol{\xi}_{t,i}) \right\|_2 \leq D \sup_{\x \in \X}\left\| \nabla f_t(\x)-\nabla f_{t-1}(\x) \right\|_2,
\end{align*}
where the second equality is due to the mean value theorem and $\boldsymbol{\xi}_{t,i} = c_{t,i}\x_{t,i} + (1-c_{t,i})\x_{\text{ref}}$ with $c_{t,i} \in [0,1]$. So by Assumption~\ref{ass:3} (boundedness of gradient norms), we have
\begin{align*}
    \sum_{t=1}^T\left\langle \p_t, \ellb_t \right\rangle - \sum_{t=1}^T\ell_{t,i}
    \leq (\ln N + 4)\sqrt{1+D^2 \tilde{V}_T} + 4G^4 \leq (\ln N + 4)D\sqrt{\tilde{V}_T} + 4G^4 + \ln N + 4.
\end{align*}
Further combining the definitions of $\ellb_t$ and $\ell_{t,i}$, we finally get
\begin{align}
\sum_{t=1}^T \sum_{j=1}^N p_{t,j}f_t(\x_{t,j})-\sum_{t=1}^T f_t(\x_{t,i})={}& \sum_{t=1}^T \sum_{j=1}^N p_{t,j}\tilde{f}_t(\x_{t,j})-\sum_{t=1}^T \tilde{f}_t(\x_{t,i})=\sum_{t=1}^T\left\langle \p_t, \ellb_t \right\rangle - \sum_{t=1}^T\ell_{t,i}\nonumber\\
\leq{}&(\ln N + 4)D\sqrt{\tilde{V}_T} + 4G^4 + \ln N + 4.\label{non-smooth:dynamic:metaregret}
\end{align}

\paragraph{Bounding the base-regret.}~Owing to the convexity of individual functions, we have
\begin{align*}
    \E\left[\sum_{t=1}^T f_t(\x_{t,i})-\sum_{t=1}^T f_t(\u_t)\right] \leq \E\left[\sum_{t=1}^T \langle \nabla f_t(\x_{t,i}), \x_{t,i} -\u_t\rangle\right].
\end{align*}
Similar to the non-smooth case of static regret, we can get the upper bound of the above instantaneous loss following the same arguments in obtaining \eqref{non-smooth:convex:sum1}:
\begin{align*}
    {}&\E\left[\sum_{t=1}^T\langle \nabla f_t(\x_{t,i}), \x_{t,i} -\u_t\rangle\right]\\
    \leq{}&\E\left[2\eta_i \sum_{t=1}^T\left\| \nabla f_t(\x_{t,i}) - \nabla f_{t-1}(\x_{t,i}) \right\|_2^2 + \frac{1}{2\eta_i}\sum_{t=2}^T\bigg( \left\| \u_t - \hat{\x}_{t,i} \right\|_2^2 -  \left\| \u_{t-1} - \hat{\x}_{t,i} \right\|_2^2 \bigg)+ \frac{D^2}{2\eta_i}\right]\\
    \leq{}&2\eta_i \E\left[\tilde{V}_T\right] +\E\left[\frac{1}{2\eta_i}\sum_{t=2}^T\|\u_t-\u_{t-1}\|_2\|\u_t-\hat{\x}_{t,i}+\u_{t-1}-\hat{\x}_{t,i}\|_2\right]+ \frac{D^2}{2\eta_i}\\
    \leq{}& 2\eta_i \E\left[\tilde{V}_T\right] + \frac{D^2 + 2DP_T}{2\eta_i},
\end{align*}
where the second inequality comes from that $\left\| \u_t - \hat{\x}_{t,i} \right\|_2^2 -  \left\| \u_{t-1} - \hat{\x}_{t,i} \right\|_2^2=\langle \u_t - \hat{\x}_{t,i} - (\u_{t-1} - \hat{\x}_{t,i}), \u_t - \hat{\x}_{t,i} + (\u_{t-1} - \hat{\x}_{t,i}) \rangle\leq \|\u_t-\u_{t-1}\|_2\|\u_t-\hat{\x}_{t,i}+\u_{t-1}-\hat{\x}_{t,i}\|_2$. Since we have
\begin{align*}
    \tilde{V}_T \leq \sum_{t=1}^T \left(2\sup_{\x\in\X}\left\| \nabla f_t(\x)\right\|_2^2 + 2\sup_{\x\in\X}\left\|\nabla f_{t-1}(\x) \right\|_2^2 \right)\leq 4TG^2,
\end{align*}
the optimal step size $\eta^{*}=\frac{1}{2}\sqrt{\frac{D^2 + 2DP_T}{1+\E[\tilde{V}_T]}}$ should lie in the range $\big[\frac{1}{2}\sqrt{\frac{D^2}{1+4TG^2}},\frac{1}{2}\sqrt{D^2 + 2D^2T}\big]$. Our designed step size pool is $\H = \left\{\frac{D}{\sqrt{1+4TG^2}}\cdot 2^{i-1} \mid i\in [N] \right\}$ with $N=\lceil \frac{1}{2}\log((1+2T)(1+4TG^2))\rceil+1$. There must be an $\eta_{i^*}\in \H$ satisfying $\eta_{i^*}\leq \eta^* \leq 2\eta_{i^*}$ and we can obtain that 
\begin{align}
{}&\E\left[\sum_{t=1}^T\langle \nabla f_t(\x_{t,i}), \x_{t,i} -\u_t\rangle\right]\nonumber\\
    \leq{}&2\eta_{i^*} \E\left[\tilde{V}_T\right] + \frac{D^2 + 2DP_T}{2\eta_{i^*}} \leq 2\eta^* \E\left[\tilde{V}_T\right] +\frac{D^2 + 2DP_T}{\eta^*}\leq 2\sqrt{2 (D^2 + 2DP_T)\E\left[\tilde{V}_T\right]}. \label{non-smooth:dynamic:baseregret}
\end{align}
\paragraph{Bounding the overall dynamic regret.}~Combining the meta-regret~\eqref{non-smooth:dynamic:metaregret} and the base-regret~\eqref{non-smooth:dynamic:baseregret}, we further obtain that
\begin{align}
    {}&\E\left[\sum_{t=1}^T f_t(\x_t)-\sum_{t=1}^T f_t(\u_t)\right]\nonumber\\
    \leq{}&\left(D(\ln N + 4)+2\sqrt{2(D^2 + 2DP_T)}\right)\sqrt{\E\left[\tilde{V}_T\right]}+ 4G^4 + \ln N + 4\nonumber\\
    \leq{}&\left(D(\ln N + 4)+2\sqrt{2(D^2 + 2DP_T)}\right)\left(G + 2\sqrt{2\tilde{\sigma}_{1:T}^2}+2\sqrt{\Sigma_{1:T}^2}\right)+ 4G^4 + \ln N + 4\nonumber\\
    ={}&\O\left(\sqrt{1+P_T}\left(\sqrt{\tilde{\sigma}_{1:T}^2}+\sqrt{\Sigma_{1:T}^2}\right)\right), \label{appendix:nonsmooth:dynamic:result}
\end{align}
where we make use of Lemma~\ref{lem:nonsmooth-sumlemma} in the last inequality and finish the proof.
\end{proof}

%% file: sections/implication.tex
\section{Implications}\label{sec:examples}
In this section, first, we demonstrate how our results can be applied to recover the regret bound for adversarial data and the excess risk bound for stochastic data. Then, we discuss the implications for other intermediate examples. 

We begin by listing two points followed by all the examples. First, for convex and smooth functions, we obtain the same $\O(\sqrt{\sigma_{1:T}^2}+\sqrt{\Sigma_{1:T}^2})$ bound as \citet{OCO:Between}, so we will not repeat the analysis below unless necessary. But we emphasize that our result eliminates the assumption for convexity of individual functions, which is required in their work. Second, for strongly convex and smooth functions, we will omit the $(\sigma_{\max}^2+\Sigma_{\max}^2)$ part in the logarithmic term of our $\O(\frac{1}{\lambda}(\sigma_{\max}^2+\Sigma_{\max}^2)\log ((\sigma_{1:T}^2 + \Sigma_{1:T}^2)/(\sigma_{\max}^2+\Sigma_{\max}^2)))$ bound below for simplicity.

\subsection{Fully Adversarial Data}
For fully adversarial data, we have $\sigma_{1:T}^2 = 0$ as $\sigma_{t}^2 = 0$ for $t \in [T]$, and $\Sigma_{1:T}^2$ is equivalent to $V_T$. In this case, our bound in Theorem~\ref{thm:2} guarantees an $\O(\frac{1}{\lambda} \log V_T)$ regret bound for $\lambda$-strongly convex and smooth functions, recovering the gradient-variation bound of \citet{ICML:2022:Zhang}. By contrast, the result of~\citet{OCO:Between} can only recover the $\O(\frac{1}{\lambda}\log T)$ worst-case bound. Furthermore, for $\alpha$-exp-concave functions, our new result (Theorem~\ref{thm:3}) implies an $\O(\frac{d}{\alpha}\log V_T)$ regret bound for OCO, recovering the result of~\citet{Gradual:COLT:12}.

\subsection{Fully Stochastic Data}
For fully stochastic data, the loss functions are i.i.d., so we have $\Sigma_{1:T}^2 = 0$ and $\sigma_t = \sigma,\,\forall t \in [T]$. Then for $\lambda$-strongly convex functions, Theorem~\ref{thm:2} implies the same $\O(\log T/[\lambda T])$ excess risk bound as~\citet{OCO:Between}. Besides, Theorem~\ref{thm:3} further delivers a new $\O(d \log T/[\alpha T])$ bound for $\alpha$-exp-concave functions. These results match the well-known bounds in SCO~\citep{ML:Hazan:2007} through online-to-batch conversion.

\subsection{Adversarially Corrupted Stochastic Data}\label{ACSD}
In the adversarially corrupted stochastic model, the loss function consists of two parts: $f_t(\cdot) = h_t(\cdot) + c_t(\cdot)$, where $h_t(\cdot)$ is the loss of i.i.d.~data sampled from a fixed distribution $\mathfrak{D}$, and $c_t(\cdot)$ is a smooth adversarial perturbation satisfying that $\sum_{t=1}^T\max_{\x \in \X} \|\nabla c_t(\x)\| \leq C_T $, where $C_T \geq 0$ is a parameter called the \emph{corruption level}. \citet{2021On} studies this model in expert and bandit problems, proposing a bound consisting of regret of i.i.d.~data and an $\sqrt{C_T}$ term measuring the corrupted performance. \citet{OCO:Between} achieve a similar $\O\big(\sigma \sqrt{T} + \sqrt{C_T}\big)$ bound in OCO problems under convexity and smoothness conditions, and raise an open question about how to extend the results to strongly convex losses. We resolve the problem by applying Theorem~\ref{thm:2} of optimistic OMD to this model. 

\begin{myCor}\label{cor:acsd}
In the adversarially corrupted stochastic model, Our Theorem~\ref{thm:2} implies an $\O(\frac{1}{\lambda}$ $\log(\sigma^2 T + C_T))$ bound for $\lambda$-strongly convex expected functions; and Theorem~\ref{thm:3} implies an $\O(\frac{d}{\alpha} \log(\sigma^2 T + C_T))$ bound for $\alpha$-exp-concave individual functions.
\end{myCor}
The proof of Corollary~\ref{cor:acsd} is in Appendix~\ref{appendix:acs}. We successfully extend results of \citet{2021On} not only to strongly convex functions, but also to exp-concave functions.

\subsection{Random Order Model}\label{subsec:rom}
Random Order Model (ROM)~\citep{ROM2020,OptimalROM2021} relaxes the adversarial setting in standard adversarial OCO, where the nature is allowed to choose the set of loss functions even with complete knowledge of the algorithm. However, nature cannot choose the order of loss functions, which will be arranged in uniformly random order. 

Same as \citet{OCO:Between}, let $\bar{\nabla}_T(\x) \triangleq \frac{1}{T} \sum_{s=1}^T \nabla f_s(\x)$. Then we have $\sigma_1^2 = \max_{\x \in \X}$ $\frac{1}{T}\sum_{t=1}^T \| \nabla f_t(\x) - \bar{\nabla}_T(\x)\|_2^2$ and we define $\Lambda = \frac{1}{T}\sum_{t=1}^T \max_{\x \in \X} \| \nabla f_t(\x) - \bar{\nabla}_T(\x) \|_2^2$. Note that $\Lambda$ is a relaxation of $\sigma_1^2$ and the logarithm of $\Lambda/\sigma_1^2$ will not be large in reasonable scenarios. \citet{OCO:Between} establish an $\O(\sigma_1\sqrt{\log (\Lambda/\sigma_1)T})$ bound but require the convexity of individual functions, and they ask whether $\sigma$-dependent regret bounds can be realized under weaker assumptions on convexity of expected functions like~\citet{OptimalROM2021}. In Corollary~\ref{cor:rom:con}, we give an affirmative answer based on  Theorem~\ref{thm:1} and obtain the results with weak assumptions. The proof is in Appendix~\ref{appendix:ROM}.
\begin{myCor}\label{cor:rom:con}
For convex expected functions, ROM enjoys an $\O(\sigma_1\sqrt{\log (\Lambda/\sigma_1)T})$ bound.
\end{myCor}
For $\lambda$-strongly convex expected functions, Theorem~\ref{thm:2} leads to an $\O(\frac{1}{\lambda}\log (T\sigma_1^2 \log(\Lambda/\sigma_1^2)))$ bound, which is more stronger than the $\O(\frac{1}{\lambda}\sigma_1^2 \log T)$ bound of \citet{OCO:Between} when $\sigma_1^2$ is not too small. Meanwhile, the best-of-both-worlds guarantee in Theorem~\ref{thm:2} safeguards that our final bound is never worse than theirs. Besides, for $\alpha$-exp-concave functions, we establish a new $\O(\frac{d}{\alpha}\log (T\sigma_1^2 \log(\Lambda/\sigma_1^2)))$ bound from Theorem~\ref{thm:3}, but the curvature assumption is imposed over individual functions. Thus an open question is whether a similar $\sigma$-dependent bound can be obtained under the convexity of expected functions.

\subsection{Slow Distribution Shift}
We consider a simple problem instance of online learning with slow distribution shifts, in which the underlying distributions selected by the nature in every two adjacent rounds are close on average. Formally, we suppose that $(1/T)\sum_{t=1}^T \sup_{\x \in \X} \|\nabla F_t(\x) - \nabla F_{t-1}(\x)\|_2^2 \leq \epsilon$, where $\epsilon$ is a constant. So we can get that $\Sigma_{1:T}^2 \leq T\epsilon$. For $\lambda$-strongly convex functions, our Theorem~\ref{thm:2} realizes an $\O(\frac{1}{\lambda}\log (\sigma_{1:T}^2 + \epsilon T))$ regret bound, which is tighter than the $\O(\frac{1}{\lambda}(\sigma_{\max}^2\log T+\epsilon T))$ bound of~\citet{OCO:Between} for a large range of $\epsilon$. Extending the analysis to \mbox{$\alpha$-exp-concave} functions yields an $\O(\frac{d}{\alpha}\log (\sigma_{1:T}^2 + \epsilon T))$ regret from Theorem~\ref{thm:3}.

\subsection{Online Learning with Limited Resources}\label{implication;limitedresource}
In real-world online learning applications, functions often arrive not individually but rather in groups. Let $K_t$ denote the number of functions coming in round $t$ and $f_t(\cdot, i)$ denote the $i$-th function. Denote by $F_t(\cdot) \triangleq \frac{1}{K_t}\sum_{i=1}^{K_t} f_t(\cdot, i)$  the average of all functions. 

We consider the scenarios with limited computing resources such that gradient estimation can only be achieved by sampling a portion of the functions, leading to gradient variance. Assume that at each time $t$ we sample $1 \leq B_t \leq K_t$ functions, where the $i$-th function is expressed as $\hat{f}_t(\cdot,i)$. We can then estimate $F_t(\cdot)$ by $f_t(\cdot) \triangleq \frac{1}{B_t}\sum_{i=1}^{B_t} \hat{f}_t(\cdot,i)$, and further we have an upper bound for $\sigma_t^2$ as follows.
\begin{align*}
    \sigma_t^2 ={}& \max_{\x\in\X}\E\left[ \left\|\frac{1}{B_t}\sum_{i=1}^{B_t} \nabla \hat{f}_t(\x,i) - \nabla F_t(\x)\right\|_2^2 \right]\\
    ={}&\frac{1}{B_t^2} \max_{\x\in\X}\bigg ( \sum_{i=1}^{B_t}\E\left[\left\|\nabla \hat{f}_t(\x,i) - \nabla F_t(\x)\right\|_2^2\right]\\
    {}&+\E\bigg[\sum_{i \neq j}\bigg\langle \E \left[\nabla \hat{f}_t(\x,i) - \nabla F_t(\x)\right], \E \left[\nabla \hat{f}_t(\x,j) - \nabla F_t(\x)\right]\bigg\rangle\bigg]  \bigg )\\
    ={}&\frac{1}{B_t^2} \max_{\x\in\mathcal{X}}\left ( \sum_{i=1}^{B_t}\E\left[\left\|\nabla \hat{f}_t(\x,i) - \nabla F_t(\x)\right\|_2^2\right]\right) \leq \frac{4G^2}{B_t},
\end{align*}
where we use the fact that $\nabla \hat{f}_t(\x,i)$ and $\nabla \hat{f}_t(\x,j)$ are independent when $i \neq j$, and the fact that $\E[\nabla \hat{f}_t(\x,i) - \nabla F_t(\x)] = 0$. The last inequality is due to Assumption~\ref{ass:3}. As a result, we have $\sigma_{1:T}^2 = \E[\sum_{t=1}^T \sigma_t^2]\leq 4G^2\sum_{t=1}^T\frac{1}{B_t}$ and obtain the following corollary by substituting it into Theorem~\ref{thm:1}, Theorem~\ref{thm:2}, and Theorem~\ref{thm:3},  respectively. 
\begin{myCor}\label{cor;limitedresource}
    In online learning with limited resources, we can obtain an $\O(2G\sqrt{\sum_{t=1}^T \frac{1}{B_t}} + \sqrt{\Sigma_{1:T}^2})$ bound for convex functions by Theorem~\ref{thm:1}; and Theorem~\ref{thm:2} implies an $\O(\frac{1}{\lambda}\log (4G^2$ $\sum_{t=1}^T \frac{1}{B_t} + \Sigma_{1:T}^2))$ bound for $\lambda$-strongly convex functions; and Theorem~\ref{thm:3} leads to an $\O\big(\frac{d}{\alpha}\log$ $ (4G^2\sum_{t=1}^T \frac{1}{B_t} + \Sigma_{1:T}^2)\big)$ bound for $\alpha$-exp-concave functions.
\end{myCor}
When the number of sampled functions increases, the estimated gradient will gradually approach the real gradient and the  variance will be close to $0$. Note that the ratio $B_t/K_t$ can be viewed as the \emph{data throughput} determined by the available computing resources~\citep{zhou2023stream}. Corollary~\ref{cor;limitedresource} demonstrates the impact of data throughput on learning performance.  

\subsection{Online Label Shift}
\label{subsec:OLaS}
This part demonstrates the application of our results for the SEA model to Online Label Shift (OLS)~\citep{NeurIPS'22:label_shift}. OLS considers a multi-class classification problem in a non-stationary environment, where the label distribution changes over time while the class-conditional is fixed. Denote by $\mathcal{Z}\subseteq \R^{d'}$ the feature space and $\mathcal{Y}=[K]\triangleq \{1,\cdots,K\}$ the label space. OLS consists of a two-stage learning process: during the \emph{offline initialization} stage, the learner trains a well-performed initial model $h_0(\cdot) = h(\x_0,\cdot): \mathcal{Z}\rightarrow \mathcal{Y}$ based on a labeled sample set $S_0 = \{(\z_n,y_n)\}_{n=1}^{N_0}$ drawn from the distribution $\mathfrak{D}_0(\z,y)$; during the \emph{online adaptation} stage, at each round $t \in [T]$, the learner needs to make predictions of a small number of unlabeled data $S_t = \{\z_n\}_{n=1}^{N_t}$ drawn from the distribution $\mathfrak{D}_t(\z)$. The distributions $\mathfrak{D}_t(\z)$ are continuously shifting over time and thereby the learner should update the model $\x_t\in \X$ adaptively. Importantly, a \emph{label shift} assumption is satisfied: the label distribution $\mathfrak{D}_t(y)$ changes over time while the class-conditional distribution $\mathfrak{D}_t(\z\,\vert\, y)$ is identical throughout the process. 

In OLS, the model's quality is evaluated by its risk $F_t(\x) = \E_{(\z,y)\sim \mathfrak{D}_t}[\ell(h(\x,\z),y)]$ in round $t$, where $\ell(\cdot,\cdot)$ can be any convex surrogate loss for classification and $h: \Z\times\W\rightarrow \R^K$ is the predictive function parametrized by $\x$. To cope with the non-stationary environment, we use \emph{dynamic regret} to measure the performance of online algorithms. However, we cannot directly use $F_t$ for updating since it is unknown due to the lack of supervision. To address this problem, \citet{NeurIPS'22:label_shift} rewrite 
\begin{equation}
\label{eq:OLS-F_t}
    F_t(\x) \triangleq \sum_{k=1}^K[\boldsymbol{\mu}_{y_t}]_k \cdot F_0^k(\x), ~\mbox{ with } F_t^k(\x)\triangleq \E_{\z \sim \mathfrak{D}_t(\z\,\vert\, y=k)}[\ell(h(\x,\z),k)]
\end{equation}
where $\boldsymbol{\mu}_{y_t} \in \Delta_K$ denotes the label distribution vector with the $k$-th entry $[\boldsymbol{\mu}_{y_t}]_k \triangleq \mathfrak{D}_t(y=k)$ and $F_t^k(\x)\triangleq \E_{\z \sim \mathfrak{D}_t(\z\,\vert\, y=k)}[\ell(h(\x,\z),k)]$ is the risk of the model over the $k$-th label at round $t$. Note that we use $F_0^k(\x) = F_t^k(\x)$ here, which is due to the assumption that the class-conditional distribution $\mathfrak{D}_t(\z\,\vert\, y)$ remains the same at each time step $t$. Further, they establish an estimator of $F_t(\x)$, defined as
\begin{equation}
\label{eq:OLS-f_t}
    f_t(\x) \triangleq \sum_{k=1}^K\big[\hat{\boldsymbol{\mu}}_{y_t}\big]_k \cdot f_0^k(\x), ~~\mbox{ with } f_0^k(\x) = \frac{1}{|S_0^k|}\sum_{\z_n \in S_0^k}\ell(h(\x,\z_n),k),
\end{equation}
where $S_0^K$ denotes a subset of $S_0$ containing all samples with label $k$ and $\hat{\boldsymbol{\mu}}_{y_t}$ is an estimator of $\boldsymbol{\mu}_{y_t}$ that can be constructed by the Black Box Shift Estimation (BBSE) method~\citep{conf/icml/LiptonWS18}. Specifically, they first obtain the predictive labels $\hat{y}_t$ by using the initial model $h_0$ to predict over the unlabeled data $S_t$, and then compute the label distribution $\boldsymbol{\mu}_{y_t}$ via solving the crucial equation $\boldsymbol{\mu}_{y_t} = C_{h_0}^{-1}\boldsymbol{\mu}_{\hat{y}_t}$, where $\boldsymbol{\mu}_{\hat{y}_t}\in\Delta_K$ is the distribution vector of the predictive labels $\hat{y}_t$ and $C_{h_0}\in \R^{K\times K}$ is the confusion matrix with $[C_{h_0}]_{ij}=\E_{\z\sim \mathfrak{D}_0(\z\,\vert\,y=j)}[\mathbbm{1}\{h_0(\z)=i\}]$. Then $C_{h_0}$ can be estimated empirically by $[\hat{C}_{h_0}]_{ij}=\sum_{(\z,y)\in S_0}\mathbbm{1}\{h_0(\z)=i\,\text{and}\, y=j\}/\mathbbm{1}\{y=j\}$, using the offline labeled data $S_0$. And $\boldsymbol{\mu}_{\hat{y}_t}$ can be estimated empirically with online data $S_t$, which is given by $[\hat{\boldsymbol{\mu}}_{\hat{y}_t}]_j = \frac{1}{|S_t|}\sum_{\z\in S_t}\mathbbm{1}\{h_0(\z)=j\}$. With the above estimation, the final estimator for the label distribution vector is constructed as $\hat{\boldsymbol{\mu}}_{y_t} = \hat{C}_{h_0}^{-1}\hat{\boldsymbol{\mu}}_{\hat{y}_t}$. We further assume that $S_0$ has sufficient samples such that $\hat{C}_{h_0}= C_{h_0}$ and $f_0^k(\x) = F_0^k(\x)$. As a result, $f_t(\x)$ is an unbiased estimator with respect to $F_t(\x)$.

Under such a setup, the SEA model can be applied to the OLS problem. Based on Theorem~\ref{thm:nonsmooth:dynamic}, we can obtain the following theoretical guarantee, whose proof is in Appendix~\ref{appendix:OLaS}.
\begin{myCor}\label{cor:OLaS}
Modeling the online label shift problem as the SEA model with the expected function defined as~\eqref{eq:OLS-F_t} and the randomized function~\eqref{eq:OLS-f_t}, and further applying Algorithm~\ref{alg:non-smooth}, we can obtain that for $\x_t^*\in \argmin_{\x\in\X}F_t(\x)$,
\begin{align*}
    {}&\E\left[\sum_{t=1}^T f_t(\x_t) - \sum_{t=1}^T f_t(\x_t^*)\right]\\
    \leq{}& \O\left(L_T^{\frac{1}{3}} T^{\frac{1}{3}} \left(\sqrt{\sum_{t=1}^T\E\left[\big\|\hat{\boldsymbol{\mu}}_{y_{t}} - \boldsymbol{\mu}_{y_{t}}\big\|_2^2\right]}+\sqrt{\sum_{t=1}^T\E\left[\big\|\boldsymbol{\mu}_{y_{t}} - \boldsymbol{\mu}_{y_{t-1}}\big\|_2^2 \right]}\right)^{\frac{2}{3}} \right),
\end{align*}
where $L_T=\sum_{t=2}^T\big\|\boldsymbol{\mu}_{y_{t}} - \boldsymbol{\mu}_{y_{t-1}}\big\|_1$ measures the label distributions changes.
\end{myCor}

\begin{myRemark}
For the OLS problem, \citet{NeurIPS'22:label_shift} provide an $\O(L_T^{\frac{1}{3}}G_T^{\frac{1}{3}}T^{\frac{1}{3}})$ bound, where $G_T\triangleq \sum_{t=1}^T\E\left[\sup_{\x\in \X}\left\| \nabla f_t(\x) - \nabla H_{t}(\x) \right\|_2^2\right]$ with the hint function $H_t(\x) = \sum_{k=1}^K\left[\boldsymbol{h}_{y_t}\right]_k\cdot f_0^k(\x)$. In fact, when we set $\boldsymbol{h}_{y_t}=\hat{\boldsymbol{\mu}}_{y_{t-1}}$, $G_T$ can be further bounded by 
\begin{align*}
 G_T\leq{}&KG^2\sum_{t=1}^T\E\left[\big\|\hat{\boldsymbol{\mu}}_{y_{t}}-\hat{\boldsymbol{\mu}}_{y_{t-1}}\big\|_2^2 \right] \\
 \leq{}&KG^2 + 2KG^2\sum_{t=2}^T\left(\E\left[\big\|\hat{\boldsymbol{\mu}}_{y_{t}} - \boldsymbol{\mu}_{y_{t}}\big\|_2^2\right]\right) + 2KG^2\sum_{t=2}^T\left(\E\left[\big\|\boldsymbol{\mu}_{y_{t}} - \hat{\boldsymbol{\mu}}_{y_{t-1}}\big\|_2^2 \right]\right)\\
 \leq{}& KG^2 +6KG^2\sum_{t=1}^T\left(\E\left[\big\|\hat{\boldsymbol{\mu}}_{y_{t}} - \boldsymbol{\mu}_{y_{t}}\big\|_2^2\right]\right) + 4KG^2\sum_{t=2}^T\left(\E\left[\big\|\boldsymbol{\mu}_{y_{t}} - \boldsymbol{\mu}_{y_{t-1}}\big\|_2^2 \right]\right).
\end{align*}
Thus our bound in Corollary~\ref{cor:OLaS} is the same as their bound in this case.
\end{myRemark}

\begin{myRemark}\label{remark:olas}
Besides, in the OLS problem, using the bound with $\sigma_{1:T}^2$ (Theorem~\ref{thm:sword}) or $\tilde{\sigma}_{1:T}^2$ (Theorem~\ref{thm:nonsmooth:dynamic}) can actually give the \emph{same} upper bound that scales with  meaningful quantities. Specifically, we can respectively bound $\sigma_{1:T}^2$ and $\tilde{\sigma}_{1:T}^2$ by
\begin{align*}
    {}&\sigma_{1:T}^2 =\E\left[\sum_{t=1}^T\sup_{\x\in\X}\E\left[\left\|\sum_{k=1}^K\big(\big[\hat{\boldsymbol{\mu}}_{y_t}\big]_k - [\boldsymbol{\mu}_{y_t}]_k\big)\cdot \nabla  F_0^k(\x) \right\|_2^2\right]\right]
    \leq KG^2\sum_{t=1}^T\E\left[\big\|\hat{\boldsymbol{\mu}}_{y_{t}} - \boldsymbol{\mu}_{y_{t}}\big\|_2^2\right],\\
    {}&\tilde{\sigma}_{1:T}^2 =\E\left[\sum_{t=1}^T\E\left[\sup_{\x\in\X}\left\|\sum_{k=1}^K\big(\big[\hat{\boldsymbol{\mu}}_{y_t}\big]_k - [\boldsymbol{\mu}_{y_t}]_k\big)\cdot \nabla  F_0^k(\x) \right\|_2^2\right]\right]
    \leq KG^2\sum_{t=1}^T\E\left[\big\|\hat{\boldsymbol{\mu}}_{y_{t}} - \boldsymbol{\mu}_{y_{t}}\big\|_2^2\right].
\end{align*}
Both quantities share the same upper bound in the form of label distribution variances.
\end{myRemark}

%% file: sections/conclusion.tex

\section{Conclusion and Future Work}
\label{sec:conclusion}
In this paper, we investigate the Stochastically Extended Adversarial (SEA) model of \citet{OCO:Between} and propose a different solution via the optimistic OMD framework. Our results yield the \emph{same} regret bound for convex and smooth functions under weaker assumptions and a \emph{better} regret bound for strongly convex and smooth functions; moreover, we establish the \emph{first} regret bound for exp-concave and smooth functions. For all three cases, we further improve analyses of optimistic FTRL, proving equal regret bounds with optimistic OMD for the SEA model. Furthermore, we study the SEA model under \emph{dynamic regret} and propose a new two-layer algorithm based on optimistic OMD, which obtains the \emph{first} dynamic regret guarantee for the SEA model. Additionally, we further explore the SEA model under \emph{non-smooth} scenarios, in which we propose to use OMD with an implicit update to achieve static and dynamic regret guarantees. Lastly, we discuss implications for intermediate learning scenarios, leading to various new results.

Although our algorithms for various functions can be unified using the optimistic OMD framework, they still necessitate distinct configurations for parameters such as step sizes and regularizers. Consequently, it becomes crucial to conceive and develop more adaptive online algorithms that eliminate the need for pre-set parameters. Exploring this area of research and designing such algorithms will be an important focus in future studies.

%% file: appendixes/appendix.tex

\section{Omitted Proofs for Section~\ref{sec:results}}
\label{appendix:proofs-section4}
This section contains the omitted proofs of optimistic FTRL for Section~\ref{sec:results}, including Theorem~\ref{thm:convex-FTRL},~\ref{thm:FTRL},~\ref{thm:FTRL-exp-concave} in \pref{appendix:proofs-convex-FTRL}--\pref{appendix:proofs-exp-FTRL}, followed by useful lemmas in~\pref{appendix:useful-lemmas}.

\subsection{{Proof of Theorem~\ref{thm:convex-FTRL}}}
\begin{proof}
\label{appendix:proofs-convex-FTRL}
For convex and smooth functions, we start by outlining the optimistic FTRL procedure. At each step $t$, a surrogate loss is defined: $\ell_t(\x) = \langle \nabla f_t(\x_t), \x -\x_t \rangle $. Unlike \citet{OCO:Between}, we use this surrogate loss instead of the original function $f_t(\cdot)$ to update $\x_t$, avoiding the need for convexity in individual functions (which is required by~\citet{OCO:Between}). The decision $\x_{t}$ is updated by deploying optimistic FTRL over the linearized loss:
\begin{align*}
    \x_{t} = \argmin_{\x \in \X} \sum_{s=1}^{t-1} \left\{\ell_s (\x) + \left\langle M_t, \x \right\rangle  + \frac{1}{\eta_t}\|\x\|_2^2\right\},
\end{align*}
where $\x_0$ can be an arbitrary point in $\X$, and the optimistic vector $M_t =\nabla f_{t-1}(\x_{t-1})$ (we set $M_1 = \nabla f_0(\x_0) = 0$). The step size $\eta_t$ is designed as $\eta_t = D^2/(\delta + \sum_{s=1}^{t-1} \eta_s \|\nabla f_s(\x_s) - f_{s-1}(\x_{s-1}) \|_2^2)$ with $\delta$ to be defined latter, which is non-increasing for $t \in [T]$. 

We can easily obtain that 
\begin{align*}
    \E \left[\sum_{t=1}^T f_t(\x_t) - \sum_{t=1}^T f_t(\u)\right] \leq \E \left[\sum_{t=1}^T \langle  \nabla f_t(\x_t) , \x_t - \u \rangle\right] \leq \E\left[\sum_{t=1}^T \ell_t(\x_t) - \sum_{t=1}^T \ell_t(\u)\right].
\end{align*}
As a result, we only need to consider the regret of the surrogate loss $\ell_t(\cdot)$. The following proof is similar to \citet{OCO:Between}. To exploit Lemma~\ref{lem:FTRL} (Standard analysis of optimistic FTRL), we map the $G_t$ term in Lemma~\ref{lem:FTRL} to $\frac{1}{\eta_t}\|\x\|_2^2 + \sum_{s=1}^{t-1}\ell_s(\x)$ and map the $\tilde{\g}_t$ term to $M_t$. Note that $G_t$ is $\frac{2}{\eta_t}$-strongly convex and $\ell_t$ is convex, we have
\begin{align}
 \sum_{t=1}^T \ell_t(\x_t) - \sum_{t=1}^T \ell_t(\u)  \leq{}& \frac{D^2}{\eta_T}  + \sum_{t=1}^T \left( \langle \nabla f_t(\x_t) - \nabla f_{t-1}(\x_{t-1}), \x_t - \x_{t+1} \rangle -\frac{1}{\eta_t}\|\x_t - \x_{t+1}\|_2^2 \right)\nonumber\\
 \leq{}& \frac{D^2}{\eta_T}  +  \sum_{t=1}^T \left( \frac{\eta_t}{2}\|\nabla f_t(\x_t) - \nabla f_{t-1}(\x_{t-1})\|_2^2 - \frac{1}{2\eta_t}\|\x_t - \x_{t+1}\|_2^2 \right) \nonumber\\
\leq{}& \delta + \frac{3}{2}\sum_{t=1}^{T} \eta_t \|\nabla f_t(\x_t) - f_{t-1}(\x_{t-1})\|_2^2 - \frac{\delta}{2D^2}\sum_{t=1}^T\|\x_t - \x_{t+1}\|_2^2 \nonumber\\
\leq{}& \frac{3\sqrt{2}}{2} D\sqrt{\Bar{V}_T} + \frac{6D^2G^2}{\delta} + \delta- \frac{\delta}{2D^2}\sum_{t=1}^T\|\x_t - \x_{t+1}\|_2^2 .\nonumber
\end{align}
where we use the fact that $\langle a,b \rangle \leq \|a\|_{*}\|b\| \leq \frac{1}{2c}\|a\|_{*}^2 + \frac{c}{2}\|b\|^2$ in the second inequality ($\|\cdot\|_{*}$ denotes the dual norm of $\|\cdot\|$), based on the H$\Ddot{\text{o}}$lder's inequality. The third step is due to the fact $\eta_t \leq\frac{D^2}{\delta}(t\in[T])$ and the 
last step use the inequality $\sum_{t=1}^{T} \eta_t \|\nabla f_t(\x_t) - f_{t-1}(\x_{t-1})\|^2\leq D\sqrt{2\Bar{V}_T} + \frac{4D^2G^2}{\delta}$ from \citet[proof of Theorem 5]{OCO:Between}.

Using Lemma~\ref{lem:sum:diff} (Boundedness of cumulative norm of the gradient difference), we have
\begin{align}
& \sum_{t=1}^T \ell_t(\x_t) - \sum_{t=1}^T \ell_t(\u) \nonumber\\
\leq{}& 6D\sqrt{ \sum_{t=1}^{T}  \|\nabla f_t(\x_t) -\nabla F_t(\x_t)\|_2^2} + 3\sqrt{2}D\sqrt{ \sum_{t=2}^{T} \|\nabla F_t(\x_{t-1}) -\nabla F_{t-1}(\x_{t-1})\|_2^2} \nonumber\\
&+ 3\sqrt{2}DL\sqrt{ \sum_{t=2}^{T} \|\x_t-\x_{t-1}\|_2^2} - \frac{\delta}{2D^2}\sum_{t=1}^T\|\x_t - \x_{t+1}\|_2^2 +\frac{6D^2G^2}{\delta} + \delta +\frac{3\sqrt{2}}{2} DG\nonumber\\
\leq{}& 6D\sqrt{ \sum_{t=1}^{T}  \|\nabla f_t(\x_t) -\nabla F_t(\x_t)\|_2^2} + 3\sqrt{2}D\sqrt{ \sum_{t=2}^{T} \|\nabla F_t(\x_{t-1}) -\nabla F_{t-1}(\x_{t-1})\|_2^2} \nonumber\\
& + \frac{9D^4L^2}{\delta} +\frac{6D^2G^2}{\delta} + \delta +\frac{3\sqrt{2}}{2} DG, \label{eqn:FTRL:convex:regret3}
\end{align}
where we use the following inequality in the last step $3\sqrt{2}DL\sqrt{ \sum_{t=2}^{T} \|\x_t-\x_{t-1}\|_2^2} \leq \frac{9D^4L^2}{\delta} + \frac{\delta}{2D^2}\sum_{t=1}^T\|\x_t - \x_{t+1}\|_2^2$, canceling out the the negative term in \eqref{eqn:FTRL:convex:regret3} with the second term. 

Then, we take expectations over \eqref{eqn:FTRL:convex:regret3} with the help of definitions of $\sigma_{1:T}^2$ and $\Sigma_{1:T}^2$, and use Jensen's inequality. Given that the expected regret of surrogate loss functions upper bounds the expected regret of original functions, we get the final result:
\begin{align*}
\E\left[\sum_{t=1}^T f_t(\x_t) - \sum_{t=1}^T f_t(\u)\right] \leq {}&\E\left[\sum_{t=1}^T \ell_t(\x_t) - \sum_{t=1}^T \ell_t(\u)\right]\\
  \leq{}& 6D\sqrt{\sigma_{1:T}^2} +3\sqrt{2}D\sqrt{\Sigma_{1:T}^2} + \frac{9D^4L+6D^2G^2}{\delta} + \delta + \frac{3\sqrt{2}}{2} DG\\
  ={}&6D\sqrt{\sigma_{1:T}^2} +3\sqrt{2}D\sqrt{\Sigma_{1:T}^2} + 2\sqrt{9D^4L+6D^2G^2} + \frac{3\sqrt{2}}{2} DG \\
  ={}&\O\Big(\sqrt{\sigma_{1:T}^2}+ \sqrt{\Sigma_{1:T}^2}\Big),
\end{align*}
where we set $\delta = \sqrt{9D^4L+6D^2G^2}$. Hence, we complete the proof.
\end{proof}

\subsection{Proof of Theorem~\ref{thm:FTRL}}
\begin{proof}
\label{appendix:proofs-sc-FTRL}
We first present the procedure of optimistic FTRL for $\lambda$-strongly convex and smooth functions \citep{OCO:Between}. In each round $t$, we define a new surrogate loss: $\ell_t(\x) = \langle \nabla f_t(\x_t), \x -\x_t \rangle + \frac{\lambda}{2} \|\x -\x_t\|_2^2$. And the decision $\x_{t+1}$ is determined by
\begin{align*}
\x_{t+1} = \argmin_{\x \in \X} \left\{ \frac{\lambda}{2} \|\x -\x_0\|_2^2+ \sum_{s=1}^{t} \ell_s(\x) +  \langle \m_{t+1}, \x \rangle \right\},
\end{align*}
where $\x_0$ is an arbitrary point in $\X$, and the optimistic vector $\m_{t+1}=\nabla f_t(\x_t)$. In the beginning, we set $\m_1=\nabla f_0(\x_0)=0$ and thus $\x_1=\x_0$. Compared with the original algorithm of \citet{OCO:Between}, we insert an additional  $\frac{\lambda}{2} \|\x -\x_0\|_2^2$ term in the updating rule above, and in this way, the objective function in the $t$-th round is $\lambda t$-strongly convex, which facilitates the subsequent analysis.

According to (\ref{eqn:exp:str}), it is easy to verify that
\begin{align}
\E\left[ \sum_{t=1}^T f_t(\x_t) - \sum_{t=1}^T f_t(\u) \right] \leq \E\left[\sum_{t=1}^T  \ell_t(\x_t) - \sum_{t=1}^T \ell_t(\u) \right]. \label{eqn:exp:str:FTRL1}
\end{align}
Thus, we can focus on the regret of surrogate loss $\ell_t(\cdot)$. From Lemma~\ref{lem:FTRL} (Standard analysis of optimistic FTRL), since $\frac{\lambda}{2} \|\x -\x_0\|_2^2+ \sum_{s=1}^{t-1} \ell_s(\x)$ is $\lambda t$-strongly convex, we obtain
\begin{align}
 {}&\sum_{t=1}^T  \ell_t(\x_t) - \sum_{t=1}^T \ell_t(\u)  \nonumber\\
 \leq{}&  \frac{\lambda}{2} \|\u -\x_0\|_2^2 + \sum_{t=1}^T  \langle \nabla f_t(\x_t)- \nabla f_{t-1}(\x_{t-1}), \x_t -\x_{t+1}\rangle - \sum_{t=1}^T\frac{\lambda t}{2} \|\x_t - \x_{t+1}\|_2^2 \nonumber\\ 
 \leq{}& \frac{\lambda D^2}{2}+ \sum_{t=1}^T   \frac{1}{\lambda t} \|\nabla f_t(\x_t)- \nabla f_{t-1}(\x_{t-1})\|_2^2 - \frac{\lambda }{4} \sum_{t=1}^T\|\x_t -\x_{t+1}\|_2^2. \label{eqn:exp:str:FTRL2}
\end{align}
Then we directly use Lemma~\ref{lem:ind:diff} (Boundedness of the norm of gradient difference) to obtain
\begin{align}
   {}& \sum_{t=1}^T  \ell_t(\x_t) - \sum_{t=1}^T \ell_t(\u) \nonumber \\
     \leq{}& \frac{G^2}{\lambda} + \sum_{t=2}^T   \frac{1}{\lambda t} \left( 4 \|\nabla f_t(\x_t) -\nabla F_t(\x_t)\|_2^2 + 4 \|\nabla F_t(\x_{t-1}) -\nabla F_{t-1}(\x_{t-1})\|_2^2 \right. \nonumber\\
    &\left.+ 4 \|\nabla F_{t-1}(\x_{t-1}) -  \nabla f_{t-1} (\x_{t-1})\|_2^2 \right) + \sum_{t=1}^{T} \left(\frac{4L^2}{\lambda (t+1)} - \frac{\lambda t}{4}\right) \|\x_t-\x_{t-1}\|_2^2 + \frac{\lambda D^2}{2} \nonumber\\
    \leq{}& \frac{G^2}{\lambda} + \sum_{t=1}^{T} \frac{8}{\lambda t}\|\nabla f_t(\x_t) -\nabla F_t(\x_t)\|_2^2 +\sum_{t=2}^{T} \frac{4}{\lambda t}\|\nabla F_t(\x_{t-1}) -\nabla F_{t-1}(\x_{t-1})\|_2^2 \nonumber\\
    {}& + \sum_{t=1}^{T} \left(\frac{4L^2}{\lambda t} - \frac{\lambda t}{4}\right) \|\x_t-\x_{t-1}\|_2^2 + \frac{\lambda D^2}{2}.\label{eqn:FTRL:sum2}
\end{align}
The above formula reuses the simplification techniques in (\ref{lem4:2}). Still defining $\kappa = \frac{L}{\lambda}$, then for $t \geq 16\kappa$, there is $\frac{4L^2}{\lambda t} - \frac{\lambda t}{4} \leq 0$. For this reason, it turns out that
\begin{align*}
    {}&\sum_{t=1}^{T} \left(\frac{4L^2}{\lambda t} - \frac{\lambda t}{4}\right) \|\x_t-\x_{t-1}\|_2^2
    \leq \sum_{t=1}^{\lceil 16\kappa \rceil} \left(\frac{4L^2}{\lambda t} - \frac{\lambda t}{4}\right)D^2
    \leq \frac{4L^2 D^2}{\lambda} \sum_{t=1}^{\lceil 16\kappa \rceil} \frac{1}{t}\\
    \leq{}& \frac{4L^2 D^2}{\lambda} \left( 1 + \int_{t=1}^{\lceil 16\kappa \rceil}\frac{1}{t} \right)
    =  \frac{4L^2 D^2}{\lambda}\ln\left(1+16\frac{L}{\lambda}\right) + \frac{4L^2 D^2}{\lambda}. 
\end{align*}
By substituting the above inequality into (\ref{eqn:FTRL:sum2}) and taking the expectation, we can obtain
\begin{align*}
    {}&\E\left[\sum_{t=1}^T  \ell_t(\x_t) - \sum_{t=1}^T \ell_t(\u)\right]\\
    \leq{}&\E\left[\sum_{t=1}^{T} \frac{8}{\lambda t}\sigma_t^2 +\sum_{t=2}^{T} \frac{4}{\lambda t}\sup_{\x\in \X} \|\nabla F_t(\x) -\nabla F_{t-1}(\x)\|_2^2\right] + \frac{4L^2 D^2}{\lambda}\ln\left(1+16\frac{L}{\lambda}\right) \\
    {}&+ \frac{4L^2 D^2+G^2}{\lambda}+\frac{\lambda D^2}{2}.
\end{align*}
Similar to the derivation using optimistic OMD, we take advantage of Lemma~\ref{lemma:strongly-convex-lemma} to get
\begin{align*}
    {}&\E\left[\sum_{t=1}^T  \ell_t(\x_t) - \sum_{t=1}^T \ell_t(\u)\right]\\
    \leq{}&\frac{8\sigma_{\max}^2+4\Sigma_{\max}^2}{\lambda}\ln \left(\frac{1}{2\sigma_{\max}^2+\Sigma_{\max}^2}\left(2\sigma_{1:T}^2 + \Sigma_{1:T}^2\right)+1\right) + \frac{8\sigma_{\max}^2+4\Sigma_{\max}^2 +4}{\lambda}\\
    {}&+\frac{4L^2 D^2}{\lambda}\ln\left(1+16\frac{L}{\lambda}\right)+ \frac{4L^2 D^2+G^2}{\lambda}+\frac{\lambda D^2}{2}\\
    ={}&\O\left(\frac{1}{\lambda}\left(\sigma_{\max}^2+\Sigma_{\max}^2\right)\log \left(\left(\sigma_{1:T}^2 + \Sigma_{1:T}^2\right)/\left(\sigma_{\max}^2+\Sigma_{\max}^2\right)\right)\right),
\end{align*}
which completes the proof.
\end{proof}

\input{appendixes/appendix_exp_concave_FTRL.tex}

\subsection{Useful Lemmas}
\label{appendix:useful-lemmas}
We first provide the proof of Lemma~\ref{lem:1},~\ref{lem:sum:diff},~\ref{lem:ind:diff} and~\ref{lemma:strongly-convex-lemma}, and then present other lemmas useful for the proofs.

~\\
\begin{proof}[{of Lemma~\ref{lem:1}}]
We can decompose the instantaneous loss as
\begin{align*}
    {}&\langle \nabla f_{t}(\x_{t}), \x_{t} - \x \rangle \\
    ={}& \underbrace{\langle \nabla f_{t}(\x_{t})- \nabla f_{t-1}(\x_{t-1}), \x_t -\hat{\x}_{t+1}\rangle}_{\term{a}} + \underbrace{\langle \nabla f_{t-1}(\x_{t-1}), \x_t -\hat{\x}_{t+1} \rangle}_{\term{b}} +\underbrace{ \langle \nabla f_{t}(\x_{t}), \hat{\x}_{t+1}-\x \rangle}_{\term{c}}.
\end{align*}
For term (a), we use Lemma~\ref{lem:stability} (stability lemma) and get that
\begin{align*}
    \term{a} ={}& \langle \nabla f_{t}(\x_{t})- \nabla f_{t-1}(\x_{t-1}), \x_t -\hat{\x}_{t+1}\rangle\\
    \leq{}&\|\nabla f_{t}(\x_{t})- \nabla f_{t-1}(\x_{t-1})\|_{*}\|\x_t -\hat{\x}_{t+1}\| \leq \frac{1}{\alpha} \|\nabla f_{t}(\x_{t}) -\nabla f_{t-1}(\x_{t-1}) \|_*^2.
\end{align*}
For term (b) and term (c), due to the  updating rules of optimistic OMD in (\ref{eqn:update:u}) and (\ref{eqn:update:x}), we can apply Lemma~\ref{lem:bregman_proximal} (Bregman proximal inequality) and obtain that
\begin{align*}
    &\term{b}=\langle \nabla f_{t-1}(\x_{t-1}), \x_t -\hat{\x}_{t+1} \rangle \leq \D_{\MR_t}(\hat{\x}_{t+1},\hat{\x}_t)-\D_{\MR_t}(\hat{\x}_{t+1},\x_t)-\D_{\MR_t}(\x_{t},\hat{\x}_t),\\
    &\term{c}=\langle \nabla f_{t}(\x_{t}), \hat{\x}_{t+1}-\x \rangle\leq \D_{\MR_t}(\x,\hat{\x}_t)-\D_{\MR_t}(\x,\hat{\x}_{t+1})-\D_{\MR_t}(\hat{\x}_{t+1},\hat{\x}_t). 
\end{align*}
We complete the proof by combining the three upper bounds.
\end{proof}

\begin{proof}[of Lemma~\ref{lem:sum:diff}]
It is easy to verify the above lemma by substituting (\ref{eqn:dif:grad:1}) and (\ref{eqn:dif:grad:2}) in Lemma~\ref{lem:ind:diff} into $\sum_{t=1}^{T} \|\nabla f_t(\x_t) - \nabla f_{t-1} (\x_{t-1}) \|_2^2$ and simplifying the result.
\end{proof}

\begin{proof}[of Lemma~\ref{lem:ind:diff}]
For $t \geq 2$, from Jensen's inequality and Assumption~\ref{ass:2} (smoothness of expected function), we have
\begin{equation} \label{eqn:dif:grad:1}
\begin{split}
 &\|\nabla f_t(\x_t) - \nabla f_{t-1} (\x_{t-1}) \|_2^2 \\
=& 16 \left\| \frac{1}{4} \big[\nabla f_t(\x_t) -\nabla F_t(\x_t)\big]  + \frac{1}{4} \big[ \nabla F_t(\x_t)   - \nabla F_t(\x_{t-1})\big]   \right.\\
 & \left.+  \frac{1}{4} \big[\nabla F_t(\x_{t-1}) -\nabla F_{t-1}(\x_{t-1}) \big] + \frac{1}{4} \big[ \nabla F_{t-1}(\x_{t-1}) -  \nabla f_{t-1} (\x_{t-1}) \big] \right\|_2^2\\
\leq & 4 \|\nabla f_t(\x_t) -\nabla F_t(\x_t)\|_2^2 +  4 \|\nabla F_t(\x_t)   - \nabla F_t(\x_{t-1})\|_2^2 \\
 & + 4 \|\nabla F_t(\x_{t-1}) -\nabla F_{t-1}(\x_{t-1})\|_2^2+ 4 \|\nabla F_{t-1}(\x_{t-1}) -  \nabla f_{t-1} (\x_{t-1})\|_2^2\\
\leq & 4 \|\nabla f_t(\x_t) -\nabla F_t(\x_t)\|_2^2 +  4L^2 \|\x_t-\x_{t-1}\|_2^2 \\
 & + 4 \|\nabla F_t(\x_{t-1}) -\nabla F_{t-1}(\x_{t-1})\|_2^2+ 4 \|\nabla F_{t-1}(\x_{t-1}) -  \nabla f_{t-1} (\x_{t-1})\|_2^2.
\end{split}
\end{equation}

For $t=1$, from Assumption~\ref{ass:3} (boundedness of the gradient norm), we have
\begin{equation}\label{eqn:dif:grad:2}
\|\nabla f_1(\x_1) - \nabla f_{0} (\x_{0}) \|_2^2=\|\nabla f_1(\x_1)  \|_2^2 \leq G^2 .
\end{equation}

Combining both cases finishes the proof.
\end{proof}

\begin{proof}[of Lemma~\ref{lemma:strongly-convex-lemma}]
We first define the following quantity:
\begin{align*}
    \alpha = \left\lceil \sum_{t=1}^T\frac{1}{2\sigma_{\max}^2+\Sigma_{\max}^2}\left(2\sigma_t^2 + \sup_{\x\in \X} \|\nabla F_t(\x) -\nabla F_{t-1}(\x)\|_2^2\right) \right\rceil.
\end{align*}
If $1\leq \alpha < T$, we bound $\sum_{t=1}^T\frac{1}{\lambda t}\left(2\sigma_t^2 + \sup_{\x\in \X} \|\nabla F_t(\x) -\nabla F_{t-1}(\x)\|_2^2\right)$ as follows.
\begin{align*}
    {}&\sum_{t=1}^T\frac{1}{\lambda t}\left(2\sigma_t^2 + \sup_{\x\in \X} \|\nabla F_t(\x) -\nabla F_{t-1}(\x)\|_2^2\right)\\
    ={}&\sum_{t=1}^{\alpha} \frac{1}{\lambda t}\left(2\sigma_t^2 + \sup_{\x\in \X} \|\nabla F_t(\x) -\nabla F_{t-1}(\x)\|_2^2\right) + \sum_{t=\alpha+1}^T \frac{1}{\lambda t}\left(2\sigma_t^2 + \sup_{\x\in \X} \|\nabla F_t(\x) -\nabla F_{t-1}(\x)\|_2^2\right)\\
    \leq{}&\frac{2\sigma_{\max}^2+\Sigma_{\max}^2}{\lambda}\sum_{t=1}^{\alpha}\frac{1}{t} + \frac{1}{\lambda(\alpha+1)}\sum_{t=\alpha+1}^T\left(2\sigma_t^2 + \sup_{\x\in \X} \|\nabla F_t(\x) -\nabla F_{t-1}(\x)\|_2^2\right)\\
    \leq{}&\frac{2\sigma_{\max}^2+\Sigma_{\max}^2}{\lambda}\left(1+\int_{t=1}^{\alpha}\frac{1}{t}dt\right) + \frac{2\sigma_{\max}^2+\Sigma_{\max}^2}{\lambda}\leq \frac{2\sigma_{\max}^2+\Sigma_{\max}^2}{\lambda}(\ln \alpha +1)+ \frac{2\sigma_{\max}^2+\Sigma_{\max}^2}{\lambda}\\
    \leq{}&\frac{2\sigma_{\max}^2+\Sigma_{\max}^2}{\lambda}\ln \left(\sum_{t=1}^T\frac{1}{2\sigma_{\max}^2+\Sigma_{\max}^2}\left(2\sigma_t^2 + \sup_{\x\in \X} \|\nabla F_t(\x) -\nabla F_{t-1}(\x)\|_2^2\right)+1\right) + \frac{4\sigma_{\max}^2+2\Sigma_{\max}^2}{\lambda}.
\end{align*}
Else if $\alpha = T$, we have
\begin{align*}
        {}&\sum_{t=1}^T\frac{1}{\lambda t}\left(2\sigma_t^2 + \sup_{\x\in \X} \|\nabla F_t(\x) -\nabla F_{t-1}(\x)\|_2^2\right)\\
    = {}&\sum_{t=1}^{\alpha} \frac{1}{\lambda t}\left(2\sigma_t^2 + \sup_{\x\in \X} \|\nabla F_t(\x) -\nabla F_{t-1}(\x)\|_2^2\right) 
    \leq \frac{2\sigma_{\max}^2+\Sigma_{\max}^2}{\lambda}\sum_{t=1}^{\alpha}\frac{1}{t} \\
    \leq{}&\frac{2\sigma_{\max}^2+\Sigma_{\max}^2}{\lambda}\left(1+\int_{t=1}^{\alpha}\frac{1}{t}dt\right) \leq \frac{2\sigma_{\max}^2+\Sigma_{\max}^2}{\lambda}(\ln \alpha +1)\\
    \leq{}&\frac{2\sigma_{\max}^2+\Sigma_{\max}^2}{\lambda}\ln \left(\sum_{t=1}^T\frac{1}{2\sigma_{\max}^2+\Sigma_{\max}^2}\left(2\sigma_t^2 + \sup_{\x\in \X} \|\nabla F_t(\x) -\nabla F_{t-1}(\x)\|_2^2\right)+1\right) + \frac{2\sigma_{\max}^2+\Sigma_{\max}^2}{\lambda}.
\end{align*}
Else if $\alpha < 1$, we have
\begin{align*}
    {}&\sum_{t=1}^T\frac{1}{\lambda t}\left(2\sigma_t^2 + \sup_{\x\in \X} \|\nabla F_t(\x) -\nabla F_{t-1}(\x)\|_2^2\right)\\
    \leq{}&\frac{1}{\lambda \alpha}\sum_{t=1}^T\left(2\sigma_t^2 + \sup_{\x\in \X} \|\nabla F_t(\x) -\nabla F_{t-1}(\x)\|_2^2\right) \leq \frac{2\sigma_{\max}^2+\Sigma_{\max}^2}{\lambda}.
\end{align*}
\end{proof}

\begin{myLemma}\label{lem:ln:pq}
Let $A_T$ be a non-negative term, $a,b$ be non-negative constants and $c$ be a positive constant, then we have \begin{equation}
    a\ln(b A_T + 1) - c A_T \leq a\ln \left(\frac{ab}{c}+1\right).
\end{equation}

\begin{proof}
We use the following inequality: $\ln p \leq \frac{p}{q}+ \ln q  - 1$ holds for all $p >0,\, q > 0$. By setting $p = b A_T + 1$ and $q = \frac{ab}{c}+1$, we obtain
\begin{align*}
    {}& a\ln(b A_T + 1) - c A_T\leq a\bigg(\frac{b A_T + 1}{ab/c+1} + \ln\bigg(\frac{ab}{c}+1\bigg) -1\bigg) - cA_T\\
    ={}& c\bigg(\frac{ab}{ab+c} - 1\bigg)A_T + \bigg(\frac{1}{ab/c+1}-1\bigg)a + a\ln\bigg(\frac{ab}{c}+1\bigg) \leq a \ln\bigg(\frac{ab}{c}+1\bigg).
\end{align*}
\end{proof}  
\end{myLemma}

\input{appendixes/appendix_extensions}

\section{Omitted Proofs for Section~\ref{sec:examples}}
\label{appendix:proofs-section5}
This section presents omitted proofs of corollaries in Section~\ref{sec:examples}, including proof of Corollary~\ref{cor:acsd} in~\pref{appendix:acs}, proof of Corollary~\ref{cor:rom:con} in~\pref{appendix:ROM} and proof of Corollary~\ref{cor:OLaS} in~\pref{appendix:OLaS}.

\subsection{Proof of Corollary~\ref{cor:acsd}}\label{appendix:acs}
\begin{proof}
Recall that in Section~\ref{ACSD} the loss functions in adversarially corrupted stochastic model satisfy $f_t(\x) = h_t(\x) + c_t(\x)$ for all $t\in [T]$, where $h_t(\cdot)$ is sampled from a fixed distribution every iteration and $\sum_{t=1}^T$ $ \max_{\x\in \X}\|\nabla c_t(\x)\| \leq C_T$. By definition of $F_t(\x)$, 
\begin{align}
F_t(\x) = \E_{f_t\sim \mathfrak{D}_t}[f_t(\x)] = \E_{h_t\sim \mathfrak{D}}[h_t(\x) + c_t(\x)]
= \E_{h_t\sim \mathfrak{D}}[h_t(\x)] + c_t(\x).   \label{eqn:acs:Ft}
\end{align}
Since $h_t(\cdot)$ is i.i.d for each $t$, their expectations are the same. Then we have
\begin{align*}
    \| \nabla F_t(\x) - \nabla F_{t-1}(\x)\|_2^2 
    \leq 2G\| \nabla F_t(\x) - \nabla F_{t-1}(\x)\|_2 
    \overset{(\ref{eqn:acs:Ft})}{=}{}& 2G \| \nabla c_t(\x) - \nabla c_{t-1}(\x)\| \nonumber\\
    \leq{}& 2G(\|\nabla c_t(\x)\| + \|\nabla c_{t-1}(\x)\|).
\end{align*}
Therefore, we have the following upper bound for the cumulative variation:
\begin{align*}
    \Sigma_{1:T}^2 = \E\left[\sum_{t=2}^T \sup_{\x \in \X} \|\nabla F_t(\x) - \nabla F_{t-1}(\x)\|_2^2 \right] \leq \sum_{t=2}^T \sup_{\x \in \X} 2G(\|\nabla c_t(\x)\| + \|\nabla c_{t-1}(\x)\|) \leq 4GC_T.
\end{align*}
Besides, we can calculate the variance as 
\begin{align*}
\sigma_t^2 = \max_{\x\in \X}\E_{f_t\sim \mathfrak{D}_t}[\|\nabla f_t(\x) - \nabla F_t(\x)\|_2^2] \overset{\eqref{eqn:acs:Ft}}{=}{} \max_{\x\in \X}\E_{h_t\sim \mathfrak{D}_t}[\|\nabla h_t(\x) - \nabla \E_{h_t\sim \mathfrak{D}}[h_t(\x_t)] \|_2^2] = \sigma,
\end{align*}
where $\sigma > 0$ is the variance of stochastic gradients. This implies $
\sigma_{1:T}^2 =\E \left[ \sum_{t=1}^T \sigma_t^2 \right] = \sigma T$. 

Combining the above two upper bounds of $\sigma_{1:T}^2$ and $\Sigma_{1:T}^2$ with the regret bounds of optimistic OMD in Theorem~\ref{thm:2} and Theorem~\ref{thm:3} completes the proof.
\end{proof}

\subsection{Proof of Corollary~\ref{cor:rom:con}}\label{appendix:ROM}
\begin{proof}
The difference between ROM and i.i.d. stochastic model is that ROM samples a loss from the loss set without replacement in each round, while i.i.d. stochastic model samples independently and uniformly with replacement in each round. However, following \citet{OCO:Between}, we can bound the variance of ROM with respect to $\mathfrak{D}_t$ for each $t$ by the variance $\sigma_1^2$ of the first round, which can also be regarded as the variance of the i.i.d. model for every round. Specifically, for $\forall \x \in \X$ and every $t \in [ T ]$, we have
\begin{align}
    \E_{f_t \sim \mathfrak{D}_t}\left[ \| \nabla f_t(\x) - \nabla F_t(\x) \|_2^2 \right] \leq \E_{f_t \sim \mathfrak{D}_t}\left[ \| \nabla f_t(\x) - \nabla F_1(\x) \|_2^2 \right] . \label{appendix:rom:con:eqn1}
\end{align}
Since ROM samples losses without replacement, let set $\Gamma_t$ represent the index set of losses that can be selected in the $t$th round, thus $\Gamma_1 = [ T ]$, then we have
\begin{align*}
    \E_{f_t \sim \mathfrak{D}_t}\left[ \| \nabla f_t(\x) - \nabla F_1(\x) \|_2^2 \right] & = \frac{1}{T-(t-1)}\sum_{i \in \Gamma_t} \| \nabla f_i(\x) - \nabla F_1(\x) \|_2^2 \\
    & \leq \frac{1}{T-(t-1)} \sum_{i \in \Gamma_1}\| \nabla f_i(\x) - \nabla F_1(\x) \|_2^2 \leq \frac{T}{T-(t-1)} \sigma_1^2. 
\end{align*}
So combining (\ref{appendix:rom:con:eqn1}) with the above inequality, we get that
\begin{align}
    \E_{f_t \sim \mathfrak{D}_t}\left[ \| \nabla f_t(\x) - \nabla F_t(\x) \|_2^2 \right] \leq \frac{T}{T-(t-1)} \sigma_1^2,\,\,\forall \x \in \X,\, t \in [ T]. \label{appendix:rom:con:eqn3}
\end{align}
Besides, from (\ref{appendix:rom:con:eqn1}), we can also get that
\begin{align}
    \E\left [ \sigma_t^2 \right] & \leq \E \left[ \max_{\x \in \X} \E_{f_t \sim \mathfrak{D}_t} \left[\| \nabla f_t(\x) - \nabla F_1 (\x) \|_2^2 \right] \right]\nonumber\\
    &\leq \E \left[  \E_{f_t \sim \mathfrak{D}_t} \left[\max_{\x \in \X}\| \nabla f_t(\x) - \nabla F_1 (\x) \|_2^2 \right] \right] = \Lambda, \label{appendix:rom:con:eqn4}
\end{align}
where we review that $\Lambda = \frac{1}{T}\sum_{t=1}^T \max_{\x \in \X} \| \nabla f_t(\x) - \bar{\nabla}_T(\x) \|_2^2$. Then, we use a technique from \citet{OCO:Between} by introduce a variable $\tau \in [ T]$, which help us upper bound $\sigma_{1:T}^2$ as
\begin{align*}
    \sigma_{1:T}^2  = \E\left[ \sum_{t=1}^T \sigma_t^2 \right]  \leq {}& \E\left[ \sum_{t=1}^{\tau} \sigma_t^2 \right] + \E\left[ \sum_{t=\tau + 1}^T \sigma_t^2 \right]
    \overset{(\ref{appendix:rom:con:eqn3}),(\ref{appendix:rom:con:eqn4})}{\leq} \sum_{t=1}^{\tau}\frac{T}{T-(t-1)} \sigma_1^2 + (T-\tau)\Lambda\\
    \leq {}&\sum_{n=T-(\tau-1)}^{T}\frac{1}{n} T\sigma_1^2 + (T-\tau)\Lambda
    \leq \left(1 + \log \frac{T}{T-(\tau -1)}\right) T\sigma_1^2 + (T-\tau)\Lambda.
\end{align*}
If $T\sigma_1^2/\Lambda > 2$, we set $\tau = T-\lfloor T\sigma_1^2/\Lambda \rfloor$, then we have
\begin{align}
\sigma_{1:T}^2 
\leq & \left(1 + \log \frac{T}{\lfloor T\sigma_1^2/\Lambda \rfloor}\right) T\sigma_1^2+ T \sigma_1^2
\leq  \left(1 + \log \frac{1}{\sigma_1^2/\Lambda - 1/T}\right) T\sigma_1^2+ T \sigma_1^2\nonumber\\
\leq & \left(1 + \log \frac{2\Lambda}{\sigma_1^2}\right) T\sigma_1^2+ T \sigma_1^2
\leq  T \sigma_1^2 \log \left( \frac{2e^2 \Lambda}{\sigma_1^2} \right).\nonumber
\end{align}
Otherwise, if $T\sigma_1^2/\Lambda \leq 2$, we set $\tau = T$, then we can get the regret bound of $\O(T\sigma_1^2(1+\log T))$. Since we have 
\begin{align}
\O(T\sigma_1^2(1+\log T))\leq \O(T\sigma_1^2(1+\log (2\Lambda/\sigma_1^2)))\leq \O(T \sigma_1^2 \log ( 2e^2 \Lambda/\sigma_1^2)), \nonumber
\end{align}
then the final bound of $\sigma_{1:T}^2$ is  of order $\O\left(T \sigma_1^2 \log \big( \frac{2e^2 \Lambda}{\sigma_1^2} \big)\right)$.

Next, we try to bound $\Sigma_{1:T}^2$. We suppose that $k_{t} = \Gamma_{t}\backslash \Gamma_{t+1}$ represents the loss selected in round $t$, then we have
\begin{align}
    \|\nabla F_t(\x) - \nabla F_{t-1}(\x)\|_2^2 &= \left\|\frac{1}{T-(t-1)}\sum_{i \in \Gamma_t}\nabla f_i(\x) - \frac{1}{T-(t-2)}\sum_{i\in \Gamma_{t-1}}\nabla f_i(\x)\right\|_2^2\nonumber\\
    & = \left\| \frac{(T-t+2)-(T-t+1)}{(T-t+1)(T-t+2)}\sum_{i\in \Gamma_t}\nabla f_i(\x) -  \frac{1}{T-t+2} \nabla f_{k_{t-1}}(\x) \right\|_2^2\nonumber\\
    & \leq \frac{2}{(T-t + 2)^2}\left \| \frac{1}{T-t+1} \sum_{i \in \Gamma_t }\nabla f_i(\x) \right\|_2^2 + \frac{2}{(T-t+2)^2}\|\nabla f_{k_{t-1}} (\x)\|_2^2\nonumber\\
    & \leq \frac{4G^2}{(T-t + 2)^2},\nonumber
\end{align}
where the last inequality is derived from Assumption~\ref{ass:3} (boundedness of the gradient norm).

Summing the above inequality over $t=1,...,T$, and taking the expectation give
\begin{align}
    \Sigma_{1:T}^2 = \E\left[ \sum_{t=1}^T \sup_{\x\in\X} \|\nabla F_t(\x) - \nabla F_{t-1}(\x)\|_2^2 \right] \leq \sum_{t=1}^T \frac{4G^2}{(T-t + 2)^2} \leq 8G^2.\nonumber
\end{align}
Finally, we substitute the bound of $\sigma_{1:T}^2$ and $\Sigma_{1:T}$ into Theorem~\ref{thm:1}, which is for convex and smooth functions, and complete the proof.
\end{proof}

\subsection{Proof of Corollary~\ref{cor:OLaS}}\label{appendix:OLaS}
According to the problem setup in Section~\ref{subsec:OLaS}, we can bound $\tilde{\sigma}_{1:T}^2$ as
\begin{align}
    \tilde{\sigma}_{1:T}^2 ={}&\E\left[\sum_{t=1}^T\E\left[\sup_{\x\in\X}\left\|\sum_{k=1}^K\big(\big[\hat{\boldsymbol{\mu}}_{y_t}\big]_k - [\boldsymbol{\mu}_{y_t}]_k\big)\cdot \nabla  F_0^k(\x) \right\|_2^2\right]\right]\nonumber\\ \leq{}&\E\left[\sum_{t=1}^T\E\left[G^2\left\|\sum_{k=1}^K\big(\big[\hat{\boldsymbol{\mu}}_{y_t}\big]_k - [\boldsymbol{\mu}_{y_t}]_k\big)\right\|_2^2\right]\right]
    \leq KG^2\sum_{t=1}^T\E\left[\big\|\hat{\boldsymbol{\mu}}_{y_{t}} - \boldsymbol{\mu}_{y_{t}}\big\|_2^2\right]\label{appendix:olas:upperboundofsigma}
\end{align}
and bound $\Sigma_{1:T}^2$ as
\begin{align}
    \Sigma_{1:T}^2={}& \E \left[ \sum_{t=1}^T \sup_{\x \in \X} \left\| \sum_{k=1}^K  \big([\boldsymbol{\mu}_{y_t}]_k - [\boldsymbol{\mu}_{y_{t-1}}]_k \big)\cdot \nabla  F_0^k(\x)\right\|_2^2\right]\nonumber\\
    \leq{}&\E \left[G^2 \sum_{t=1}^T\left\| \sum_{k=1}^K  \big([\boldsymbol{\mu}_{y_t}]_k - [\boldsymbol{\mu}_{y_{t-1}}]_k \big)\right\|_2^2\right]
    \leq KG^2\sum_{t=1}^T\E\left[\big\|\boldsymbol{\mu}_{y_{t}} - \boldsymbol{\mu}_{y_{t-1}}\big\|_2^2 \right]. \label{appendix:olas:upperboundofSigma}
\end{align}

Next, it is necessary to apply the key technique of~\citet{UAI'20:simple} highlighted in \citet{NeurIPS'22:label_shift} to deal with the $P_T$ term in Theorem~\ref{thm:nonsmooth:dynamic}. We know that $P_T$ is the path length of the comparator sequence $\u_1,\cdots,\u_T$, where $\u_t$ can be any point in $\X$. While frequently used in exploring dynamic regret, this quantity lacks explicit significance in OLS. To this end, \citet{NeurIPS'22:label_shift} propose a new quantity $L_T=\sum_{t=2}^T\big\|\boldsymbol{\mu}_{y_{t}} - \boldsymbol{\mu}_{y_{t-1}}\big\|_1$ to measure the variation of label distributions and consider the dynamic regret against the sequence $\{\x_t^*\}_{t\in[T]}$, where $\x_t^*\in \argmin_{\x\in\X}F_t(\x)$ but $\x_t^*\notin \argmin_{\x\in\X}f_t(\x)$. In their footsteps, we first decompose the dynamic regret bound into two parts by introducing a reference sequence that only changes every $\Delta$ iteration. Specifically, the $m$-th time interval is denoted as $\mathcal{I}_m = [(m-1)\Delta +1, m\Delta]$ and any comparator $\u_t$ within $\mathcal{I}_m$ is considered the optimum decision for the interval, i.e., $\u_t = \x^*_{\mathcal{I}_m}\in\argmin_{\x\in\X}\sum_{t\in\mathcal{I}_m}F_t(\x)$ for any $t\in\mathcal{I}_m$. Then we have
\begin{align*}
    \E\left[\sum_{t=1}^T f_t(\x_t) - \sum_{t=1}^T f_t(\x_t^*)\right] ={}& \underbrace{\E\left[\sum_{t=1}^T f_t(\x_t) - \sum_{t=1}^T f_t(\u_t)\right]}_{\term{a}}+ \underbrace{\E\left[\sum_{m=1}^M\sum_{t\in\mathcal{I}_m} f_t(\x^*_{\mathcal{I}_m}) - \sum_{t=1}^T f_t(\x_t^*)\right]}_{\term{b}},
\end{align*}
where $M = \lceil T/\Delta \rceil \leq T/\Delta+1$.

For term (a), we can directly use \eqref{appendix:nonsmooth:dynamic:result} in the proof of Theorem~\ref{thm:nonsmooth:dynamic} to get
\begin{align*}
    \term{a} \leq{}& \Big(D(\ln N + 4)+2\sqrt{2(D^2 + 2DP_T)}\Big)\Big(G + 2\sqrt{2\tilde{\sigma}_{1:T}^2}+2\sqrt{\Sigma_{1:T}^2}\Big) + 4G^4 + \ln N + 4.
\end{align*}
Notice we can derive that $P_T\leq D(M-1)\leq DT/\Delta$ because the comparator sequence $\{\u_t\}_{t=1}^T$ only changes $M-1$ times. Hence, we have
\begin{align*}
    \term{a} \leq{}&\Big(D(\ln N + 4)+4D\sqrt{\frac{2T}{\Delta}}+2\sqrt{2}D\Big)\Big(G + 2\sqrt{2\tilde{\sigma}_{1:T}^2}+2\sqrt{\Sigma_{1:T}^2}\Big) + 4G^4 + \ln N + 4.
\end{align*}
For term (b), we follow the analysis of \citet{NeurIPS'22:label_shift} to show that
\begin{align*}
    \term{b} ={}&\sum_{m=1}^M\sum_{t\in\mathcal{I}_m} \left(F_t(\x^*_{\mathcal{I}_m}) -  F_t(\x_t^*)\right)
    \leq \sum_{m=1}^M\sum_{t\in \mathcal{I}_m} \left(F_t(\x^*_{s_m}) -  F_t(\x_t^*)\right)\\
    ={}&\sum_{m=1}^M\sum_{t\in \mathcal{I}_m} \left(F_t(\x^*_{s_m}) -  F_{s_m}(\x^*_{s_m})+F_{s_m}(\x^*_{s_m})-F_t(\x_t^*)\right)\\
    \leq{}&\sum_{m=1}^M\sum_{t\in \mathcal{I}_m} \left(F_t(\x^*_{s_m}) -  F_{s_m}(\x^*_{s_m})+F_{s_m}(\x_t^*)-F_t(\x_t^*)\right)\\
    \leq{}&2\Delta\sum_{m=1}^M\sum_{t\in \mathcal{I}_m}\sup_{\x\in\X}\left|F_t(\x)-F_{t-1}(\x)\right|
    = 2\Delta\sum_{t=2}^T\sup_{\x\in\X}\left|F_t(\x)-F_{t-1}(\x)\right|,
\end{align*}
where $s_m=(m-1)\Delta+1$ is the first time step at $\mathcal{I}_m$. Since $\term{b}=0$ when $\Delta=1$, we can bound it as $\term{b}\leq \mathbbm{1}\{\Delta>1\}\cdot 2\Delta\sum_{t=2}^T\sup_{\x\in\X}\left|F_t(\x)-F_{t-1}(\x)\right|$. Furthermore, we transform the $\sum_{t=2}^T\sup_{\x\in\X}\left|F_t(\x)-F_{t-1}(\x)\right|$ term into a term related to $L_T$ as
\begin{align*}
\sum_{t=2}^T\sup_{\x\in\X}\left|F_t(\x)-F_{t-1}(\x)\right|
    ={}&\sum_{t=2}^T\sup_{\x\in\X}\left|\sum_{k=1}^K\left([\boldsymbol{\mu}_{y_t}]_k-[\boldsymbol{\mu}_{y_{t-1}}]_k\right)F_0^k(\x)\right|\\
    \leq{}&\sum_{t=2}^T B\sum_{k=1}^K\left|[\boldsymbol{\mu}_{y_t}]_k-[\boldsymbol{\mu}_{y_{t-1}}]_k\right|
    = B\sum_{t=2}^T\left\| \boldsymbol{\mu}_{y_t} -\boldsymbol{\mu}_{y_{t-1}} \right\|_1 = B L_T,
\end{align*}
where $B\triangleq \sup_{(\z,y)\in \mathcal{Z}\times \mathcal{Y},\x\in\X}\left|\ell\left(h(\x,\z),y\right) \right|$ is the upper bound of function values. Thus, $\term{b}\leq \mathbbm{1}\{\Delta>1\}\cdot 2B\Delta L_T$. Combining it with the upper bound of term (a) yields
\begin{align*}
    {}&\E\left[\sum_{t=1}^T f_t(\x_t) - \sum_{t=1}^T f_t(\x_t^*)\right]\\
    \leq{}&  \mathbbm{1}\{\Delta>1\}\cdot 2B\Delta L_T +4D\sqrt{\frac{2T}{\Delta}}\left(G + 2\sqrt{2\tilde{\sigma}_{1:T}^2}+2\sqrt{\Sigma_{1:T}^2}\right)\\
    {}&\qquad + \left(\ln N + 4+2\sqrt{2}\right)D\left(G + 2\sqrt{2\tilde{\sigma}_{1:T}^2}+2\sqrt{\Sigma_{1:T}^2}\right) + 4G^4 + \ln N + 4.
\end{align*}
Below, we set different values for $\Delta$ in two cases and obtain the final regret bound.
\paragraph{Case 1.} When $D\sqrt{2T}\left(G + 2\sqrt{2\tilde{\sigma}_{1:T}^2}+2\sqrt{\Sigma_{1:T}^2}\right) > BL_T$, in such a case, we can set $\Delta = \left\lceil \big( D\sqrt{2T}(G + 2\sqrt{2\tilde{\sigma}_{1:T}^2}+2\sqrt{\Sigma_{1:T}^2})\big)^{\frac{2}{3}} (BL_T)^{-\frac{2}{3}}  \right\rceil$. Then we get that
\begin{align*}
    {}&\E\left[\sum_{t=1}^T f_t(\x_t) - \sum_{t=1}^T f_t(\x_t^*)\right]\\
    \leq{}& 12(BD^2 L_T T)^{\frac{1}{3}} \Big(G + \sqrt{2\tilde{\sigma}_{1:T}^2}+\sqrt{\Sigma_{1:T}^2}\Big)^{\frac{2}{3}} +(\ln N + 8)\left(1 + D\Big(G + 2\sqrt{2\tilde{\sigma}_{1:T}^2}+2\sqrt{\Sigma_{1:T}^2}\Big)\right) + 4G^4\\
    ={}&\O\left(L_T^{\frac{1}{3}} T^{\frac{1}{3}} \left(\sqrt{\tilde{\sigma}_{1:T}^2}+\sqrt{\Sigma_{1:T}^2}\right)^{\frac{2}{3}} +\left(\sqrt{\tilde{\sigma}_{1:T}^2}+\sqrt{\Sigma_{1:T}^2}\right) \right).
\end{align*}
\paragraph{Case 2.} When $D\sqrt{2T}\left(G + 2\sqrt{2\tilde{\sigma}_{1:T}^2}+2\sqrt{\Sigma_{1:T}^2}\right) \leq BL_T$, we set $\Delta = 1$ and get
\begin{align*}
    {}&\E\left[\sum_{t=1}^T f_t(\x_t) - \sum_{t=1}^T f_t(\x_t^*)\right]\\
    \leq{}& 8(BD^2 L_T T)^{\frac{1}{3}} \Big(G + \sqrt{2\tilde{\sigma}_{1:T}^2}+\sqrt{\Sigma_{1:T}^2}\Big)^{\frac{2}{3}}
    + \left(\ln N + 8\right)\left(1 + D\Big(G + 2\sqrt{2\tilde{\sigma}_{1:T}^2}+2\sqrt{\Sigma_{1:T}^2}\Big)\right)+ 4G^4\\
    ={}&\O\left(L_T^{\frac{1}{3}} T^{\frac{1}{3}} \left(\sqrt{\tilde{\sigma}_{1:T}^2}+\sqrt{\Sigma_{1:T}^2}\right)^{\frac{2}{3}} +\left(\sqrt{\tilde{\sigma}_{1:T}^2}+\sqrt{\Sigma_{1:T}^2}\right) \right).
\end{align*}
We end the proof by combining the two cases with the upper bounds of $\tilde{\sigma}_{1:T}^2$ and $\Sigma_{1:T}^2$.

\section{Technical Lemmas}
\label{appendix:tech-lemma}
\begin{myLemma}[{Bregman proximal inequality~\citep[Lemma 3.2]{chen1993convergence}}]
\label{lem:bregman_proximal}
Let $\X$ be a convex set in a Banach space. Let $f:\X \rightarrow \mathbbm{R}$ be a closed proper convex function on $\X$. Given a convex regularizer $\psi:\X \rightarrow \mathbbm{R}$, we denote its induced Bregman divergence by $\D_{\MR}(\cdot,\cdot)$. Then, any update of the form $\x_k = \argmin_{\x\in\X}\left\{f(\x)+\D_{\MR}(\x,\x_{k-1})\right\}$ satisfies the following inequality for any $\u\in\X$
\begin{align*}
    f(\x_k)-f(\u)\leq \D_{\MR}(\u,\x_{k-1})-\D_{\MR}(\u,\x_{k})-\D_{\MR}(\x_k,\x_{k-1}).
\end{align*}
\end{myLemma}

\begin{myLemma}\label{lem:sum} Let $l_1,\ldots$, $l_T$ and $\delta$ be non-negative real numbers. Then $\sum_{t=1}^T \frac{l_t}{\sqrt{\delta+\sum_{i=1}^t l_i}} \leq 2 \sqrt{\delta + \sum_{t=1}^T l_t}$, where we define $0/\sqrt{0}=0$ for simplicity.
\end{myLemma}

\begin{myLemma}\label{lem:selftuning} Let $l_1,\ldots$, $l_T$ and $\delta$ be non-negative real numbers. Then $\sum_{t=1}^T \frac{l_t}{\sqrt{\delta+\sum_{i=1}^{t-1} l_i}} \leq 4 \sqrt{\delta + \sum_{t=1}^T l_t} + \max_{t\in [T]}l_t$, where we define $0/\sqrt{0}=0$ for simplicity.
\end{myLemma}

\begin{myLemma}[Lemma 12 of~\citet{ML:Hazan:2007}]
\label{lem:trace} Let $A \succeq B \succ 0$ be positive definite matrices. Then $\langle A^{-1} , A-B \rangle   \leq \ln \frac{|A|}{|B|}$, where $|A|$ denotes the determinant of matrix $A$.
\end{myLemma}

\begin{myLemma}\label{lem:hazan}
Let $\u_t \in \R^d$ $(t=1,...,T)$, be a sequence of vectors. Define $S_t = \sum_{\tau = 1}^t \u_\tau \u_\tau^\top + \epsilon I$, where $\epsilon>0$. Then $\sum_{t=1}^T \u_t^\top S_t^{-1} \u_t \leq d \ln \left( 1+ \frac{\sum_{t=1}^T \|\u_t\|_2^2}{d \epsilon} \right)$.
\end{myLemma}
\begin{proof}
Using Lemma~\ref{lem:trace}, we have $\langle A^{-1} , A-B \rangle \leq \ln \frac{|A|}{|B|}$  for any two positive definite matrices $A \succeq B \succ 0$. Following the argument of~\citet[Theorem 2]{NIPS2016_6207}, we have
\[
\begin{split}
& \sum_{t=1}^T \u_t^\top S_t^{-1} \u_t =\sum_{t=1}^T \langle S_t^{-1}, \u_t\u_t^\top \rangle  = \sum_{t=1}^T \langle S_t^{-1}, S_t - S_{t-1} \rangle \leq \sum_{t=1}^T  \ln \frac{|S_t|}{|S_{t-1}|}  \\
= {}& \ln \frac{|S_T|}{|S_{0}|} = \sum_{i=1}^d \ln \Big( 1+ \frac{\lambda_i(\sum_{t = 1}^T \u_t\u_t^\top)}{\epsilon} \Big)= d \sum_{i=1}^d \frac{1}{d} \ln \Big( 1+ \frac{\lambda_i(\sum_{t = 1}^T \u_t\u_t^\top)}{\epsilon} \Big) \\
  \leq {} & d \ln \Big( 1+ \frac{\sum_{i=1}^d \lambda_i(\sum_{t = 1}^T \u_t\u_t^\top)}{d \epsilon} \Big) = d \ln \Big( 1+ \frac{\sum_{t=1}^T \|\u_t\|_2^2}{d \epsilon} \Big),
\end{split}
\]
where the last inequality is due to Jensen's inequality.
\end{proof}

\begin{myLemma}[{regret analysis of optimistic FTRL \citep[Theorem 7.35]{Modern:Online:Learning}}]\label{lem:FTRL}
    Let $V \in \R^d$ be convex, closed, and non-empty. Denote by $G_t(\x) = \Psi_t(\x) + \sum_{s=1}^{t-1}\ell_s(\x)$. Assume for $t=1,\cdots, T$ that $G_t$ is proper and $\lambda_t$-strongly convex with respect to~$\|\cdot\|$, $\ell_t$ and $\tilde{\ell}_t$ proper and convex ($\tilde{\ell}_t$ is the predicted next loss), and int dom $G_t \cap V \neq \{\}$. Also, assume that $\partial \ell_t(\x_t)$ and $\partial \tilde{\ell}_t(\x_t)$ are non-empty. Then there exists $\tilde{\g}_t \in \partial \tilde{\ell}_t(\x_t)$ for $t \in [T]$ such that 
    \begin{align*}
        &\sum_{t=1}^T \ell_t(\x_t) - \sum_{t=1}^T \ell_t(\x)\\
        \leq{}& \Psi_{T+1}(\x) - \Psi_1(\x_1)  + \sum_{t=1}^T \Big( \langle \g_t - \tilde{\g}_t, \x_t - \x_{t+1} \rangle -\frac{\lambda_t}{2}\|\x_t - \x_{t+1}\|^2 + \Psi_t(\x_{t+1})-\Psi_{t+1}(\x_{t+1})\Big)
    \end{align*}
    for all $\g_t \in \partial \ell_t(\x_t)$.
\end{myLemma}

\begin{myLemma}[Lemma 13 of~\citet{JMLR:sword++}]\label{lem:min}
    Let $a_1,a_2,\cdots,a_T,b$ and $\Bar{c}$ be non-negative real numbers and $a_t \in [0,B]$ for any $t\in [T]$. Let the step size be $c_t = \min \left\{ \Bar{c},\sqrt{\frac{b}{\sum_{s=1}^t a_s}} \right\}$ and $c_0 = \Bar{c}$. Then, we have $\sum_{t=1}^T c_{t-1}a_t \leq 2\Bar{c}B + 4\sqrt{b\sum_{t=1}^T a_t}$.
\end{myLemma}

%% file: appendixes/appendix_exp_concave_FTRL.tex
\subsection{Proof of Theorem~\ref{thm:FTRL-exp-concave}}
\begin{proof}
\label{appendix:proofs-exp-FTRL}
We use the following optimistic FTRL for $\alpha$-exp-concave and smooth functions,
\begin{align*}
    \x_{t+1} = \argmin_{\x\in\X} \left\{\frac{1}{2}(1 + \beta G^2)\|\x\|_2^2 + \sum_{s=1}^t \ell_s(\x) + \langle M_{t+1},\x \rangle \right\},
\end{align*}
where $\x_0$ is an arbitrary point in $\X$, $M_{t+1} = \nabla f_t(\x_t)$, and the surrogate loss $\ell_t(\x) = \left\langle \nabla f_t(\x_t), \x - \x_t \right\rangle + \frac{\beta}{2}\|\x-\x_t\|_{h_t}^2$ with $\beta = \frac{1}{2}\min \left\{ \frac{1}{4GD},\alpha \right\}$, and $h_t = \nabla f_t(\x_t)\nabla f_t(\x_t)^{\top}$. Furthermore, we set $M_1 = \nabla f_0(\x_0) = 0$. From \eqref{eqn:exp:exp}, we can easily derive that
\begin{align}
    \E\left[ \sum_{t=1}^T f_t(\x_t) - \sum_{t=1}^T f_t(\u) \right] \leq \E\left[\sum_{t=1}^T  \ell_t(\x_t) - \sum_{t=1}^T \ell_t(\u) \right]. \label{eqn:exp:FTRL:exp}
\end{align}
So in the following, we focus on the regret of surrogate losses. Denoting by $H_t = I + \beta G^2 I +\beta\sum_{s=1}^{t-1}h_s$ (where $I$ is the $d \times d$ identity matrix) and $G_t(\x) = \frac{1}{2}(1 + \beta G^2)\|\x\|_2^2 + \sum_{s=1}^{t-1} \ell_s(\x)$, we have that $G_t(\x)$ is 1-strongly convex w.r.t. $\|\cdot\|_{H_{t}}$. Hence, using Lemma~\ref{lem:FTRL} (Standard analysis of optimistic FTRL), we immediately get the following guarantee
\begin{align}
    &\sum_{t=1}^T \ell_t(\x_t) - \sum_{t=1}^T \ell_t(\u)\nonumber\\
    \leq{}& \frac{1+ \beta G^2}{2}\|\u\|_2^2 + \sum_{t=1}^T\left(\langle \nabla f_t(\x_t) - \nabla f_{t-1}(\x_{t-1}), \x_t - \x_{t+1}\rangle - \frac{1}{2}\|\x_t - \x_{t+1}\|_{H_{t}}^2 \right)\nonumber\\
    \leq{}&\frac{(1+\beta G^2)D^2}{2} + \underbrace{\sum_{t=1}^T \|\nabla f_t(\x_t) - \nabla f_{t-1}(\x_{t-1})\|_{H_{t}^{-1}}^2}_{\term{a}} - \underbrace{\frac{1}{4}\sum_{t=1}^T\|\x_t - \x_{t+1}\|_{H_{t}}^2}_{\term{b}} ,\label{eqn:exp:FTRL:surrogateexp}
\end{align}
where we denote the dual norm of $\|\cdot\|_{H_{t}}$ by $\|\cdot\|_{H_{t}^{-1}}$, and use Assumption~\ref{ass:4} (domain boundedness) and $\langle a,b \rangle \leq \|a\|_{*}\|b\| \leq \frac{1}{2c}\|a\|_{*}^2 + \frac{c}{2}\|b\|^2$ in the second inequality.

To bound term (a) in \eqref{eqn:exp:FTRL:surrogateexp}, we begin with the fact that
\begin{align}
    H_{t} \succeq {}& I +\beta\sum_{s=1}^{t} \nabla f_s(\x_s)\nabla f_s(\x_s)^{\top} \nonumber\\
    \succeq{}& I + \frac{\beta}{2}\sum_{s=1}^{t}\left(\nabla f_s(\x_s)\nabla f_s(\x_s)^{\top} + \nabla f_{s-1}(\x_{s-1})\nabla f_{s-1}(\x_{s-1})^{\top}\right) ,\label{eqn:exp:FTRL:St}
\end{align}
where the first inequality is due to Assumption~\ref{ass:3} (boundedness of gradient norms) and the second inequality comes from the definition that $\nabla f_0(\x_0) = 0$. We substitute \eqref{eqn:exp:Ht:nabla} in the proof of Theorem~\ref{thm:3} into \eqref{eqn:exp:FTRL:St} and obtain that
\begin{align*}
    H_{t} \succeq I + \frac{\beta}{4}\sum_{s=1}^{t}\left(\nabla f_{s}(\x_{s}) -\nabla f_{s-1}(\x_{s-1})\right)\left(\nabla f_{s}(\x_{s}) -\nabla f_{s-1}(\x_{s-1})\right)^{\top}.
\end{align*}
Let $P_t= I + \frac{\beta}{4}\sum_{s=1}^{t}\left(\nabla f_{s}(\x_{s}) -\nabla f_{s-1}(\x_{s-1})\right)\left(\nabla f_{s}(\x_{s}) -\nabla f_{s-1}(\x_{s-1})\right)^{\top}$ so that $H_{t} \succeq P_t$, then we can bound term (a) in \eqref{eqn:exp:FTRL:surrogateexp} as
\begin{align*}
    \term{a} \leq\sum_{t=1}^T \|\nabla f_t(\x_t) - \nabla f_{t-1}(\x_{t-1})\|_{P_{t}^{-1}}^2= \frac{4}{\beta}\sum_{t=1}^T \left\|\sqrt{\frac{\beta}{4}}(\nabla f_t(\x_t) - \nabla f_{t-1}(\x_{t-1}))\right\|_{P_{t}^{-1}}^2.
\end{align*}
By applying Lemma~\ref{lem:hazan} with $\u_t = \sqrt{\frac{\beta}{4}}(\nabla f_t(\x_t) - \nabla f_{t-1}(\x_{t-1}))$ and $\epsilon = 1$, we get that
\begin{align*}
    \term{a} \leq \frac{4d}{\beta}\ln \left(\frac{\beta}{4d}\Bar{V}_T  + 1\right).
\end{align*}
Then we move to term (b). Since $H_t = I +\beta G^2 I + \beta\sum_{s=1}^{t-1}h_s \succeq I$, we can derive that
\begin{align*}
    \term{b} = \frac{1}{4}\sum_{t=1}^T\|\x_t - \x_{t+1}\|_{H_{t}}^2 \geq \frac{1}{4}\sum_{t=1}^T\|\x_t - \x_{t+1}\|_{I}^2 = \frac{1}{4}\sum_{t=1}^T\|\x_t - \x_{t+1}\|_2^2.
\end{align*}
So we bound the guarantee in \eqref{eqn:exp:FTRL:surrogateexp} by substituting the bounds of term (a) and term (b):
\begin{align*}
    \sum_{t=1}^T \ell_t(\x_t) - \sum_{t=1}^T \ell_t(\u)
    \leq \frac{(1+\beta G^2)D^2}{2} + \frac{4d}{\beta}\ln \left(\frac{\beta}{4d}\Bar{V}_T  + 1\right) - \frac{1}{4}\sum_{t=1}^T\|\x_t - \x_{t+1}\|_2^2. 
\end{align*}
Through Lemma~\ref{lem:sum:diff} (Boundedness of cumulative norm of gradient difference) together with the inequality of $\ln(1+u+v) \leq \ln(1+u) + \ln(1+v) (u,v > 0)$, we get that
\begin{align}
    &\sum_{t=1}^T \ell_t(\x_t) - \sum_{t=1}^T \ell_t(\u)\nonumber\\
    \leq{}& \frac{4d}{\beta}\ln \bigg(\frac{2\beta}{d} \sum_{t=1}^{T}  \|\nabla f_t(\x_t) -\nabla F_t(\x_t)\|_2^2 +\frac{\beta}{d} \sum_{t=2}^{T} \|\nabla F_t(\x_{t-1}) -\nabla F_{t-1}(\x_{t-1})\|_2^2+ \frac{\beta}{4d}G^2 + 1 \bigg)\nonumber\\
    &+ \frac{(1+\beta G^2)D^2}{2}+ \frac{4d}{\beta}\ln \bigg(\frac{\beta L^2}{d} \sum_{t=2}^{T} \|\x_t-\x_{t-1}\|_2^2 + 1\bigg)- \frac{1}{4}\sum_{t=1}^T\|\x_t - \x_{t+1}\|_2^2\nonumber\\
    \leq{}& \frac{4d}{\beta}\ln \bigg(\frac{2\beta}{d} \sum_{t=1}^{T}  \|\nabla f_t(\x_t) -\nabla F_t(\x_t)\|_2^2 +\frac{\beta}{d} \sum_{t=2}^{T} \|\nabla F_t(\x_{t-1}) -\nabla F_{t-1}(\x_{t-1})\|_2^2+ \frac{\beta}{4d}G^2 + 1 \bigg)\nonumber\\
    &+ \frac{(1+\beta G^2)D^2}{2}+\frac{4d}{\beta}\ln(16L^2 +1).\nonumber
\end{align}
where the last step comes from Lemma~\ref{lem:ln:pq}.

Then we compute the expected regret by taking the expectation over the above regret with the help of Jensen's inequality and the derived result in \eqref{eqn:exp:FTRL:exp}:
\begin{align*}
    {}&\E\left[ \sum_{t=1}^T f_t(\x_t) - \sum_{t=1}^T f_t(\u) \right] 
    \leq \E\left[\sum_{t=1}^T \ell_t(\x_t) - \sum_{t=1}^T \ell_t(\u)\right]\\
    \leq{}&\frac{4d}{\beta}\ln \bigg(\frac{2\beta}{d} \sigma_{1:T}^2 +\frac{\beta}{d} \Sigma_{1:T}^2 + \frac{\beta}{4d}G^2 + 1 \bigg)+ \frac{(1+\beta G^2)D^2}{2}+\frac{4d}{\beta}\ln(16L^2 +1)\\
    ={}& \O\Big(\frac{d}{\alpha}\log(\sigma_{1:T}^2 +\Sigma_{1:T}^2)\Big).
\end{align*}
\end{proof}

%% file: appendixes/appendix_extensions.tex

\section{Omitted Proofs for Section~\ref{sec:extensions}}
\label{appendix:sec-extensions}
In this section, we present the omitted details for Section~\ref{sec:extensions}, including a discussion of the method using alternative optimism design and a useful lemma.

\subsection{Elaborations on an alternative method}\label{appendix:proof-xbar-regret} 
In this part, we demonstrate that when employing an alternative optimism design with $M_{t+1} = \nabla f_{t}(\bar{\x}_{t+1})$ where $\bar{\x}_{t+1}=\sum_{i=1}^N p_{t,i}\x_{t+1,i}$, we can only obtain a slightly worse regret scaling with the quantity $\tilde{\sigma}_{1:T}^2$. 

We first briefly describe the algorithm, which is a variant of Algorithm~\ref{alg:Sword++forSEA}. With the same two-layer structure as Algorithm~\ref{alg:Sword++forSEA}, each base-learner $\mathcal{B}_i$ updates its local decision by
\begin{align*}
    \xh_{t+1,i} = \Pi_{\X} \big[ \xh_{t,i} - \eta_i  \nabla f_t(\x_{t,i})\big],~~\x_{t+1,i} = \Pi_{\X} \big[ \xh_{t+1,i} - \eta_i  \nabla f_t(\x_{t,i})\big],
\end{align*}
which requires its own gradient direction; the meta-learner omits correction terms and updates the weight vector $\p_{t+1}\in\Delta_N$ by $p_{t+1,i} \propto \exp\left(-\epsilon_t \left(\sum_{s=1}^{t} \ell_{s,i} + m_{t+1,i} \right)\right)$ where the feedback loss $\ellb_t \in \R^N$ is constructed by $\ell_{t,i} = \inner{\nabla f_t(\x_t)}{\x_{t,i}}$ and the optimism $\mb_{t+1} \in \R^N$ is constructed as $m_{t+1,i} = \inner{M_{t+1}}{\x_{t+1,i}}$ with $M_{t+1}= \nabla f_{t}(\bar{\x}_{t+1})$. The step size $\eta_i$ of base-learners and the learning rate $\epsilon_t$ of meta-learner will be given later. Then for the above alternative algorithm, we can obtain the following theoretical guarantee.

\begin{myThm}\label{thm:dynamic:use_x_bar}
Under Assumptions~\ref{ass:3}, \ref{ass:4}, \ref{ass:2} and \ref{ass:1}, setting the step size pool $\H = \{\eta_1,\ldots,\eta_N\}$ with $\eta_i = \min\{ 1/(4L),2^{i-1}\sqrt{D^2/(98G^2T)}\}$ and $N = \lceil 2^{-1}\log_2(8G^2T/(L^2D^2)) \rceil+1$, and setting the learning rate as $\epsilon_t = 1/\sqrt{\delta + 4G^2 + \sum_{s=1}^{t-1} \left\| \nabla f_s(\x_s)-\nabla f_{s-1}(\bar{\x}_s) \right\|_2^2}$ with $\delta =4D^2L^2 \left(\ln N + 2D^2\right)$ for all $t \in [T]$, this variant of Algorithm~\ref{alg:Sword++forSEA} using the above optimism design and with no correction terms (more specifically, setting $\lambda = 0$ and $M_{t+1} = \nabla f_t(\bar{\x}_{t+1})$ with $\bar{\x}_{t+1}=\sum_{i=1}^N p_{t,i}\x_{t+1,i}$) can obtain the following bound
\begin{align*}
    \E\left[\DReg\right]\leq \O\left(P_T+\sqrt{1+P_T}\left(\sqrt{\tilde{\sigma}_{1:T}^2} + \sqrt{\Sigma_{1:T}^2} \right)\right).
\end{align*} 
\end{myThm}

\begin{proof}[of Theorem~\ref{thm:dynamic:use_x_bar}]
Notice that the dynamic regret can be decomposed into two parts:
\begin{align*}
    \E\left[\sum_{t=1}^T f_t(\x_t)-\sum_{t=1}^T f_t(\u^*_t)\right]=\underbrace{\E\left[\sum_{t=1}^T f_t(\x_t)-\sum_{t=1}^T f_t(\x_{t,i})\right]}_{\metaregret}+\underbrace{\E\left[\sum_{t=1}^T f_t(\x_{t,i})-\sum_{t=1}^T f_t(\u^*_t)\right]}_{\baseregret}.
\end{align*}
Then we provide the upper bounds for the two terms respectively.

\paragraph{Bounding the meta-regret.} Before giving the analysis, we first define $\hat{V}_t =\sum_{s=1}^{t}$ $\left\| \nabla f_s(\x_s)-\nabla f_{s-1}(\bar{\x}_s) \right\|_2^2 $ for the brevity of subsequent analysis. Similar to the proof in the previous section, we can easily get
\begin{align}
     \sum_{t=1}^T\left\langle \nabla f_t(\x_t), \x_t - \x_{t,i} \right\rangle \leq{}& \sum_{t=1}^T\epsilon_{t}\left\|\ellb_t-\mb_t\right\|_{\infty}^2+\frac{\ln N}{\epsilon_{T+1}}-\sum_{t=2}^T\frac{1}{4\epsilon_{t}}\left\|\p_t-\p_{t-1}\right\|_{1}^2\nonumber\\
    \leq{}& D^2\sum_{t=1}^T\epsilon_{t}\left\| \nabla f_t(\x_t)-\nabla f_{t-1}(\bar{\x}_t) \right\|_2^2+\frac{\ln N}{\epsilon_{T+1}}-\sum_{t=2}^T\frac{1}{4\epsilon_{t}}\left\|\p_t-\p_{t-1}\right\|_{1}^2\nonumber\\
    \leq{}& \left(\ln N + 2D^2\right)\sqrt{\delta + 4G^2 + \hat{V}_T} - \frac{\sqrt{\delta}}{4} \sum_{t=2}^T\left\|\p_t-\p_{t-1}\right\|_{1}^2, \label{appendix:eqn:x_bar:1}
\end{align}
where we bound the adaptivity term in the second inequality by
\begin{align*}
    \left\|\ellb_t-\mb_t\right\|_{\infty}^2 = \max_{i\in[N]}\left\langle \nabla f_t(\x_t)-\nabla f_{t-1}(\bar{\x}_t),\x_{t,i}\right\rangle^2 \leq D^2\left\| \nabla f_t(\x_t)-\nabla f_{t-1}(\bar{\x}_t) \right\|_2^2,
\end{align*}
and the last inequality comes from Lemma~\ref{lem:selftuning} and the fact that $\epsilon_t \leq \frac{1}{\sqrt{\delta + \hat{V}_t}} \leq \frac{1}{\sqrt{\delta}}$. Based on Assumption~\ref{ass:2} (smoothness of expected function), $\hat{V}_T$ can be bounded by
\begin{align}
    {}&\hat{V}_T = \sum_{t=1}^{T} \left\| \nabla f_t(\x_t)-\nabla f_{t-1}(\bar{\x}_t) \right\|_2^2\nonumber\\
    \leq{}& G^2 + 4\sum_{t=2}^T\bigg(\left\| \nabla f_t(\x_t)-\nabla F_{t}(\x_t) \right\|_2^2+ \left\| \nabla F_t(\x_t)-\nabla F_{t-1}(\x_t) \right\|_2^2+\left\| \nabla F_{t-1}(\x_t)-\nabla F_{t-1}(\bar{\x}_t) \right\|_2^2\nonumber\\
    {}&+\colorbox{gray}{$\left\| \nabla F_{t-1}(\bar{\x}_t)-\nabla f_{t-1}(\bar{\x}_t) \right\|_2^2$} \bigg)\label{appendix:eqn:gray}\\
    \leq{}& G^2+ 8 \sum_{t=1}^T\sup_{\x\in\X}\left\| \nabla f_t(\x)-\nabla F_{t}(\x) \right\|_2^2 + 4 \sum_{t=2}^T\left\| \nabla F_t(\x_t)-\nabla F_{t-1}(\x_t) \right\|_2^2+ 4L^2\sum_{t=2}^T\left\|\x_t -\bar{\x}_t \right\|_2^2\nonumber\\
    \leq{}& G^2+ 8\sum_{t=1}^T\sup_{\x\in\X}\left\| \nabla f_t(\x)-\nabla F_{t}(\x) \right\|_2^2+ 4 \sum_{t=2}^T\left\| \nabla F_t(\x_t)-\nabla F_{t-1}(\x_t) \right\|_2^2+ 4D^2L^2\sum_{t=2}^T\left\|\p_t -\p_{t-1} \right\|_1^2,\nonumber
\end{align}
where the last inequality is due to the fact
\begin{align*}
    \left\|\x_t -\bar{\x}_t \right\|_2^2 = \left\|\sum_{i=1}^N \left(p_{t,i}-p_{t-1,i}\right)\x_{t,i} \right\|_2^2\leq \left(\sum_{i=1}^N \left|p_{t,i}-p_{t-1,i}\right|\|\x_{t,i}\|_2\right)^2\leq D^2\left\|\p_t -\p_{t-1} \right\|_1^2.
\end{align*}

As a result, substitute the above upper bound into~\eqref{appendix:eqn:x_bar:1}, we arrive at
\begin{align*}
    {}&\sum_{t=1}^T\left\langle \nabla f_t(\x_t), \x_t - \x_{t,i} \right\rangle \\ 
    \leq{}& \left(\ln N + 2D^2\right)\sqrt{\delta + 5G^2} +\left(2\sqrt{2}\ln N + 4\sqrt{2}D^2\right)\sqrt{\sum_{t=1}^T\sup_{\x\in\X}\left\| \nabla f_t(\x)-\nabla F_{t}(\x) \right\|_2^2}\\
    {}&+\left(2\ln N + 4D^2\right)\sqrt{ \sum_{t=2}^T\left\| \nabla F_t(\x_t)-\nabla F_{t-1}(\x_t) \right\|_2^2}  +\left(2DL\ln N + 4D^3L\right)\sqrt{\sum_{t=2}^T\left\|\p_t -\p_{t-1} \right\|_1^2} \\
    {}&- \frac{\sqrt{\delta}}{4} \sum_{t=2}^T\left\|\p_t-\p_{t-1}\right\|_{1}^2\\
    \leq{}& \left(\ln N + 2D^2\right)\sqrt{\delta + 5G^2} +\left(2\sqrt{2}\ln N + 4\sqrt{2}D^2\right)\sqrt{\sum_{t=1}^T\sup_{\x\in\X}\left\| \nabla f_t(\x)-\nabla F_{t}(\x) \right\|_2^2}\\
    {}&+\left(2\ln N + 4D^2\right)\sqrt{ \sum_{t=2}^T\left\| \nabla F_t(\x_t)-\nabla F_{t-1}(\x_t) \right\|_2^2}+\frac{\left(2DL\ln N + 4D^3L\right)^2}{\sqrt{\delta}},
\end{align*}
where the last inequality is due to
\begin{align*}
    \left(2DL\ln N + 4D^3L\right)\sqrt{\sum_{t=2}^T\left\|\p_t -\p_{t-1} \right\|_1^2} \leq \frac{\left(2DL\ln N + 4D^3L\right)^2}{\sqrt{\delta}} + \frac{\sqrt{\delta}}{4}\sum_{t=2}^T\left\|\p_t -\p_{t-1} \right\|_1^2.
\end{align*}

Finally, we take expectations with the help of Jensen's inequality and obtain
\begin{align*}
    \E\left[\sum_{t=1}^T\left\langle \nabla f_t(\x_t), \x_t - \x_{t,i} \right\rangle\right] \leq {}&\left(2\sqrt{2}\ln N + 4\sqrt{2}D^2\right)\sqrt{\tilde{\sigma}_{1:T}^2}+ \left(2\ln N + 4D^2\right)\sqrt{ \Sigma_{1:T}^2} \\
    {}&+ \left(\ln N + 2D^2\right)\left(\sqrt{ 5}G + 4DL\sqrt{\ln N + 2D^2}\right),
\end{align*}
where we set $\delta =4D^2L^2 \left(\ln N + 2D^2\right)$.
\paragraph{Bounding the base-regret.} 
Notice that the base-learner actually performs optimistic OMD with the regularizer $\psi_t(\x)=\frac{1}{2\eta_i}\|\x\|_2^2$, we can apply Lemma~\ref{lem:1} to obtain the base-regret for any index $i \in [N]$ as:
\begin{align*}
{}&\sum_{t=1}^T\left\langle \nabla f_t(\x_{t,i}), \x_{t,i} -\u^*_t \right\rangle \\
\leq{}& \eta_i \sum_{t=1}^T  \left\| \nabla f_t(\x_{t,i}) - \nabla f_{t-1}(\x_{t-1,i}) \right\|_2^2 + \frac{1}{2\eta_i}\sum_{t=1}^T\left(\left\|\u^*_t - \hat{\x}_{t,i}\right\|_2^2-\left\|\u^*_t - \hat{\x}_{t+1,i}\right\|_2^2\right)\\
{}&- \frac{1}{2\eta_i}\sum_{t=1}^T\left(\left\|\hat{\x}_{t+1,i}-\x_{t,i}\right\|_2^2-\left\|\hat{\x}_{t,i}-\x_{t,i}\right\|_2^2\right)\\
\leq{}& \eta_i \sum_{t=1}^T  \left\| \nabla f_t(\x_{t,i}) - \nabla f_{t-1}(\x_{t-1,i}) \right\|_2^2 + \frac{D^2 + 2DP_T}{2\eta_i}-\frac{1}{4\eta_i}\sum_{t=2}^T\left\|\x_{t,i}-\x_{t-1,i}\right\|_2^2,
\end{align*}
where the derivation of the last inequality is similar to the previous proof and will not be repeated here. By exploiting Lemma~\ref{lem:sum:diff}, the above formula can be further bounded by
\begin{align*}
    {}&\sum_{t=1}^T\left\langle \nabla f_t(\x_{t,i}), \x_{t,i} -\u^*_t \right\rangle \\
    \leq{}& \eta_i G^2 + 8\eta_i \sum_{t=1}^T\left\|\nabla f_t(\x_{t,i})-\nabla F_t(\x_{t,i})\right\|_2^2 + 4\eta_i \sum_{t=2}^T\left\|\nabla F_t(\x_{t-1,i})-\nabla F_{t-1}(\x_{t-1,i})\right\|_2^2 \\
    {}&+ \left(4\eta_i L^2 -\frac{1}{4\eta_i} \right)\sum_{t=2}^T\left\|\x_{t,i}-\x_{t-1,i}\right\|_2^2  + \frac{D^2 + 2DP_T}{2\eta_i}\\
    \leq{}& 8\eta_i \sum_{t=1}^T\left\|\nabla f_t(\x_{t,i})-\nabla F_t(\x_{t,i})\right\|_2^2 + 4\eta_i \sum_{t=2}^T\left\|\nabla F_t(\x_{t-1,i})-\nabla F_{t-1}(\x_{t-1,i})\right\|_2^2 \\
    {}&  + \frac{D^2 + 2DP_T}{2\eta_i}+\eta_i G^2.
\end{align*}
The last inequality holds by ensuring the step size satisfies $\eta_i \leq 1/(4L)$ for any $i\in[N]$. Moreover, the above formula shows that the best step size is $\eta^{\dag} = \min\{1/(4L), \eta^*\}$, where 
\begin{align*}
    \eta^* = \sqrt{\frac{D^2 + 2DP_T}{16\sum_{t=1}^T\left\|\nabla f_t(\x_{t,i})-\nabla F_t(\x_{t,i})\right\|_2^2 + 8\sum_{t=2}^T\left\|\nabla F_t(\x_{t-1,i})-\nabla F_{t-1}(\x_{t-1,i})\right\|_2^2 + 2G^2}}.
\end{align*}
We set the step size pool $\mathcal{H}=\left\{\eta_i =\min\left\{ \frac{1}{4L},2^{i-1}\sqrt{\frac{D^2}{98G^2T}}\right\}\mid i\in[N]\right\}$, which ensures that $\eta^{\dag}$ is included. Then, if $\eta^* \leq 1/(4L)$, there must be an $\eta_{i^*} \in \mathcal{H}$ satisfying that $\eta_{i^*}\leq \eta^* \leq 2\eta_{i^*}$, and we can obtain that
\begin{align*}
    {}&\eta_{i^*}\left(8 \sum_{t=1}^T\left\|\nabla f_t(\x_{t,i})-\nabla F_t(\x_{t,i})\right\|_2^2 + 4 \sum_{t=2}^T\left\|\nabla F_t(\x_{t-1,i})-\nabla F_{t-1}(\x_{t-1,i})\right\|_2^2+G^2\right)+ \frac{D^2 + 2DP_T}{2\eta_{i^*}}\\
    \leq{}& \eta^*\left(8 \sum_{t=1}^T\left\|\nabla f_t(\x_{t,i})-\nabla F_t(\x_{t,i})\right\|_2^2 + 4 \sum_{t=2}^T\left\|\nabla F_t(\x_{t-1,i})-\nabla F_{t-1}(\x_{t-1,i})\right\|_2^2+G^2\right)+ \frac{D^2 + 2DP_T}{\eta^*}\\
    ={}& 2\sqrt{\left(D^2 + 2DP_T\right)\left(8 \sum_{t=1}^T\left\|\nabla f_t(\x_{t,i})-\nabla F_t(\x_{t,i})\right\|_2^2 + 4 \sum_{t=2}^T\left\|\nabla F_t(\x_{t-1,i})-\nabla F_{t-1}(\x_{t-1,i})\right\|_2^2+G^2\right)}.
\end{align*}
Otherwise, if $\eta^* > 1/(4L)$, we will choose $\eta_i = 1/(4L)$ and obtain that
\begin{align*}
    \frac{1}{4L}\bigg(8 \sum_{t=1}^T\left\|\nabla f_t(\x_{t,i})-\nabla F_t(\x_{t,i})\right\|_2^2 +{}& 4 \sum_{t=2}^T\left\|\nabla F_t(\x_{t-1,i})-\nabla F_{t-1}(\x_{t-1,i})\right\|_2^2+G^2\bigg)\\
    {}&+ 2L\left(D^2 + 2DP_T\right)\leq 6L\left(D^2 + 2DP_T\right).
\end{align*}
Hence, we get the final regret bound of the base-learner by taking both cases into account and taking expectations with Jensen's inequality
\begin{align*}
    \E\left[\sum_{t=1}^T\left\langle \nabla f_t(\x_{t,i}), \x_{t,i} -\u^*_t \right\rangle\right] \leq 2\sqrt{D^2 + 2DP_T}\left(2\sqrt{2\tilde{\sigma}_{1:T}^2} + 2\sqrt{\Sigma_{1:T}^2} + G\right) + 6L\left(D^2 + 2DP_T\right).
\end{align*}
\paragraph{Bounding the overall dynamic regret.} Combining the meta-regret and the base-regret, and using the convexity of expected functions, we have
\begin{align*}
    {}&\E\left[\sum_{t=1}^T f_t(\x_t)-\sum_{t=1}^T f_t(\u^*_t)\right]\\
    \leq{}&\E\left[\sum_{t=1}^T\left\langle \nabla f_t(\x_t), \x_t - \x_{t,i} \right\rangle\right] + \E\left[\sum_{t=1}^T\left\langle \nabla f_t(\x_{t,i}), \x_{t,i} -\u^*_t \right\rangle\right]\\
    \leq{}&\left(2\ln N + 4D^2 + 4\sqrt{D^2 + 2DP_T}\right)\left(\sqrt{2\tilde{\sigma}_{1:T}^2} + \sqrt{\Sigma_{1:T}^2} \right)\\
    {}&+ \left(\ln N + 2D^2\right)\left(\sqrt{ 5}G + 4DL\sqrt{\ln N + 2D^2}\right) +2G\sqrt{D^2 + 2DP_T} + 6L\left(D^2 + 2DP_T\right)\\
    ={}&\O\left(P_T+\sqrt{1+P_T}\left(\sqrt{\tilde{\sigma}_{1:T}^2} + \sqrt{\Sigma_{1:T}^2} \right)\right),
\end{align*}
which completes the proof.
\end{proof}
\begin{myRemark}
In fact, we can also prove Theorem~\ref{thm:dynamic:use_x_bar} simply by taking expectations over the $\O(\sqrt{(1+P_T+V_T)(1+P_T)})$ bound in Theorem 3 of~\citet{Problem:Dynamic:Regret}, but we give the above specific proof to illustrate the \emph{dependence issue} of this alternative optimism design. Specifically, according to the definition of $\bar{\x}_t$, it has dependency on $f_{t-1}$. So we cannot directly obtain the expectation of the gray item in~\eqref{appendix:eqn:gray}, but can only perform the supremum operation first, which results in the inability to get a bound scaling with $\sigma_{1:T}^2$.
\remarkend
\end{myRemark}

\subsection{Useful Lemma}\label{appendix:nonsmooth-usefullemma}

\begin{proof}[of Lemma~\ref{lem:proposition4.1}] According to the proof of the second inequality in Proposition 4.1 of \citet{campolongo2020temporal}, using the first-order optimality condition for $\x_{t+1}$ yields
\begin{align*}
    \langle \nabla f_t(\x_{t+1}) + \nabla \psi_{t+1}(\x_{t+1})-\nabla \psi_{t+1}(\hat{\x}_{t+1}),\x-\x_{t+1} \rangle \geq 0,\quad \forall \x \in \X.
\end{align*}
By moving terms, the above equation can be rewritten as
\begin{align*}
    \langle \nabla f_t(\x_{t+1}),\x_{t+1}-\x \rangle \leq{}& \langle  \nabla \psi_{t+1}(\hat{\x}_{t+1})- \nabla \psi_{t+1}(\x_{t+1}), \x_{t+1}-\x\rangle\\
    ={}&B_{\psi_{t+1}}(\x,\hat{\x}_{t+1}) - B_{\psi_{t+1}}(\x,\x_{t+1})-B_{\psi_{t+1}}(\x_{t+1},\hat{\x}_{t+1}),\quad \forall \x \in \X.
\end{align*}
Since $\psi_{t+1}(\x)=\frac{1}{2\eta_{t+1}}\|\x\|_2^2$, we can transform the above equation into
\begin{align*}
    \langle \nabla f_t(\x_{t+1}),\x_{t+1}-\x \rangle \leq \frac{1}{2\eta_{t+1}}\left( \|\x-\hat{\x}_{t+1}\|_2^2 - \|\x-\x_{t+1}\|_2^2 - \| \x_{t+1}-\hat{\x}_{t+1}\|_2^2 \right),\quad \forall \x \in \X,
\end{align*}
which completes the proof.
\end{proof}

\begin{proof}[of Lemma~\ref{lem:nonsmooth-sumlemma}] 
For $t \geq 2$, using Jensen’s inequality, we have
\begin{align*}
    {}&\sup_{\x\in\X}\|\nabla f_t(\x)-\nabla f_{t-1}(\x)\|_2^2\\
    \leq{}&\sup_{\x\in\X}\left(2\|\nabla f_t(\x)-\nabla F_{t}(\x)\|_2^2+2\|\nabla F_t(\x)-\nabla f_{t-1}(\x)\|_2^2\right)\\ 
    \leq{}& \sup_{\x\in\X}\left(2\|\nabla f_t(\x)-\nabla F_{t}(\x)\|_2^2+4\|\nabla F_t(\x)-\nabla F_{t-1}(\x)\|_2^2+4\|\nabla F_{t-1}(\x)-\nabla f_{t-1}(\x)\|_2^2\right).
\end{align*}
For $t=1$, from Assumption~\ref{ass:3} (boundedness of the gradient norm), we have
\begin{align*}
    \sup_{\x\in\X}\|\nabla f_1(\x) - \nabla f_{0} (\x) \|_2^2=\sup_{\x\in\X}\|\nabla f_1(\x)  \|_2^2 \leq G^2 .
\end{align*}
As a result, we can complete the proof by adding the terms from $t=1$ to $t=T$.
\end{proof}

%% file: OMD4SEA_arxiv.bbl
\begin{thebibliography}{57}
\providecommand{\natexlab}[1]{#1}
\providecommand{\url}[1]{\texttt{#1}}
\expandafter\ifx\csname urlstyle\endcsname\relax
  \providecommand{\doi}[1]{doi: #1}\else
  \providecommand{\doi}{doi: \begingroup \urlstyle{rm}\Url}\fi

\bibitem[Abernethy et~al.(2008)Abernethy, Bartlett, Rakhlin, and
  Tewari]{Minimax:Online}
Jacob Abernethy, Peter~L. Bartlett, Alexander Rakhlin, and Ambuj Tewari.
\newblock Optimal strategies and minimax lower bounds for online convex games.
\newblock In \emph{Proceedings of the 21st Annual Conference on Learning Theory
  (COLT)}, pages 415--423, 2008.

\bibitem[Ahn et~al.(2020)Ahn, Yun, and Sra]{ahn2020sgd}
Kwangjun Ahn, Chulhee Yun, and Suvrit Sra.
\newblock {SGD} with shuffling: optimal rates without component convexity and
  large epoch requirements.
\newblock In \emph{Advances in Neural Information Processing Systems 33
  (NeurIPS)}, pages 17526--17535, 2020.

\bibitem[Amir et~al.(2020)Amir, Attias, Koren, Livni, and
  Mansour]{2020Prediction}
Idan Amir, Idan Attias, Tomer Koren, Roi Livni, and Yishay Mansour.
\newblock Prediction with corrupted expert advice.
\newblock In \emph{Advances in Neural Information Processing Systems 33
  (NeurIPS)}, pages 14315--14325, 2020.

\bibitem[Auer et~al.(2002)Auer, Cesa-Bianchi, and Gentile]{AUER200248}
Peter Auer, Nicol\`{o} Cesa-Bianchi, and Claudio Gentile.
\newblock Adaptive and self-confident on-line learning algorithms.
\newblock \emph{Journal of Computer and System Sciences}, 64\penalty0
  (1):\penalty0 48--75, 2002.

\bibitem[Baby and Wang(2022)]{AISTATS'22:sc-proper}
Dheeraj Baby and Yu-Xiang Wang.
\newblock Optimal dynamic regret in proper online learning with strongly convex
  losses and beyond.
\newblock In \emph{Proceedings of the 25th International Conference on
  Artificial Intelligence and Statistics (AISTATS)}, pages 1805--1845, 2022.

\bibitem[Bai et~al.(2022)Bai, Zhang, Zhao, Sugiyama, and
  Zhou]{NeurIPS'22:label_shift}
Yong Bai, Yu-Jie Zhang, Peng Zhao, Masashi Sugiyama, and Zhi-Hua Zhou.
\newblock Adapting to online label shift with provable guarantees.
\newblock In \emph{Advances in Neural Information Processing Systems 35
  (NeurIPS)}, pages 29960--29974, 2022.

\bibitem[Campolongo and Orabona(2020)]{campolongo2020temporal}
Nicolo Campolongo and Francesco Orabona.
\newblock Temporal variability in implicit online learning.
\newblock In \emph{Advances in Neural Information Processing Systems 33
  (NeurIPS)}, pages 12377--12387, 2020.

\bibitem[Cesa-Bianchi et~al.(2004)Cesa-Bianchi, Conconi, and
  Gentile]{TIT04:Bianchi}
Nicol\`{o} Cesa-Bianchi, Alex Conconi, and Claudio Gentile.
\newblock On the generalization ability of on-line learning algorithms.
\newblock \emph{IEEE Transactions on Information Theory}, 50\penalty0
  (9):\penalty0 2050--2057, 2004.

\bibitem[Chen and Teboulle(1993)]{chen1993convergence}
Gong Chen and Marc Teboulle.
\newblock Convergence analysis of a proximal-like minimization algorithm using
  bregman functions.
\newblock \emph{SIAM Journal on Optimization}, 3\penalty0 (3):\penalty0
  538--543, 1993.

\bibitem[Chen and Orabona(2023)]{chen2023generalized}
Keyi Chen and Francesco Orabona.
\newblock Generalized implicit follow-the-regularized-leader.
\newblock In \emph{Proceedings of the 40th International Conference on Machine
  Learning (ICML)}, pages 4826--4838, 2023.

\bibitem[Chen et~al.(2023)Chen, Tu, Zhao, and Zhang]{ICML'23:OMD4SEA}
Sijia Chen, Wei-Wei Tu, Peng Zhao, and Lijun Zhang.
\newblock Optimistic online mirror descent for bridging stochastic and
  adversarial online convex optimization.
\newblock In \emph{Proceedings of the 40th International Conference on Machine
  Learning (ICML)}, pages 5002--5035, 2023.

\bibitem[Chiang et~al.(2012)Chiang, Yang, Lee, Mahdavi, Lu, Jin, and
  Zhu]{Gradual:COLT:12}
Chao-Kai Chiang, Tianbao Yang, Chia-Jung Lee, Mehrdad Mahdavi, Chi-Jen Lu, Rong
  Jin, and Shenghuo Zhu.
\newblock Online optimization with gradual variations.
\newblock In \emph{Proceedings of the 25th Annual Conference on Learning Theory
  (COLT)}, pages 6.1--6.20, 2012.

\bibitem[Duchi et~al.(2011)Duchi, Hazan, and Singer]{JMLR:Adaptive}
John Duchi, Elad Hazan, and Yoram Singer.
\newblock Adaptive subgradient methods for online learning and stochastic
  optimization.
\newblock \emph{Journal of Machine Learning Research}, 12\penalty0
  (7):\penalty0 2121--2159, 2011.

\bibitem[Gaillard and Wintenberger(2018)]{NIPS'18:sparse-OCO}
Pierre Gaillard and Olivier Wintenberger.
\newblock Efficient online algorithms for fast-rate regret bounds under
  sparsity.
\newblock In \emph{Advances in Neural Information Processing Systems 31
  (NeurIPS)}, pages 7026--7036, 2018.

\bibitem[Garber et~al.(2020)Garber, Korcia, and Levy]{ROM2020}
Dan Garber, Gal Korcia, and Kfir~Y. Levy.
\newblock Online convex optimization in the random order model.
\newblock In \emph{Proceedings of the 37th International Conference on Machine
  Learning (ICML)}, pages 3387--3396, 2020.

\bibitem[Ghadimi and Lan(2012)]{Lan:SCSC}
Saeed Ghadimi and Guanghui Lan.
\newblock Optimal stochastic approximation algorithms for strongly convex
  stochastic composite optimization {I}: A generic algorithmic framework.
\newblock \emph{Siam Journal on Optimization}, 22\penalty0 (4):\penalty0
  1469--1492, 2012.

\bibitem[Hazan(2016)]{Intro:Online:Convex}
Elad Hazan.
\newblock Introduction to online convex optimization.
\newblock \emph{Foundations and Trends in Optimization}, 2\penalty0
  (3-4):\penalty0 157--325, 2016.

\bibitem[Hazan and Kale(2011)]{COLT:Hazan:2011}
Elad Hazan and Satyen Kale.
\newblock Beyond the regret minimization barrier: an optimal algorithm for
  stochastic strongly-convex optimization.
\newblock In \emph{Proceedings of the 24th Annual Conference on Learning Theory
  (COLT)}, pages 421--436, 2011.

\bibitem[Hazan et~al.(2007)Hazan, Agarwal, and Kale]{ML:Hazan:2007}
Elad Hazan, Amit Agarwal, and Satyen Kale.
\newblock Logarithmic regret algorithms for online convex optimization.
\newblock \emph{Machine Learning}, 69\penalty0 (2-3):\penalty0 169--192, 2007.

\bibitem[Hu et~al.(2017)Hu, Seiler, and Rantzer]{hu2017unified}
Bin Hu, Peter Seiler, and Anders Rantzer.
\newblock A unified analysis of stochastic optimization methods using jump
  system theory and quadratic constraints.
\newblock In \emph{Proceedings of the 30th Conference on Learning Theory
  (COLT)}, pages 1157--1189, 2017.

\bibitem[Ito(2021)]{2021On}
Shinji Ito.
\newblock On optimal robustness to adversarial corruption in online decision
  problems.
\newblock In \emph{Advances in Neural Information Processing Systems 34
  (NeurIPS)}, pages 7409--7420, 2021.

\bibitem[Jacobsen and Cutkosky(2022)]{jacobsen2022parameter}
Andrew Jacobsen and Ashok Cutkosky.
\newblock Parameter-free mirror descent.
\newblock In \emph{Proceedings of the 35th Conference on Learning Theory
  (COLT)}, pages 4160--4211, 2022.

\bibitem[Jadbabaie et~al.(2015)Jadbabaie, Rakhlin, Shahrampour, and
  Sridharan]{Dynamic:AISTATS:15}
Ali Jadbabaie, Alexander Rakhlin, Shahin Shahrampour, and Karthik Sridharan.
\newblock Online optimization: Competing with dynamic comparators.
\newblock In \emph{Proceedings of the 18th International Conference on
  Artificial Intelligence and Statistics (AISTATS)}, pages 398--406, 2015.

\bibitem[Johnson and Zhang(2013)]{NIPS13:ASGD}
Rie Johnson and Tong Zhang.
\newblock Accelerating stochastic gradient descent using predictive variance
  reduction.
\newblock In \emph{Advances in Neural Information Processing Systems 26
  (NIPS)}, pages 315--323, 2013.

\bibitem[Joulani et~al.(2020)Joulani, Gy{\"{o}}rgy, and
  Szepesv{\'{a}}ri]{TCS'20:modular-composite}
Pooria Joulani, Andr{\'{a}}s Gy{\"{o}}rgy, and Csaba Szepesv{\'{a}}ri.
\newblock A modular analysis of adaptive (non-)convex optimization: Optimism,
  composite objectives, variance reduction, and variational bounds.
\newblock \emph{Theoretical Computer Science}, 808:\penalty0 108--138, 2020.

\bibitem[Kingma and Ba(2015)]{Adam}
Diederik~P. Kingma and Jimmy~Lei Ba.
\newblock Adam: A method for stochastic optimization.
\newblock In \emph{Proceedings of the 3rd International Conference on Learning
  Representations (ICLR)}, 2015.

\bibitem[Koren and Levy(2015)]{NIPS2015_Exp}
Tomer Koren and Kfir Levy.
\newblock Fast rates for exp-concave empirical risk minimization.
\newblock In \emph{Advances in Neural Information Processing Systems 28
  (NIPS)}, pages 1477--1485, 2015.

\bibitem[Lan(2012)]{Lan:SCO}
Guanghui Lan.
\newblock An optimal method for stochastic composite optimization.
\newblock \emph{Mathematical Programming}, 133\penalty0 (1):\penalty0 365--397,
  2012.

\bibitem[Lipton et~al.(2018)Lipton, Wang, and Smola]{conf/icml/LiptonWS18}
Zachary~C. Lipton, Yu{-}Xiang Wang, and Alexander~J. Smola.
\newblock Detecting and correcting for label shift with black box predictors.
\newblock In \emph{Proceedings of the 35th International Conference on Machine
  Learning (ICML)}, pages 3128--3136, 2018.

\bibitem[Luo et~al.(2016)Luo, Agarwal, Cesa-Bianchi, and
  Langford]{NIPS2016_6207}
Haipeng Luo, Alekh Agarwal, Nicol\`{o} Cesa-Bianchi, and John Langford.
\newblock Efficient second order online learning by sketching.
\newblock In \emph{Advances in Neural Information Processing Systems 29
  (NIPS)}, pages 902--910, 2016.

\bibitem[Mahdavi et~al.(2015)Mahdavi, Zhang, and Jin]{COLT:2015:Mahdavi}
Mehrdad Mahdavi, Lijun Zhang, and Rong Jin.
\newblock Lower and upper bounds on the generalization of stochastic
  exponentially concave optimization.
\newblock In \emph{Proceedings of the 28th Annual Conference on Learning Theory
  (COLT)}, page 1305–1320, 2015.

\bibitem[McMahan and Streeter(2010)]{COLT'10:adaptiveOCO}
H.~Brendan McMahan and Matthew~J. Streeter.
\newblock Adaptive bound optimization for online convex optimization.
\newblock In \emph{Proceedings of the 23rd Conference on Learning Theory
  (COLT)}, pages 244--256, 2010.

\bibitem[Nemirovski(2005)]{nemirovski-2005-prox}
Arkadi Nemirovski.
\newblock Prox-method with rate of convergence ${O}(1/t)$ for variational
  inequalities with lipschitz continuous monotone operators and smooth
  convex-concave saddle point problems.
\newblock \emph{SIAM Journal on Optimization}, 15\penalty0 (1):\penalty0
  229--251, 2005.

\bibitem[Nemirovski et~al.(2009)Nemirovski, Juditsky, Lan, and
  Shapiro]{nemirovski-2008-robust}
Arkadi Nemirovski, Anatoli Juditsky, Guanghui Lan, and Alexander Shapiro.
\newblock Robust stochastic approximation approach to stochastic programming.
\newblock \emph{SIAM Journal on Optimization}, 19\penalty0 (4):\penalty0
  1574--1609, 2009.

\bibitem[Neu and Rosasco(2018)]{COLT'18:weighted-LS}
Gergely Neu and Lorenzo Rosasco.
\newblock Iterate averaging as regularization for stochastic gradient descent.
\newblock In \emph{Proceedings of the 31st Annual Conference on Learning Theory
  (COLT)}, pages 3222--3242, 2018.

\bibitem[Orabona(2019)]{Modern:Online:Learning}
Francesco Orabona.
\newblock A modern introduction to online learning.
\newblock \emph{ArXiv preprint}, arXiv:1912.13213, 2019.

\bibitem[Orabona et~al.(2012)Orabona, Cesa-Bianchi, and
  Gentile]{Beyond:Logarithmic}
Francesco Orabona, Nicolo Cesa-Bianchi, and Claudio Gentile.
\newblock Beyond logarithmic bounds in online learning.
\newblock In \emph{Proceedings of the 15th International Conference on
  Artificial Intelligence and Statistics (AISTATS)}, pages 823--831, 2012.

\bibitem[Ordentlich and Cover(1998)]{Lower:bound:Portfolio}
Erik Ordentlich and Thomas~M. Cover.
\newblock The cost of achieving the best portfolio in hindsight.
\newblock \emph{Mathematics of Operations Research}, 23\penalty0 (4):\penalty0
  960--982, 1998.

\bibitem[Rakhlin and Sridharan(2013)]{Predictable:COLT:2013}
Alexander Rakhlin and Karthik Sridharan.
\newblock Online learning with predictable sequences.
\newblock In \emph{Proceedings of the 26th Conference on Learning Theory
  (COLT)}, pages 993--1019, 2013.

\bibitem[Sachs et~al.(2022)Sachs, Hadiji, van Erven, and
  Guzm{\'a}n]{OCO:Between}
Sarah Sachs, Hedi Hadiji, Tim van Erven, and Crist{\'o}bal~A Guzm{\'a}n.
\newblock Between stochastic and adversarial online convex optimization:
  Improved regret bounds via smoothness.
\newblock In \emph{Advances in Neural Information Processing Systems 35
  (NeurIPS)}, pages 691--702, 2022.

\bibitem[Sachs et~al.(2023)Sachs, Hadiji, van Erven, and
  Guzman]{arxiv'2023:SEA}
Sarah Sachs, Hedi Hadiji, Tim van Erven, and Cristobal Guzman.
\newblock Accelerated rates between stochastic and adversarial online convex
  optimization.
\newblock \emph{ArXiv preprint}, arXiv:2303.03272, 2023.

\bibitem[Shalev-Shwartz(2007)]{Shai:thesis}
Shai Shalev-Shwartz.
\newblock \emph{Online Learning: Theory, Algorithms, and Applications}.
\newblock PhD thesis, The Hebrew University of Jerusalem, 2007.

\bibitem[Shalev-Shwartz(2016)]{shalev2016sdca}
Shai Shalev-Shwartz.
\newblock {SDCA} without duality, regularization, and individual convexity.
\newblock In \emph{Proceedings of the 33th International Conference on Machine
  Learning (ICML)}, pages 747--754, 2016.

\bibitem[Shalev-Shwartz et~al.(2009)Shalev-Shwartz, Shamir, Srebro, and
  Sridharan]{COLT:Shalev:2009}
Shai Shalev-Shwartz, Ohad Shamir, Nathan Srebro, and Karthik Sridharan.
\newblock Stochastic convex optimization.
\newblock In \emph{Proceedings of the 22nd Annual Conference on Learning Theory
  (COLT)}, page~5, 2009.

\bibitem[Sherman et~al.(2021)Sherman, Koren, and Mansour.]{OptimalROM2021}
Uri Sherman, Tomer Koren, and Yishay Mansour.
\newblock Optimal rates for random order online optimization.
\newblock In \emph{Advances in Neural Information Processing Systems 34
  (NeurIPS)}, pages 2097--2108, 2021.

\bibitem[Srebro et~al.(2010)Srebro, Sridharan, and Tewari]{NIPS2010_Smooth}
Nathan Srebro, Karthik Sridharan, and Ambuj Tewari.
\newblock Smoothness, low-noise and fast rates.
\newblock In \emph{Advances in Neural Information Processing Systems 23
  (NIPS)}, pages 2199--2207, 2010.

\bibitem[Syrgkanis et~al.(2015)Syrgkanis, Agarwal, Luo, and
  Schapire]{syrgkanis2015fast}
Vasilis Syrgkanis, Alekh Agarwal, Haipeng Luo, and Robert~E Schapire.
\newblock Fast convergence of regularized learning in games.
\newblock In \emph{Advances in Neural Information Processing Systems 28
  (NIPS)}, pages 2989--2997, 2015.

\bibitem[Zhang and Zhou(2019)]{Stochastic:Approximation:COLT}
Lijun Zhang and Zhi-Hua Zhou.
\newblock Stochastic approximation of smooth and strongly convex functions:
  Beyond the ${O}(1/{T})$ convergence rate.
\newblock In \emph{Proceedings of the 32nd Annual Conference on Learning Theory
  (COLT)}, pages 3160--3179, 2019.

\bibitem[Zhang et~al.(2013)Zhang, Mahdavi, and Jin]{NIPS:13:Mixed}
Lijun Zhang, Mehrdad Mahdavi, and Rong Jin.
\newblock Linear convergence with condition number independent access of full
  gradients.
\newblock In \emph{Advance in Neural Information Processing Systems 26 (NIPS)},
  pages 980--988, 2013.

\bibitem[Zhang et~al.(2018)Zhang, Lu, and Zhou]{Adaptive:Dynamic:Regret:NIPS}
Lijun Zhang, Shiyin Lu, and Zhi-Hua Zhou.
\newblock Adaptive online learning in dynamic environments.
\newblock In \emph{Advances in Neural Information Processing Systems 31
  (NeurIPS)}, pages 1323--1333, 2018.

\bibitem[Zhang et~al.(2022)Zhang, Wang, Yi, and Yang]{ICML:2022:Zhang}
Lijun Zhang, Guanghui Wang, Jinfeng Yi, and Tianbao Yang.
\newblock A simple yet universal strategy for online convex optimization.
\newblock In \emph{Proceedings of the 39th International Conference on Machine
  Learning (ICML)}, pages 26605--26623, 2022.

\bibitem[Zhang et~al.(2020)Zhang, Zhao, and Zhou]{UAI'20:simple}
Yu-Jie Zhang, Peng Zhao, and Zhi-Hua Zhou.
\newblock A simple online algorithm for competing with dynamic comparators.
\newblock In \emph{Proceedings of the 36th Conference on Uncertainty in
  Artificial Intelligence (UAI)}, pages 390--399, 2020.

\bibitem[Zhao et~al.(2020)Zhao, Zhang, Zhang, and Zhou]{Problem:Dynamic:Regret}
Peng Zhao, Yu-Jie Zhang, Lijun Zhang, and Zhi-Hua Zhou.
\newblock Dynamic regret of convex and smooth functions.
\newblock In \emph{Advances in Neural Information Processing Systems 33
  (NeurIPS)}, pages 12510--12520, 2020.

\bibitem[Zhao et~al.(2021)Zhao, Zhang, Zhang, and Zhou]{JMLR:sword++}
Peng Zhao, Yu-Jie Zhang, Lijun Zhang, and Zhi-Hua Zhou.
\newblock Adaptivity and non-stationarity: Problem-dependent dynamic regret for
  online convex optimization.
\newblock \emph{ArXiv preprint}, arXiv:2112.14368, 2021.

\bibitem[Zhou(2023)]{zhou2023stream}
Zhi-Hua Zhou.
\newblock A theoretical perspective of machine learning with computational
  resource concerns.
\newblock \emph{ArXiv preprint}, arXiv:2305.02217, 2023.

\bibitem[Zimmert and Seldin(2021)]{zimmert2021tsallis}
Julian Zimmert and Yevgeny Seldin.
\newblock Tsallis-{INF}: An optimal algorithm for stochastic and adversarial
  bandits.
\newblock \emph{Journal of Machine Learning Research}, 22\penalty0
  (28):\penalty0 1--49, 2021.

\bibitem[Zinkevich(2003)]{zinkevich-2003-online}
Martin Zinkevich.
\newblock Online convex programming and generalized infinitesimal gradient
  ascent.
\newblock In \emph{Proceedings of the 20th International Conference on Machine
  Learning (ICML)}, pages 928--936, 2003.

\end{thebibliography}
